\documentclass[twoside]{article}

\usepackage[accepted]{aistats2026}

\usepackage[round]{natbib}

\usepackage{amsmath}
\usepackage{amssymb}
\usepackage{mathtools}
\usepackage{amsthm}

\theoremstyle{plain}
\newtheorem{theorem}{Theorem}[section]
\newtheorem{proposition}[theorem]{Proposition}
\newtheorem{lemma}[theorem]{Lemma}
\newtheorem{corollary}[theorem]{Corollary}
\theoremstyle{definition}
\newtheorem{definition}[theorem]{Definition}
\newtheorem{assumption}[theorem]{Assumption}
\theoremstyle{remark}

\usepackage[table]{xcolor}
\usepackage{tikz}
\usetikzlibrary{positioning}
\newcommand{\sg}[1]{{\textcolor{violet}{\bf [{\sc SG:} #1]}}} \newcommand{\ag}[1]{{\textcolor{blue}{\bf [{\sc AG:} #1]}}} \newcommand{\ds}[1]{{\textcolor{magenta}{\bf [{\sc Dhanya:} #1]}}} \newcommand{\nh}[1]{{\textcolor{purple}{\bf [{\sc Nidhi:} #1]}}}

\renewcommand{\sg}[1]{}
\renewcommand{\ag}[1]{}
\renewcommand{\ds}[1]{}
\renewcommand{\nh}[1]{}

\usepackage{enumitem} 
\usepackage{dsfont}
\usepackage{soul}
\usepackage{bbding}

\usepackage{hyperref}
\usepackage{float}

\definecolor{mydarkblue}{rgb}{0.0, 0.0, 0.55}
\hypersetup{
  colorlinks=true,
  linkcolor=mydarkblue,
  citecolor=mydarkblue,
  filecolor=mydarkblue,
  urlcolor=mydarkblue
}

\usepackage{cleveref}
\usepackage{subcaption}
\usetikzlibrary{calc}

\begin{document}

\runningauthor{G\'{o}is, G{\"u}nl{\"u}k, Rosenfeld, Hegde, Lacoste-Julien, Sridhar}

\twocolumn[

\aistatstitle{The Role of Causal Features in Strategic Classification for Robustness and Alignment}

\vspace{-0.5cm}
\aistatsauthor{ Ant\'{o}nio~G\'{o}is$^{*\, 1}$ \And Sophia G{\"u}nl{\"u}k$^{*\, 1}$ \And Nir Rosenfeld$^{3}$}
\vspace{0.1cm}
\aistatsauthor{ Nidhi Hegde$^{2,4}$ \And Simon Lacoste-Julien$^{1, 2}$ \And Dhanya Sridhar$^{1, 2}$ }
\vspace{0.2cm}
\aistatsaddress{$^{1}$Mila \& Universit\'{e} de Montr\'{e}al \And $^{2}$Canada CIFAR AI Chair}
\vspace{-0.7cm}
\aistatsaddress{$^{3}$Faculty of Computer Science, \\ Technion - Israel Institute of Technology \And $^{4}$Dept. of Computing Science,\\ Amii \& University of Alberta, Canada
}
]

\begin{abstract}

In strategic classification, an institution (e.g., a bank) anticipates adaptation from users who change their features to increase utility in a classification task (e.g., loan repayment).
Since a key challenge is the distribution shift induced by users, we turn to causal models, which have been shown to bound the worst-case out-of-distribution (OOD) risk, and establish several new results that link causality and strategic classification.
First, we show that causal classification leads to optimal classification error after any sufficiently large adaptation, when the noise is bounded in a certain way. Second, when these assumptions do not hold, we show OOD cross-entropy risk of optimal classifiers decomposes into an OOD bias term and a term arising from
not using all observable features, allowing us to understand when causal classifiers have an advantage.
Finally, we show that the use of causal features can allow alignment of long-term incentives between institutions and users, contrasting with previous work that highlights social costs of such approaches.
We validate our theory empirically on synthetic data, finding that our results predict behavior in practice.

\end{abstract}

\section{INTRODUCTION}

As classifiers are deployed in decision-making contexts, it becomes increasingly important to study strategic classification \citep{hardt2016strategic}, where decision-makers seek to maximize accuracy as agents adapt their features in response to classifications.
As we develop better algorithms under varying assumptions about adaptation \citep{levanon2022generalized,klienberg2018incentive}, there are growing concerns about negative social impact on the agents who adapt to these systems, whether outcomes are static \citep{milli2019socialcost} or dynamic \citep{gois2025performative}.
When agents adapt, depending on the underlying causal model \citep{taleoftwoshifts,miller2020strategiccausal}, some changes improve agent outcomes while others constitute gaming the classifier, worsening classification error. 
In this paper, we study whether classifiers can maintain accuracy without sacrificing alignment with predicted agent's goals.

Taking inspiration from the link between causal models and robustness to distribution shifts \citep{peters2017elements}, we explore the impact of causal features in strategic classification, both as a reliable predictor and as an incentive.
We consider settings where an unobserved variable confounds prediction by introducing spurious features that do not cause the outcome of interest.
We first show that a causal classifier can reach optimal loss when the unobserved variable introduces ambiguity in a bounded region of the input space. Intuitively, if agents are willing to adapt enough, a causal classifier will move points away from ambiguous regions, reducing error, while classifiers that exploit spurious correlations run the risk of gaming. Such classifier is optimal not only against a specific adaptation, but to any large-enough adaptation.
Next, we show that even without this bounded influence from the latent, causal features lead to bounded error under changing distributions.
We then study how the predicted population is impacted, when an institution switches from vanilla prediction to methods that leverage agent adaptation -- dubbed a \textit{strategic institution}. Although myopic agents concerned with short-term utility may perceive a drop in utility, considering long-term consequences of gaming behavior shows that agents may be better off interacting with a strategic institution. Surprisingly, this can be the case even if the agents' short-term gain surpasses the long-term cost, when gaming. This shows that, unlike previous work where causal influence of $X$ on $Y$ was not considered, strategic classification may improve agents' utility instead of imposing a social cost.

\section{RELATED WORK}
This paper contributes to the body of work on strategic classification \citep{hardt2016strategic}, where utility-maximizing agents are incentivised to adapt their features in response to deployed models, changing the distribution of their features and potentially even outcomes \citep{klienberg2018incentive,miller2020strategiccausal, perdomo2020performative}.
Work on strategic classification largely focuses on developing algorithms in service of maintaining predictive performance \citep{dong2017, chen2019classifystrategic, ahmadi2020strategicperceptron, rosenfeld2021scpractical}, and studying the social costs and (mis)aligned incentives of strategic institutions \citep{klienberg2018incentive,bechavod2021gaminghelps, levanon2022generalized, krikamol2024causal, chen2025withheldrecourse}.
Our work most closely follows papers that consider both incentive alignment and robust prediction through the lens of causal models \citep{shavit2020learningFS, rosenfeld2020lookahead}.
Most similar to this work is \citet{taleoftwoshifts}, who show that causal classifiers face covariate shift while general classifiers face shifts due to predicted agents gaming. 
We also build upon \citet{miller2020strategiccausal} who show that incentivising improvement requires learning causal features.
We go beyond these results by establishing a richer range of implications of causal classifiers for strategic classification, from optimality under bounded ambiguity to robustness under no strong assumptions. We further link causal classification to aligned incentives. \citet{milli2019socialcost} show that strategic classification harms social welfare for the predicted, when there is no causal relation between features and outcomes. \citet{Somerstep_causal_twosided_markets} study agents who can directly manipulate their outcome, identifying conditions where they are positively impacted by strategic classification, in labour markets. We study social welfare when agents can manipulate their features, and outcomes change indirectly via a causal mechanism.

Our work also builds extensively on a long line of work on causality and out-of-distribution (OOD) generalization starting with \citet{Peters2016_invariant_prediction,peters2017elements,heinze2018invariant} that establishes that causal models have bounded OOD risk since they remain invariant to changes in the feature distribution. 
Much of work in this area focuses on using data from multiple distributions to discover causal models based on the invariance principle \citep{Arjovsky2019_invariant_risk_minimization,Perry2022_sparse_mechanism_shift,Eastwood2022_qrm,rojascarulla2018invariant}. 
However, \citet{magliacane2018causalbias} go further and show that regression under domain shift entails a trade-off: restricting to invariant features  guarantees robustness but may sacrifice predictive information, while using all features risks unbounded error in the target domain. 
Our work follows mostly closely from this contribution. We establish this tradeoff in the case of strategic classification, analyzing directly the post-adaptation cross entropy loss.

\begin{figure}[t]
    \centering \includegraphics[width=0.55\linewidth]{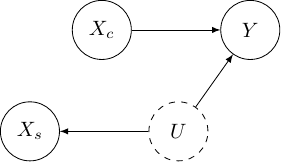}
    \caption{Causal graph for our data generating process. $U$ is an unobserved confounder between outcome $Y$ and spurious feature $X_s$. $X_c$ is a direct cause of $Y$.}
    \label{fig:causal_graph}
\end{figure}

\section{SETTING}

In strategic classification, we consider institutions that deploy classifiers to make decisions, and predicted agents that strategically respond to these classifications.
We start by considering the causal model in \Cref{fig:causal_graph}, which relates $d$ input features and the binary outcome $Y$ realized by a predicted agent.  
In this model, a set of \textbf{causal features} $X_c \in \mathcal{X}_c\subseteq \mathbb R^{d_c}$ directly influence the value of the outcome $Y \in \{0,1\}$ while \textbf{spurious features} $X_s \in \mathcal{X}_s\subseteq \mathbb R^{d_s}$ are only predictive of the outcome $Y$ due to confounding from an unobserved variable $U \in \mathcal{U}$. 
The binary outcome $Y$ is:
\begin{align}
    \begin{split}
        Y &= 1_{\{y_{\mathrm{sco}}(x_c, u) \geq 0\}} \quad \text{or}\\
        Y &\sim \mathrm{Bernoulli}\big( \sigma(y_{\mathrm{sco}}(x_c, u))\big).
    \end{split}
\end{align}
That is, the outcome is a deterministic or stochastic (potentially nonlinear) function of its causal parents, as specified by the real-valued function $y_{\mathrm{sco}}(x_c, u)$, and $\sigma$ refers to the standard sigmoid function. We refer to the set of all observed features as $X \in \mathcal{X}\subseteq \mathbb R^d$.
This assumed causal model captures many decision-making scenarios. Consider predicting whether an applicant can pay back their loan: an agent's income might be causally linked to loan defaulting (the outcome), while an agent's level of education might only spuriously predict defaulting due to latent common causes like socio-economic status. 

Typically, the institution seeks a classifier $f:\mathcal X\rightarrow \mathbb{R}$ based on \emph{all available features} and make decisions using the function $h: \mathcal X \rightarrow \{0,1\}$ such that,
\begin{align}
    \begin{split}
        \hat y = h(x) &= 1_{\{f(x) \geq 0\}}
    \end{split}
\end{align}
In standard classification, we want a classifier $\hat f(x)$ that minimizes classification errors (i.e., 0-1 loss),
\begin{align}
    \hat f \in \underset{f\in\mathcal F}{\arg\min} \underset{x,y}{\mathbb E}[\ell_{0-1}( 1_{\{f(x) \geq 0\}}, y)].
\end{align}
Since the 0-1 loss is costly to minimize directly, we typically deploy classifiers $\hat f(x)$ that minimize a surrogate loss function, e.g., the cross entropy loss that we study later. 
We also distinguish between two families of classifiers in this work: $\mathcal{F}$, the family of classifiers that use all available features, and $\mathcal{F}_{\mathrm{causal}} \left(\subset \mathcal{F}\right)$ which is a subset of classifiers that only use the causal features (essentially masking the spurious ones). 

In the strategic classification setting, however, after the institution deploys the $\hat f(x)$ and its corresponding decision rule $h(x)$, the predicted agents respond to the classification they receive. Each predicted agent gains utility $\delta$ from a positive prediction ($\hat y =1$), and 0 utility from $\hat y=0$. Hence, a utility-maximizing agent spends at most a budget $\delta$ to flip $\hat y =0$ into $\hat y =1$. Expanding on the typical assumption that $\delta$ is static, we provide results that are valid both for a static $\delta$ and for any $\delta'>\delta$, i.e. under a shift in the agents' budget. This represents, for instance, malicious or highly motivated actors -- such as an applicant who has a lot to gain by getting a loan approved -- who will go to greater lengths to adjust their features than what is observed in historical training data. Given knowledge of $\hat f$ and a cost function $c(x,x'):\mathcal X\times\mathcal X\rightarrow \mathbb R_{\geq 0}$, agents with $\hat y=0$ compute the cheapest intervention over $x$ that flips their prediction:
\begin{align}
\begin{split}
\label{eq:agent_adaptation}
    \Delta_{h}(x;\delta):  \mathcal X \times \mathcal F\times \mathbb R_{\geq 0}\rightarrow \mathcal X \triangleq \\
    \underset{x'\in \mathcal{X}}{\arg\max} \ \,\delta h(x')-c(x,x')
\end{split}
\end{align}

In the rest of the paper we assume a tie-breaking rule, making this solution set a singleton. 
Since predicted agents can take actions to change their features in order to improve their outcome, as described above, classifiers induce new distributions over the population. We denote the \textit{post-adaptation distribution} (after one step of adaptation) as $\mathcal{D}^{(f, \delta)}$, where $\mathcal{D}$ is the original distribution, $f$ is a classifier that induces the adaptation, and $\delta$ is the agents' budget. Formally, $\mathcal{D}^{({f}, \delta)}$ is obtained by mapping each point $(x,y)$ from $\mathcal{D}$ to its adapted counterpart, i.e. \begin{align}(x,y) \mapsto (\Delta_h(x), y(\Delta_h(x),u)),\end{align} 
where the notation $y(\Delta_h(x),u)$ is the counterfactual outcome after adaptation given observed features $x$ and the unobserved value $u$. 
In the running example, this counterfactual reflects whether a particular loan applicant would repay their loan if they increased their salary from $x$ to $x'$, given their unobserved socioeconomic status $u$. 

An institution that is \textbf{strategic} anticipates the responses of the predicted agents, and seeks a classifier $f^*(x)$ that minimizes the classification error after \emph{one step of adaptation.} That is, the institution wants to minimize their post-adaptation 0-1 loss,
\begin{align}
   \mathbb E_{x,u} [\ell_{0-1}(h(\Delta_h(x)),y(\Delta_h(x),u))].
\end{align}

\textbf{Main idea.} The goal of this paper is to shed light on how the causal model (in \Cref{fig:causal_graph}) of the outcome is related to the optimal post-adaptation classifier $f^*(x)$.
At a high-level, we leverage two key insights to develop results about the role of causal features in strategic classification. First, when predicted agents change their causal features, they can \textbf{improve} their post-adaptation outcome $y(\Delta_h(x),u)$ \citep{miller2020strategiccausal}. Second, causal features are related to the outcome $y$ by the mechanism $y_\mathrm{sco}(x_c, u)$ that remains invariant to distribution shifts. In contrast, the relationship between spurious features and the outcome changes as predicted agents adapt these features~\citep{magliacane2018causalbias}.
We show that by using causal features, institutions can obtain good post-adaptation loss and align their incentives with those of the predicted agents.

\section{OPTIMALITY OF CAUSAL CLASSIFIERS AFTER ADAPTATION}
\label{section_01_opt}

We start by considering the setting where the outcome $Y$ is a deterministic function of the causal features $X_c$ and the unobserved variable $U$, i.e.,
$y=\mathds 1\{y_\mathrm{sco}(x_c,u) \geq 0\}$.
For ease, in this section, we will directly refer to the decision function $h(x) = 1\{f(x) \geq 0\}$ as the classifier. 
In this setting, we study the impact for an institution when it deploys a causal classifier: a classifier that uses only the causal features $X_c$, variables that are invariant predictors of the outcome. We prove that, under some assumptions, causal classifiers $h(x_c)$ lead to zero post-adaptation 0-1 loss ($\underset{x_c,u}{\mathbb{E}} [\ell_{0-1}(h(\Delta_h(x_c;\delta)),y)]=0$). Furthermore they remain optimal to any budget $\delta'>\delta$, and are in this sense robust to all large-enough adaptations.
To show this result, intuitively, we note that the unobserved variable $U$ creates ambiguity in the outcome values. If this ambiguity can be limited to a specific region in the input space $\mathcal{X}\subseteq \mathbb R^d$, and predicted agents have a sufficient budget $\delta$ to move out of this region, then by deploying a causal classifier $h(x_c)$, an institution induces previously misclassified agents to improve into true positives. Conversely, if the institution deploys a classifier that uses spurious features, predicted agents can waste their adaptation budget $\delta$ on feature changes that do not affect their true outcome value, thereby gaming and only contributing to false positives post-adaptation. We show that this optimality holds even when we consider post-adaptation cross-entropy loss.

We first define the region of causal feature space $\mathcal{X}_c$ where the outcome $Y$ is uncertain, to characterize its boundaries and conditions to move points outside it. 

\begin{definition}
    (Domain with ambiguous outcome) 
    The domain with ambiguous outcome, $\mathcal{X}_{\mathrm{ambiguous}} \subseteq \mathcal{X}_c$, is the subset of causal features where the unobserved $u$ can flip the outcome's sign. A point $x_c$ is considered ambiguous if we can find two latent states, $u, u' \in \mathcal{U}$, such that $y_{\mathrm{sco}}(x_c, u) \geq 0$ but $y_{\mathrm{sco}}(x_c, u') < 0$.

\end{definition}

In what follows we'll study the scenario where this region is bounded, i.e., $\mathcal X_c\setminus \mathcal X_\mathrm{ambiguous}\not= \varnothing$.
We begin by introducing an assumption on the distance from any ambiguous $x_c$ to a non-ambiguous $x_c$, in $L_p$-norm.

\begin{assumption}
\label{ass:ambiguity_comp}
    (Ambiguity compensation by $\delta$) We assume that any ambiguous point can be pushed to a non-negative outcome using a perturbation of bounded size. Formally, there exists a finite constant $\delta \ge 0$ such that for any ambiguous feature $x_c \in \mathcal{X}_{\mathrm{ambiguous}}$, there is a vector $v \in \mathbb{R}^d$ bounded by $\|v\|_p \leq \delta$ that satisfies $y_{\mathrm{sco}}(x_c+v, u) \geq 0$ for all $u$.
\end{assumption}

The intuition for this assumption is that each data point has the possibility to improve its true outcome given enough effort. Improvement is only possible when effort is applied to features $x_c$ with a causal impact on the outcome $y$. Alternatively, applying effort to features $x_s$ that do not affect $y$ but change the prediction consists in gaming. This is close in spirit to recourse, but our assumption applies to the true generative process whereas in recourse it is with respect to the classifier.

To help characterize an optimal classifier, we rely on the concept of an $O_s$-nondecreasing function \citep{boyd2004convex} defined with respect to an orthant $O_s$ in $\mathbb{R}^d$ (where an orthant is a subset of $\mathbb{R}^d$ as detailed in Appendix~\ref{app:opt_01_high_shift}). We define a partial ordering such that $x \preceq_{O_s} y$ if and only if the difference $y-x$ belongs to $O_s$. A function $f$ is then considered $O_s$-nondecreasing when $x \preceq_{O_s} y$ implies $f(x) \leq f(y)$. Similarly, we write $x \prec_{O_s} y$ to indicate strict inequality, meaning $y-x$ lies in the interior of the orthant, denoted $\mathrm{int}(O_s)$.

\begin{figure}
    \centering
    \includegraphics[width=0.9\linewidth]{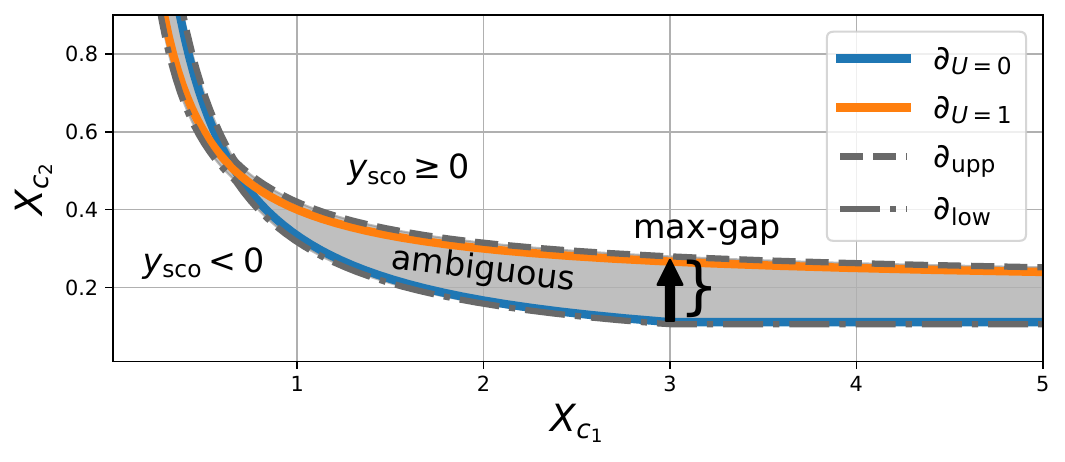}
    \caption{
    Example of bounds $\partial$ for an O-nondecreasing $y_\mathrm{sco}(x_c,u)$. Here $O_s$ is the positive orthant, so increasing either $X_{C_1}$ or $X_{C_2}$ leads to an improvement in $y_\mathrm{sco}$. The unobserved $U\in\{0,1\}$ is a binary variable in this example. $\partial_{U=0}$ and $\partial_{U=1}$ show the boundary for each $u$, where $y_\mathrm{sco}$ flips between negative and positive values. From them we can determine the overall boundaries $\partial_\mathrm{upp}, \partial_\mathrm{low}$, delimiting the ambiguous region where $U$ can flip the sign of $y_\mathrm{sco}$. The \textit{max-gap} is determined from $\partial_\mathrm{upp}$ and $\partial_\mathrm{low}$.}
    \label{fig:bounds}
\end{figure}

Assuming $y_\mathrm{sco}$ is $O_s$-nondecreasing with respect to $x_c$, we define the decision boundary $\partial_u$ for a specific latent state $u$ as the set of features $x_c$ where the score is exactly zero ($y_{\mathrm{sco}}(x_c,u)=0$).
Note that while $y_\mathrm{sco}$ does not need to be $O_s$-nondecreasing with respect to $u$, the orthant $O_s$ must remain the same across all values of $U$. We then let $B = \bigcup_u \partial_u$ denote the union of all such boundaries over $u$. This allows us to define the bounds of the ambiguous domain:

\begin{definition}
    (Upper and lower bound of $\mathcal{X}_\mathrm{ambiguous}$)
    We define the upper bound of the ambiguous domain, denoted $\partial_\mathrm{upp}$, as the subset of boundary points $x_\mathrm{upp} \in B$ for which no strictly larger point belongs to $B$. Formally, the intersection of the shifted interior orthant $\mathrm{int}(x_\mathrm{upp}+O_s)$ with the boundary union $B$ is empty. 
    Similarly, the lower bound $\partial_\mathrm{low}$ is the subset of boundary points $x_\mathrm{low} \in B$ for which no strictly smaller point belongs to $B$, meaning the intersection $\mathrm{int}(x_\mathrm{low}-O_s) \cap B$ is empty.
\end{definition}

Using the previous assumptions together with the continuity of $y_\mathrm{sco}$, we can show that $\partial_\mathrm{upp}$ and $\partial_\mathrm{low}$ separate the ambiguous domain from two non-empty regions of definitive outcomes.

\begin{lemma}
\label{lem:upp-low-bound-sec4}
    (Partition of $\mathcal X_c$ through $\partial_\mathrm{upp}$ and $\partial_\mathrm{low}$) 
    Assume $y=\text{sign}(y_{\text{sco}})$, the score $y_{\text{sco}}$ is continuous and $O_s$-nondecreasing with respect to $x_c$, and Assumption~\ref{ass:ambiguity_comp} holds.  
    
    The bounds $\partial_\mathrm{upp}$ and $\partial_\mathrm{low}$ partition the causal feature space $\mathcal{X}_c$ into three disjoint, non-empty subsets: the ambiguous domain $\mathcal{X}_\mathrm{ambiguous}$, an unambiguously positive region, and an unambiguously negative region. Specifically, for any point $x_c \in \mathcal{X}_c$:
    \begin{enumerate}
        \item If there exists a point $x_\mathrm{upp} \in \partial_\mathrm{upp}$ such that $x_\mathrm{upp} \preceq_{O_s} x_c$, then $x_c$ is outside the ambiguous domain and yields a non-negative score for all latent states ($y_\mathrm{sco}(x_c,u) \geq 0$ for all $u$).
        \item If there exists a point $x_\mathrm{low} \in \partial_\mathrm{low}$ such that $x_c \prec_{O_s} x_\mathrm{low}$, then $x_c$ is outside the ambiguous domain and yields a strictly negative score for all latent states ($y_\mathrm{sco}(x_c,u) < 0$ for all $u$).
    \end{enumerate}
    The proof is provided in Appendix~\ref{app:opt_01_high_shift}, Corollary~\ref{cor:two_bounds_app}.
\end{lemma}

Additionally, we define the \textit{max-gap} as the maximum $L_p$ distance from any point in the ambiguous domain to the upper bound $\partial_\mathrm{upp}$. Formally, this is given by $\underset{u}{\max}\,\underset{x_\mathrm{low}\in \partial_u}{\max}\,\underset{x_\mathrm{upp}\in\partial_\mathrm{upp}}{\min}\,\|x_\mathrm{upp} - x_\mathrm{low}\|_p$. In the following lemma, we establish that this gap is strictly bounded:

\begin{lemma}
    (Bounded max-gap in $\mathcal X_\mathrm{ambiguous}$) 
    For every point $x_c$ in the ambiguous domain $\mathcal{X}_\mathrm{ambiguous}$, the shortest $L_p$ distance to the upper bound $\partial_\mathrm{upp}$ is at most the max-gap, which is guaranteed to be finite. That is, for any $x_c\in\mathcal X_\mathrm{ambiguous}$:

    \begin{gather*}    \underset{x_\mathrm{upp}\in\partial_\mathrm{upp}}{\min}||x_\mathrm{upp}-x_c||_p\leq
    \\
    \underset{u}{\max}\,\underset{x_\mathrm{low}\in \partial_u}{\max}\,\underset{x_\mathrm{upp}\in\partial_\mathrm{upp}}{\min}\,||x_\mathrm{upp}
    -x_\mathrm{low}||_p<+\infty
    \end{gather*}
    
    The proof is provided in Appendix~\ref{cor:bounded_dist}.
\end{lemma}

Note that this result is well-defined following Lemma \ref{lem:upp-low-bound-sec4}, which shows that the boundaries $\partial_{\text{upp}}, \partial_{\text{low}}$, and each specified subset of $\mathcal{X}_c$ are non-empty sets. We show a visualization of
the concept of $\mathcal{X}_\mathrm{ambiguous}$ and its boundaries in \Cref{fig:bounds} for the setting of $\mathcal{X}_c=\mathbb R^2$ and $\mathcal U=\{0,1\}$.

Further, assuming the cost of adapting features is an $L_p$-norm, we prove there exists a causal classifier which achieves zero $\ell_{0-1}$ post-adaptation, for all adaptations where $\delta$ is high enough. This is achieved by moving points away from $\mathcal{X}_\mathrm{ambiguous}$ and anticipating the adaptation's impact on post-adaptation outcome $y$:

\begin{theorem}
    (Causal $\ell_{0-1}$-optimality) 
    \label{the:causal_01opt}
    Suppose the assumptions of Lemma~\ref{lem:upp-low-bound-sec4} hold and the adaptation cost is an $L_p$-norm. Then there exists a finite threshold $e \in \mathbb{R}$ and a causal classifier $h_c \in \mathcal{H}_\mathrm{causal} \subset \mathcal{H}$ (whose outputs are unaffected by spurious features $x_s$) such that for any adaptation budget $\delta \geq e$, the expected post-adaptation $0$-$1$ loss is strictly zero. That is,
    \[
        \mathbb E_{x,u}\left[\ell_{0-1}\left(h_c(\Delta_{h_c}(x;\delta)),\mathrm{sign}(y_{\mathrm{sco}}(\Delta_{h_c}(x;\delta),u))\right)\right] = 0.
    \]
    The proof is provided in Appendix~\ref{app:opt_01_high_shift}, Theorem~\ref{the:causal_01opt_app}.
\end{theorem}

The essence of this result is that the causal classifier works by being overly demanding but incentivising false negatives in $\mathcal{X}_\mathrm{ambiguous}$ to adapt, turning them into true positives. Intuitively, a spurious classifier wastes effort in the sense that users use their budget to game (introducing false positives) rather than improve to true positives.
Hence this post-adaptation optimal classifier $h_c(x)$ remains optimal even if agents adapt with $\delta^\prime > \delta$ at a subsequent time point.  This is further validated in the empirical studies, and in Appendix~\ref{app:example_section} we also illustrate this mathematically with an example, showing that the causal classifier characterized by Theorem~\ref{the:causal_01opt} can remain optimal even outside
the range $\delta\in[\mathrm{max\textit{-}gap},+\infty)$ proved here. In that example we show a phase transition, where a causal classifier suddenly becomes optimal as $\delta$ increases.

Additionally, we note that the same optimality result holds for cross-entropy in this setting. For a learned probability estimator $\hat f(x):\mathcal{X}\rightarrow [0,1]$, define $\ell_\mathrm{CE}(\hat f(x), y)\triangleq - \big[ y \log \hat f(x) + (1-y) \log (1-\hat f(x)) \big]$.

\begin{corollary}
\label{the:the:causal_CE_opt_detY}
    (Causal cross-entropy optimality under bounded $\mathcal{X}_\mathrm{ambiguous}$) Under the assumptions of Theorem~\ref{the:causal_01opt}, we have zero $\ell_\mathrm{CE}(\hat f(x),y)$ for all post-adaptation points, using a large-enough $\delta$. Proof in Appendix~\ref{app:opt_CE_high_shift_detY}.
\end{corollary}

\section{ROBUSTNESS OF CAUSAL CLASSIFIERS IN STOCHASTIC SETTINGS}
\label{sec:ce_robust}

We now consider the setting where the outcome $Y$ is a stochastic function of the causal features $X_c$ and the unobserved variable $U$, i.e., there is some function of causal parents $g(x_c, u) \triangleq P(Y=1 \vert x_c, u) = \sigma(y_\mathrm{sco}(x_c, u))$.
Unlike in the previous setting where we consider a bounded ambiguous region where the latent variable $U$ influences the value of the outcome $Y$, here, without further assumptions, the latent variable is informative about the outcome in all regions of the input.
This means that we cannot simply deploy a causal classifier $f \in \mathcal F_{\mathrm{causal}}$ to move predicted agents out of the ambiguous region to achieve optimal loss post-adaptation.
Nevertheless, in this setting, we show a different advantage of causal classifiers: if we only have access to historical data before observing agents' adaptation, by training causal classifiers, we can incur bounded CE loss after the classifier is deployed, while spurious classifiers risk arbitrarily bad post-adaptation CE loss.

We start by defining optimal classifiers,
\begin{align}
\label{eqn:optimal_class}
\begin{split}
    \hat{f} &= \underset{f \in \mathcal{F}}{\arg\min} \ \mathbb{E}_{D}\left[\mathrm{KL} \left(g(x_c, u) \| f(x) \right) \right]\\ 
    \hat{f}^*_\delta &= \underset{f^* \in \mathcal{F}}{\arg\min} \ \mathbb{E}_{\mathcal{D}^{(\hat{f}, \delta)}}\left[\mathrm{KL} \left(g(x_c, u) \| f^*(x) \right) \right]
\end{split}
\end{align}
The classifier $\hat f$ minimizes the cross entropy loss on the training data -- that is, the historical data obtained before agents adapt to a deployed classifier.
The classifier $\hat f^*$ refers to the classifier that minimizes cross entropy loss on the distribution \emph{induced by deploying the classifier $\hat f$}. That is, we consider samples from the distribution $\mathcal{D}^{(\hat{f}, \delta)}$.

With classifiers defined this way, we can now analyze the impact of minimizing CE loss on \textit{training} data to generalize to samples that result from predicted agents adapting to classifications. First, we show that the CE loss post-adaptation follows a decomposition that gives us insights into the behavior of causal classifiers. 

\begin{lemma}
   (Decomposition of CE loss post-adaptation, informal) Under mild assumptions on measurability and support of the classifiers,
   post-adaptation cross-entropy loss of a classifier $\hat{f}$ that was trained on pre-adaptation data with family $\mathcal{F}$ can be decomposed as: 
    \begin{align*}\begin{split}\mathcal{L}_{\mathrm{CE}}(\hat{f}, \delta) 
&=  \underbrace{\mathbb{E}_{ \mathcal{D}^{(\hat{f}, \delta)}} \left[\mathrm{KL} \left(g(x_c, u) \| \hat{f}^*_\delta(x_c, x_s) \right) \right]}_{\text{incomplete information error}} 
\\[1ex] &+ 
   \mathbb{E}_{ \mathcal{D}^{(\hat{f}, \delta)}} \left[
     g(x_c, u) \log\frac{\hat{f}^*_\delta(x_c, x_s)}{\hat{f}(x_c, x_s)}
    \right. 
    \\ &  \underbrace{\hspace{4em}\left.+ (1-g(x_c, u)) \log\frac{1-\hat{f}^*_\delta(x_c, x_s)}{1-\hat{f}(x_c, x_s)}
   \right]}_{\text{transfer error}} 
   \\[1ex] &+
  \underbrace{\mathbb{E}_{\mathcal{D}^{(\hat{f}, \delta)}} \left[H \left(g(x_c, u) \right) \right]}_{\text{entropy of post-adapt distribution}}\end{split}
    \end{align*}

    where $H(X)$ is the entropy of $X$. 
   \label{lem:ce_decomp}
\end{lemma}
We derive this decomposition in Appendix~\ref{app:CE_loss_decomposition}.
Here, we extend results from \citet{magliacane2018causalbias} who derive a similar tradeoff when considering the bias of classifiers that are trained on different distributions.
In the next result, we consider what happens when we restrict our hypothesis class to functions $f \in \mathcal F_{\mathrm{causal}}$ versus when we consider all functions including those that use spurious features.

\begin{theorem}
    \label{the:robustness}
    (Robustness of causal classifiers) For a classifier $\hat{f}$ in the family $\mathcal{F}_{causal}$ of causal classifiers, assuming that it is the optimal classifier in the sense of \Cref{eqn:optimal_class}, the post-adaption cross-entropy loss $\mathcal{L}_{\mathrm{CE}}(\hat{f}, \delta)$ is bounded by the sum of entropy terms $H(U) + H(Y\vert X)$. The post-adaptation loss of an optimal classifier $\hat{f}$ in the family of spurious classifiers $\mathcal{F}_{all}$ 
    cannot be bounded due to non-zero transfer error.

\end{theorem}
We derive this result in \Cref{app:ce_error_analysis_causal,app:ce_error_analysis_spurious}.
When we consider classifiers that only use causal features, i.e. $f \in \mathcal{F}_{causal}$, the pre-adaptation optimal classifier $\hat f$ is the same as the post-adaptation optimal classifier $\hat f^*_\delta$, avoiding transfer error. 
Intuitively, this is because the causal relationship between $X_c$ and the outcome $Y$ remains invariant to changing the feature distribution $\mathbb{P}(X_c)$, a fact that we formalize in \Cref{app:transfer_error_causal}.
However, the optimal causal classifier $\hat f$ incurs incomplete information error by discarding spurious features,
which is equivalent to the conditional mutual information between the latent variable $U$ and $Y$ given causal features $X_c$, and bounded by the entropy of $U$. 

In contrast, in \Cref{app:ce_error_analysis_spurious} we show that a classifier $\hat f$ that uses all the features avoids some incomplete information error (bounded by the entropy of $U$ conditioned on $X_s$ in the post-adaptation distribution). This is because even after predicted agents intervene on their spurious features $X_s$, changing its association with the latent variable $U$, it may still remain informative about $U$ and thus, the outcome $Y$ (which relies on the value of $U$).
However, a classifier that uses all features incurs transfer error exactly because of the intervention that agents perform on their features: by changing their spurious features $X_s$, they can arbitrarily change how well the outcome is predicted by these features. 
Thus, the transfer error of a spurious classifier cannot be bounded. In the  \Cref{app:transfer_error_spurious_analysis}, we further analyze thresholded classifiers and establish conditions under which transfer error can be made arbitrarily large.

Note that defining the classifiers $\hat{f}$ and $\hat{f}^*$ as optimal classifiers before and after adaptation is key to this interpretation, since if these classifier were not optimal, the transfer error could be nonzero even for a causal classifier. 
Finally, note that this analysis can be extended to consider the post-adaptation loss of classifiers $\hat f$ that are trained on data after historical deployments where agents had a different budget $\delta^\prime$, which we show in \Cref{app:ce_decomp_extend}.

\section{INCENTIVE ALIGNMENT}

We now consider the question of how agents are impacted in a strategic setting and whether their interests are aligned with the institution. If $Y$ is static (i.e. there is no causal effect of $X$ on $Y$) and agents seek positive predictions $h(\Delta_h(x))=1$, \citet{milli2019socialcost} show that strategic classification harms utility of predicted agents.
\citet{levanon2022generalized} study the case where alignment is built-in -- the predicted agents gain from accuracy just like the institution, and $Y$ is static. \citet{miller2020strategiccausal} provide an equivalency between agent improvement and causal discovery, but leave unclear whether the agents and/or the institution benefit from this improvement. This leaves unanswered the question of whether strategic classification can benefit the predicted agents, when $X$ causally affects $Y$ and agents seek positive predictions $\hat Y=1$.

We begin by defining the long-term goals of the institution and of the predicted agents. An agent with a positive prediction ($h(\Delta_{h}(x))=1$) has an immediate gain, but may suffer a loss in the long-term if its true outcome is negative ($h(\Delta_{h}(x))=1\land y(\Delta_{h}(x),u)=0$). For instance an agent can get a home loan approved but later lose it by defaulting, or be accepted to college but then fail courses.
Analogously an institution aims to minimize $\ell_{0-1}$, but in the long-term may benefit more from true positives (TPs) than true negatives (TNs) — banks need to identify good borrowers and universities need good students.
We now define long-term rewards or utilities.

\begin{definition}
\label{def:longterm_pred_util}
    (Predicted agents' long-term utility)
Agents gain $\delta$ from a positive prediction, and we let $\delta_2$ denote the loss when agents obtain a false positive.  Setting $\delta_2=0$ represents a short-term goal.

    \begin{align*}
     r_p(h,x,u)=&\delta h (\Delta_{h}(x))-c(x,\Delta_{h}(x))
     \\& -\delta_2 h(\Delta_{h}(x))(1-y(\Delta_{h}(x),u))  
    \end{align*}
\end{definition}

\begin{definition}
\label{def:longterm_inst_util}
    (Institution's long-term utility)
    Institution gains from lowering post-adaptation 0-1 loss $\ell_{0-1}$ and, among correct predictions, prefers true positives over true negatives over the long term.  We call this an $\epsilon$-advantage where $\epsilon$ denotes a loss in the institution utility due to true negatives.  When $\epsilon=0$, the utility represents a short-term goal.
    \begin{align*}
        r_i(h,x,u)=&- \mathds{1}\{h(\Delta_{h}(x))\neq y(\Delta_{h}(x),u)\}
        \\& -\epsilon \mathds{1}\{h(\Delta_{h}(x))=y(\Delta_{h}(x),u)=0\}
    \end{align*}
\end{definition}

We are interested in understanding how these two goals interact in the strategic setting, where agents react to predictions by modifying $X$ and the institution anticipates agent modifications.  For this analysis we introduce the notion of $h$-change, the expected change in utility when the classifier is changed.

\begin{definition}
    ($h$-change) Expected change in utility when switching from classifier $h$ into $h'$, for role $k$, where $k\in\{p,i\}$ is predicted agent $p$ or institution $i$:
    \begin{align*}
        \Delta r_{k}(h',h)&= \mathbb E_{x,u}
        [r_{k}(h',\Delta_{h'}(x),u)]
        \\& \quad -\mathbb E_{x,u}[r_{k}(h,\Delta_{h}(x),u)]
    \end{align*}

\end{definition}

Before we consider the case of strategic agents, we build intuition by first considering static agents that can never adapt, $\Delta_h(x)=x$ ($c(\cdot)\rightarrow+\infty$), and the setting of short-term goals ($\delta_2=\epsilon=0$).

\begin{proposition}
    (Static alignment) Let $\Delta r_{k|Y=y}(h',h)$ be the $h$-change for the subpopulation $Y=y$ (\ref{def:cond_h_change}). Assume $\Delta_h(x)=x$ and $\delta_2=\epsilon=0$. For any pair of classifiers $(\hat f',\hat f)$, institution's goals match the goals of agents whose $Y=1$ but not whose $Y=0$, in the following sense (proof in Appendix~\ref{pro:static_align}):
\begin{align*}
    \Delta r_{p|Y=1}(h',h)>0\iff \Delta r_{i|Y=1}(h',h)>0
    \\
    \Delta r_{p|Y=0}(h',h)<0\iff \Delta r_{i|Y=0}(h',h)>0
\end{align*}
\end{proposition}
Intuitively, agents in the $Y=0$ group are either TNs or false positives (FPs). Since in this static setting, agents cannot improve their outcomes, institution only gains utility from an $h^\prime$ that switches FPs into TNs, which strictly lowers agents' utility.

We now consider the more general case where agents can adapt $\Delta_h(x)\neq x$, and hence change $y(\Delta_h(x),u)$. The institution becomes strategic when it switches from a classifier that wrongly assumes agents are static ($h^{\text{pre}}$), into one considering agents' adaptation ($h^{\text{post}}$). Note that, by definition, \mbox{$\Delta r_i(h^{\text{post}},h^{\text{pre}})\geq0$}.
There is alignment if, as the institution becomes strategic, predicted agents also benefit.

\begin{definition}
\label{def:aligned_incentives}
    (Aligned incentives) Consider $h^{\text{pre}}$, the classifier that maximizes
    $\mathbb E_{x,u}[r_i(h,x,u)]$ 
    wrongly assuming that $\Delta_{h}(x)=x$, and $h^{\text{post}}$
    maximizing
    $\mathbb E_{x,u}[r_i(h,\Delta_h(x),u)]$ 
    with the correct $\Delta_{h}(x)$ (i.e. post-adaptation). We say that incentives are aligned if:
    \begin{align*}
        \Delta r_p(h^{\text{post}},h^{\text{pre}})\geq0
    \end{align*}
\end{definition}

To study short-term alignment when agents are strategic, we first consider the set of pre-adaptation $x$ that adapted towards a point $\Delta_h(x)$. We define it as its preimage $\Delta^{-1}_{h}(x;\delta):=\{ x'\in\mathbb R^d: \Delta_{h}( x';\delta)=x\}$.

The next lemma shows that by assuming $\mathcal{X}$ has support over all values that could have adapted given $\delta$, short-term goals ($\delta_2=\epsilon=0$), and a flexible enough  $\mathcal H$, fewer points receive $h(x)=1$ post-adaptation.

\begin{lemma}
\label{pro:positive_y_support_maintext}
    (Support over positive predictions)
    For $h(x)=\mathds 1\{f(x)\geq0\}$ let its boundary be $\partial_h=\{x\in\mathcal X:f(x)=0\}$. Assume full support over points that can adapt towards $h^\mathrm{pre}$ and $h^\mathrm{post}$: $\Delta^{-1}_{h^{\mathrm{pre}}}(\partial_{h^{\mathrm{pre}}}) \cup \Delta^{-1}_{h^{\mathrm{post}}}(\partial_{h^{\mathrm{post}}})\subset \mathcal X$, and $\delta_2=\epsilon=0$. For a flexible enough hypothesis family $\mathcal{H}$ we have (proof in \ref{app:dynamic_align}):

    \begin{align*}
       \{x\in\mathcal X: h^{\mathrm{post}}(x)=1\}\subset 
        \{x\in\mathcal X: h^{\mathrm{pre}}(x)=1\} 
    \end{align*}
    
\end{lemma}

Under the assumptions above, less agents receive a positive prediction under $h^\mathrm{post}$ than under $h^\mathrm{pre}$. With Lemma~\ref{pro:positive_y_support_maintext} and the assumptions above, we now show misalignment in short-term goals, where agent utility in the short-term decreases when institution becomes strategic.

\begin{proposition}
    (Short-term misalignment) 
    Let $\Delta r_{p\mathrm{\textit{-}short}}$ be the $h$-change for $\delta_2=0$, and $h^\mathrm{post\textit{-}short}$ be optimal for $\epsilon=0$. Under the assumptions of Lemma~\ref{pro:positive_y_support_maintext}, we have (proof in \ref{app:dynamic_align}):

    \begin{align*}
        \Delta r_{p\mathrm{\textit{-}short}}(h^\mathrm{post\textit{-}short},h^\mathrm{pre})<0.
    \end{align*}
\end{proposition}

Having shown that there is misalignment in short-term goals, we now show that considering long-term goals allows incentives to be aligned. In particular, we show that maintaining a flexible $\mathcal{H}$ with $\epsilon=0$ for institution, and a high-enough $\delta_2$ for agents, enables alignment. Consider all four post-adaption cases, when switching from $h^\mathrm{pre}$ into a more demanding $h^\mathrm{post}$ (\ref{app:long-term-align}):

\begin{itemize}
    \item (maint) Maintained FP or TP at higher cost;
    \item (impr) Switched from gaming to improvement (FP$\rightarrow$ TP);
    \item (TP$\rightarrow$ N) Switched from TP into FN or TN;
    \item (FP$\rightarrow$ TN) Switched from FP into TN.
\end{itemize}

For each of these cases, denote their densities as: $\mathbb P(\mathrm{maint}):=\int_{x\in \{\mathrm{maint}\}} \mathbb P(x)$, and similarly $\mathbb P(\mathrm{impr})$, $\mathbb P(\mathrm{TP\rightarrow N})$, $\mathbb P(\mathrm{FP\rightarrow TN})$. Define $ \Delta c(x)\triangleq c(x,\Delta_{h^{\mathrm{post}}}(x))-c(x,\Delta_{h^{\mathrm{pre}}}(x))$. Their total costs are $c(\mathrm{all}):=\int_{x\in \mathcal X} \Delta c(x)\mathbb P(x)$.

Our next result shows a lower bound for the cost of FP, such that agents value enough switching from gaming to improvement, in order to have alignment.

\begin{proposition}
\label{pro:long-term-align}
    (Long-term alignment) Considering a flexible enough $\mathcal H$ and $\epsilon=0$, having alignment requires (proof in \ref{app:dynamic_align}):

    $$\delta_2>\frac{c(\mathrm{all}) 
     +\delta(\mathbb P(\mathrm{TP\rightarrow N})
     +\mathbb P(\mathrm{FP\rightarrow TN}))}
     {\mathbb P(\mathrm{impr})+\mathbb P(\mathrm{FP\rightarrow TN})}$$
\end{proposition}

We provide an example in Appendix~\ref{app:longterm_alignment_example} where alignment occurs for any $\delta_2>c$, for a constant $c<\delta$. If we additionally consider $\epsilon>0$, interactions between $\delta_2$, $\epsilon$ and alignment become more complex. We resort to simulations to illustrate their behaviour in Figure~\ref{fig:incentives_align}, and analyze mathematically an example in Appendix~\ref{app:align_simulations}.

Note that, by our definition, we study alignment of optimal classifiers, which are causal classifiers under the conditions presented in \S~\ref{section_01_opt}. Our theory does not exclude alignment with classifiers using both causal and spurious features, under different conditions. If no causal features are present, \citet{milli2019socialcost} show that alignment is not possible for $\delta_2=0$.

\section{EXPERIMENTS}
\label{sec:xps}

\begin{figure}[ht]
        \centering
          \includegraphics[width=0.95\linewidth]{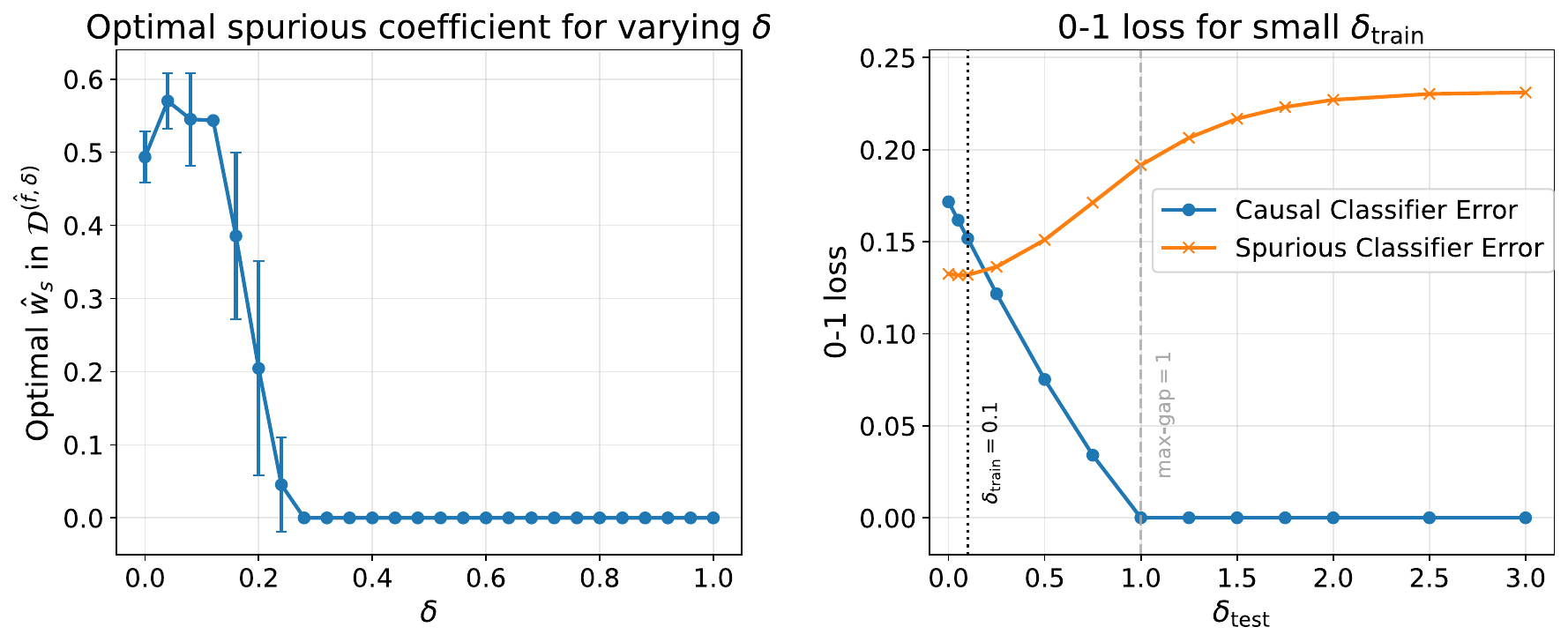}
          \caption{0-1 Loss. Left: As $\delta$ increases, the optimal classifier on simulated post-adaptation data puts progressively less weight on spurious feature $X_s$, until arriving at the optimal causal classifier characterized in Theorem~\ref{the:causal_01opt}. We averaged the optimal weight over three datasets of size $N=20000$ to obtain the error bars. Right: When training and evaluation involve different levels of strategic shift $\delta_{\mathrm{train}}$ and $\delta_{\mathrm{test}}$, a classifier restricted to the causal feature recovers this optimal post-adaptation classifier, while one that exploits the spurious feature leads to poor performance when the shift is larger than anticipated.}
          \label{fig:01_robust}
\end{figure}

We empirically validate our theoretical results with simulated data from the causal model in \Cref{fig:causal_graph} parameterized by a linear outcome and two features: 
\begin{align*}
  u &\sim \mathrm{Bernoulli}(p_{u}) \\
  x_c &\sim \mathcal{N}(0,\sigma_c^2) \\
  x_s &= w_{u\to s}\,u + \varepsilon_s,\quad \varepsilon_s \sim \mathcal{N}(0,\sigma_s^2) \\
    y_\mathrm{sco} &\triangleq  w_c x_c + w_u u + b
    \\
    \text{(det)}\,\,\, y&=\mathds 1\{y_\mathrm{sco}(x_c,u)\}
    \\
    \text{(stoch)}\,\,\, y &\sim \mathrm{Bernoulli}\big( \sigma_{\tau}(y_{\mathrm{sco}}(x_c, u))\big).
\end{align*}

We assume $L_2$-norm for the agents' adaptation cost $c(x',x)\triangleq||x'-x||_2$, and they adapt following \ref{eq:agent_adaptation}.

\textbf{Optimality under bounded ambiguity.}
First, we study \Cref{the:causal_01opt}, which says that with a bounded ambiguous region, a causal classifier can incur zero $\ell_{0-1}$ loss post-adaptation given that agents have enough budget $\delta$ to adapt away from the ambiguous region. 
To validate this result, we plot the coefficients of the classifier $f^* \in \mathcal F$ that achieves optimal post-adaptation $\ell_{0-1}$. Because the post-adaptation data distribution involves computing agents' best response to a classifier, we cannot use gradient-based optimization to minimize post-adaptation loss. Instead, we perform a grid search over the coefficients of linear classifiers $f$ and select the model that minimizes the loss on the post-adaptation distribution generated by some adaptation budget $\delta$.
The left panel in \Cref{fig:01_robust} shows that with sufficient $\delta$, optimal classifiers post-adaptation choose to ignore the spurious feature, validating the result.

\textbf{Robustness to changing $\mathbf{\delta}$.}
The optimality result of \Cref{the:causal_01opt} also holds if, after deployment, agents increase their adaptation budget so that $\delta^\prime > \delta$. We study this in the right panel of \Cref{fig:01_robust}, by performing the same grid search procedure over two classifier families, one that considers all features and another that considers only causal features, both optimized for a fixed adaptation budget we refer to as $\delta_{\mathrm{train}}$. We evaluate the best classifier in each family on data generated by deploying that classifier but with a different, unseen $\delta_{\mathrm{test}}$. We see that when $\delta_{\mathrm{test}} > \delta_{\mathrm{train}}$, causal classifiers remain optimal as expected while spurious classifiers incur error due to differences between the train and test distributions. Additionally, the causal classifier is able to achieve \textit{zero} loss after sufficiently large $\delta_{\mathrm{train}}$ (i.e. $ \geq \mathrm{max\textit{-}gap}$, marked by the light grey line).

In \Cref{app:robustness_add_exp}, we also evaluate the implications of the result on robustness in \Cref{the:robustness}, finding settings where causal classifiers’ zero transfer loss translates into advantages post-adaptation.

\textbf{Incentive Alignment.} We explore how alignment (Definition~\ref{def:aligned_incentives}) is impacted
as we vary properties of agent utility $r_p$ (Definition~\ref{def:longterm_pred_util}) and institution utility $r_i$ (Definition~\ref{def:longterm_inst_util}). Specifically, as we increase $\delta_2$, agents increasingly prefer to avoid false positives in the long term, and as we increase $\epsilon$, institutions prefer true positives over true negatives.
Following Proposition~\ref{pro:long-term-align} we observe in Figure~\ref{fig:incentives_align} there is a minimum $\delta_2$ to have alignment ($\Delta r_p>0$) for $\epsilon=0$. This is due to $h^\mathrm{post}$ preventing more gaming than $h^\mathrm{pre}$ (TN$\rightarrow$FP), which is more valued by agents with higher $\delta_2$. As $\epsilon$ increases, $h^\mathrm{pre}$ avoids TN by increasing FP, since it cannot anticipate agent adaptation. $h^\mathrm{post}$ can avoid TN by encouraging points to improve (TN$\rightarrow$TP) instead of allowing gaming (TN$\rightarrow$FP). A switch from gaming to improvement is valued by agents when $\delta_2$ is higher, whereas low $\delta_2$ makes agents care only about the total count of positive predictions (FP+TP). 
The dynamics between alignment and utilities are further explored in Appendix~\ref{app:align_simulations}.

\begin{figure}[!h]
    \centering
    \includegraphics[width=0.85\linewidth]{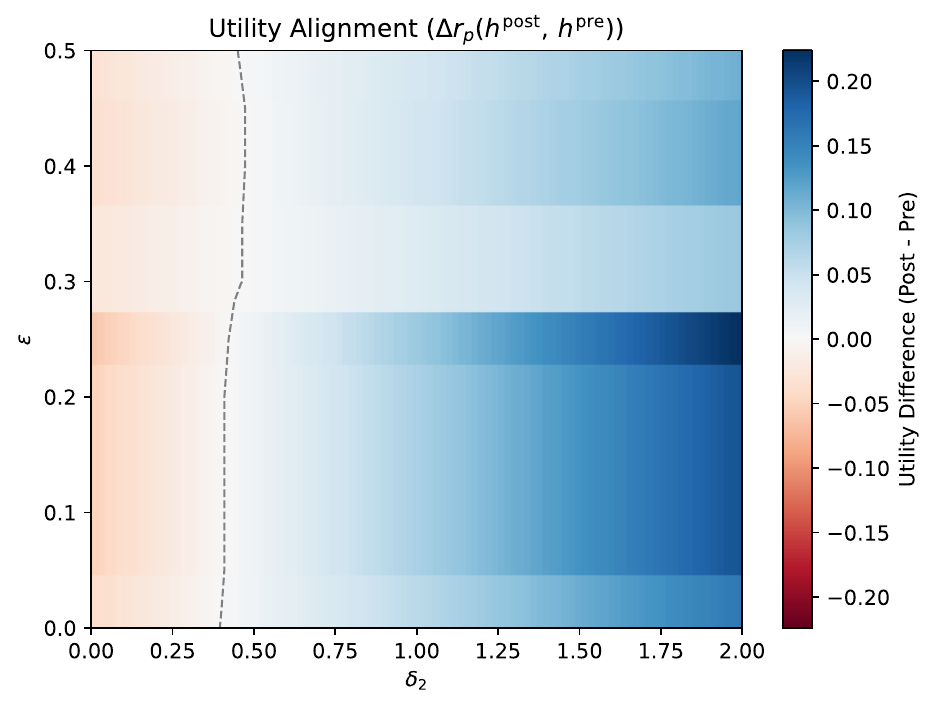}
    \caption{Simulation for $\delta = 0.3$. As the institution anticipates agent adaptation to maximize $r_i$ (i.e., becomes strategic) agents' utility $r_p$ can increase as well, depending on $\delta_2$ and $\epsilon$. The grey dashed line indicates where $r_p$ stays constant, and to its right there is alignment (i.e. an increase in $r_p$).
    }
    \label{fig:incentives_align}
\end{figure}

 While these experiments were performed on fully synthetic data, we also have semi-synthetic results using real-world data for the observed feature in \Cref{app:semi_synthetic_add_exp} which also confirms our theoretical findings hold under real-world observed distributions.

    \label{fig:optimalt_delta}

\section{CONCLUSION}

In this work we characterize the role of causal variables in classification, when agents adapt to predictions. Strategic classification is inherently a causal problem, since whether feature adaptation translates to outcomes depends on the underlying causal model. While previous work connects causality to robustness, we take a step further and identify conditions where causal classifiers are simultaneously optimal to a range of adaptations. Additionally, we show a nuanced picture of how the welfare of predicted agents is affected by strategic classification. While existing work highlights a social burden imposed on the population, we identify conditions where both predicted and predictor can be better off under strategic classification. This is possible due to the causal modeling aspect of our generative process. Assuming the existence of a bounded feature region where outcomes are ambiguous, and an acceptable effort that moves points out of such region, are limitations of our analysis. Therefore, future work should characterize optimality inside ambiguous regions under adaptation. Developing practical methods for strategic classification is also an important line of research. In this direction, designing learning algorithms that encourage data points to move outside ambiguous regions is also a promising avenue for future work.

\subsubsection*{Acknowledgements}
This research was supported in part by the Canada CIFAR AI Chair program, by a grant from Samsung Electronics Co., Ltd., an unrestricted gift from Google, an NSERC Discovery Grant (RGPIN-2023-04869) and the Israel Science Foundation (grant no. 278/22). Simon Lacoste-Julien is a CIFAR Associate Fellow in the Learning in Machines \& Brains program. We would like to thank Pedram Khorsandi for their feedback.

\bibliography{refs}

\begin{thebibliography}{}

\bibitem[Ahmadi et~al., 2020]{ahmadi2020strategicperceptron}
Ahmadi, S., Beyhaghi, H., Blum, A., and Naggita, K. (2020).
\newblock The strategic perceptron.
\newblock {\em CoRR}, abs/2008.01710.

\bibitem[Arjovsky et~al., 2019]{Arjovsky2019_invariant_risk_minimization}
Arjovsky, M., Bottou, L., Gulrajani, I., and Lopez-Paz, D. (2019).
\newblock Invariant risk minimization.
\newblock In {\em arXiv preprint arXiv:1907.02893}.

\bibitem[Bechavod et~al., 2021]{bechavod2021gaminghelps}
Bechavod, Y., Ligett, K., Wu, Z.~S., and Ziani, J. (2021).
\newblock Gaming helps! learning from strategic interactions in natural dynamics.
\newblock In {\em Proceedings of the 38th International Conference on Machine Learning (ICML)}, volume 139, pages 756--765. PMLR.

\bibitem[Boyd and Vandenberghe, 2004]{boyd2004convex}
Boyd, S.~P. and Vandenberghe, L. (2004).
\newblock {\em Convex optimization}.
\newblock Cambridge university press.

\bibitem[Chen et~al., 2025]{chen2025withheldrecourse}
Chen, Y., Estornell, A., Vorobeychik, Y., and Liu, Y. (2025).
\newblock To give or not to give? the impacts of strategically withheld recourse.
\newblock In {\em Proceedings of the 28th International Conference on Artificial Intelligence and Statistics (AISTATS)}. PMLR.

\bibitem[Chen et~al., 2019]{chen2019classifystrategic}
Chen, Y., Liu, Y., and Podimata, C. (2019).
\newblock Grinding the space: Learning to classify against strategic agents.
\newblock {\em CoRR}, abs/1911.04004.

\bibitem[Dong et~al., 2017]{dong2017}
Dong, J., Roth, A., Schutzman, Z., Waggoner, B., and Wu, Z.~S. (2017).
\newblock Strategic classification from revealed preferences.
\newblock {\em CoRR}, abs/1710.07887.

\bibitem[Eastwood et~al., 2022]{Eastwood2022_qrm}
Eastwood, C., Robey, A., Singh, S., von K{\"u}gelgen, J., Hassani, H., Pappas, G.~J., and Sch{\"o}lkopf, B. (2022).
\newblock Probable domain generalization via quantile risk minimization.
\newblock In {\em Proceedings of the 36th Conference on Neural Information Processing Systems (NeurIPS)}.

\bibitem[G{\'o}is et~al., 2025]{gois2025performative}
G{\'o}is, A., Mofakhami, M., Santos, F.~P., Lacoste-Julien, S., and Gidel, G. (2025).
\newblock Performative prediction on games and mechanism design.
\newblock In {\em International Conference on Artificial Intelligence and Statistics}, pages 1855--1863. PMLR.

\bibitem[Hardt et~al., 2016]{hardt2016strategic}
Hardt, M., Megiddo, N., Papadimitriou, C., and Wootters, M. (2016).
\newblock Strategic classification.
\newblock In {\em Proceedings of the 2016 ACM conference on innovations in theoretical computer science}, pages 111--122.

\bibitem[Heinze-Deml et~al., 2018]{heinze2018invariant}
Heinze-Deml, C., Peters, J., and Meinshausen, N. (2018).
\newblock Invariant causal prediction for nonlinear models.
\newblock {\em Journal of Causal Inference}, 6(2).

\bibitem[Horowitz and Rosenfeld, 2018]{taleoftwoshifts}
Horowitz, G. and Rosenfeld, N. (2018).
\newblock Causal strategic classification: A tale of two shifts.
\newblock In {\em Proceedings of the 40th International Conference on Machine Learning (ICML 2023)}.

\bibitem[Kleinberg and Raghavan, 2018]{klienberg2018incentive}
Kleinberg, J.~M. and Raghavan, M. (2018).
\newblock How do classifiers induce agents to invest effort strategically?
\newblock {\em CoRR}, abs/1807.05307.

\bibitem[Levanon and Rosenfeld, 2021]{rosenfeld2021scpractical}
Levanon, S. and Rosenfeld, N. (2021).
\newblock Strategic classification made practical.
\newblock {\em CoRR}, abs/2103.01826.

\bibitem[Levanon and Rosenfeld, 2022]{levanon2022generalized}
Levanon, S. and Rosenfeld, N. (2022).
\newblock Generalized strategic classification and the case of aligned incentives.
\newblock In {\em International Conference on Machine Learning}, pages 12593--12618. PMLR.

\bibitem[Magliacane, 2018]{magliacane2018causalbias}
Magliacane, S. (2018).
\newblock Domain adaptation by using causal inference to predict invariant conditional distributions.
\newblock In {\em 32nd Conference on Neural Information Processing Systems (NIPS 2018)}.

\bibitem[Miller et~al., 2020]{miller2020strategiccausal}
Miller, J., Milli, S., and Hardt, M. (2020).
\newblock Strategic classification is causal modeling in disguise.
\newblock In {\em International Conference on Machine Learning}, pages 6917--6926. PMLR.

\bibitem[Milli et~al., 2019]{milli2019socialcost}
Milli, S., Miller, J., Dragan, A.~D., and Hardt, M. (2019).
\newblock The social cost of strategic classification.
\newblock In {\em Proceedings of the conference on fairness, accountability, and transparency}, pages 230--239.

\bibitem[Perdomo et~al., 2020]{perdomo2020performative}
Perdomo, J., Zrnic, T., Mendler-D{\"u}nner, C., and Hardt, M. (2020).
\newblock Performative prediction.
\newblock In {\em International Conference on Machine Learning}, pages 7599--7609. PMLR.

\bibitem[Perry et~al., 2022]{Perry2022_sparse_mechanism_shift}
Perry, R., von K{\"u}gelgen, J., and Sch{\"o}lkopf, B. (2022).
\newblock Causal discovery in heterogeneous environments under the sparse mechanism shift hypothesis.
\newblock In {\em Proceedings of the 36th Conference on Neural Information Processing Systems (NeurIPS)}.

\bibitem[Peters et~al., 2016]{Peters2016_invariant_prediction}
Peters, J., B{\"u}hlmann, P., and Meinshausen, N. (2016).
\newblock Causal inference using invariant prediction: identification and confidence intervals.
\newblock In {\em Journal of the Royal Statistical Society, Series B (Statistical Methodology)}, volume~78, pages 947--1012.

\bibitem[Peters et~al., 2017]{peters2017elements}
Peters, J., Janzing, D., and Sch{\"o}lkopf, B. (2017).
\newblock {\em Elements of Causal Inference: Foundations and Learning Algorithms}.
\newblock MIT Press, Cambridge, MA.

\bibitem[Rojas-Carulla et~al., 2018]{rojascarulla2018invariant}
Rojas-Carulla, M., Sch{\"o}lkopf, B., Turner, R., and Peters, J. (2018).
\newblock Invariant models for causal transfer learning.
\newblock In {\em Proceedings of the 21st International Conference on Artificial Intelligence and Statistics (AISTATS)}, volume~84 of {\em Proceedings of Machine Learning Research}, pages 164--173.

\bibitem[Rosenfeld et~al., 2020]{rosenfeld2020lookahead}
Rosenfeld, N., Hilgard, S., Ravindranath, S.~S., and Parkes, D.~C. (2020).
\newblock From predictions to decisions: Using lookahead regularization.
\newblock In {\em 34th Conference on Neural Information Processing Systems (NeurIPS 2020)}. Curran Associates, Inc.

\bibitem[Shavit et~al., 2020]{shavit2020learningFS}
Shavit, Y., Edelman, B.~L., and Axelrod, B. (2020).
\newblock Learning from strategic agents: Accuracy, improvement, and causality.
\newblock {\em ArXiv}, abs/2002.10066.

\bibitem[Somerstep et~al., 2024]{Somerstep_causal_twosided_markets}
Somerstep, S., Sun, Y., and Ritov, Y. (2024).
\newblock Learning in reverse causal strategic environments with ramifications on two sided markets.
\newblock In Kim, B., Yue, Y., Chaudhuri, S., Fragkiadaki, K., Khan, M., and Sun, Y., editors, {\em International Conference on Learning Representations}, volume 2024, pages 56533--56555.

\bibitem[Tse, 2018]{semisynthetic_dataset}
Tse, L. (2018).
\newblock Credit risk dataset.
\newblock \url{https://www.kaggle.com/datasets/laotse/credit-risk-dataset}.
\newblock Accessed: Nov. 25, 2025.

\bibitem[Vo et~al., 2024]{krikamol2024causal}
Vo, K. Q.~H., Aadil, M., Chau, S.~L., and Muandet, K. (2024).
\newblock Causal strategic learning with competitive selection.
\newblock In {\em Proceedings of the 38th AAAI Conference on Artificial Intelligence (AAAI 2024)}. AAAI Press.

\end{thebibliography}

\section*{Checklist}

\begin{enumerate}

  \item For all models and algorithms presented, check if you include:
  \begin{enumerate}
    \item A clear description of the mathematical setting, assumptions, algorithm, and/or model. [Yes]
    \item An analysis of the properties and complexity (time, space, sample size) of any algorithm. [Not Applicable]
    \item (Optional) Anonymized source code, with specification of all dependencies, including external libraries. [Not Applicable]
  \end{enumerate}

  \item For any theoretical claim, check if you include:
  \begin{enumerate}
    \item Statements of the full set of assumptions of all theoretical results. [Yes]
    \item Complete proofs of all theoretical results. [Yes]
    \item Clear explanations of any assumptions. [Yes]     
  \end{enumerate}

  \item For all figures and tables that present empirical results, check if you include:
  \begin{enumerate}
    \item The code, data, and instructions needed to reproduce the main experimental results (either in the supplemental material or as a URL). [Yes]
    \item All the training details (e.g., data splits, hyperparameters, how they were chosen). [Yes]
    \item A clear definition of the specific measure or statistics and error bars (e.g., with respect to the random seed after running experiments multiple times). [Yes]
    \item A description of the computing infrastructure used. (e.g., type of GPUs, internal cluster, or cloud provider). [Not Applicable]
  \end{enumerate}

  \item If you are using existing assets (e.g., code, data, models) or curating/releasing new assets, check if you include:
  \begin{enumerate}
    \item Citations of the creator If your work uses existing assets. [Not Applicable]
    \item The license information of the assets, if applicable. [Not Applicable]
    \item New assets either in the supplemental material or as a URL, if applicable. [Not Applicable]
    \item Information about consent from data providers/curators. [Not Applicable]
    \item Discussion of sensible content if applicable, e.g., personally identifiable information or offensive content. [Not Applicable]
  \end{enumerate}

  \item If you used crowdsourcing or conducted research with human subjects, check if you include:
  \begin{enumerate}
    \item The full text of instructions given to participants and screenshots. [Not Applicable]
    \item Descriptions of potential participant risks, with links to Institutional Review Board (IRB) approvals if applicable. [Not Applicable]
    \item The estimated hourly wage paid to participants and the total amount spent on participant compensation. [Not Applicable]
  \end{enumerate}

\end{enumerate}

\clearpage
\raggedbottom
 \onecolumn
\appendix

\aistatstitle{The Role of Causal Features in Strategic Classification for
Robustness and Alignment: \\
Supplementary Materials}

\section{0-1 optimality of causal classifier under enough shift}
\label{app:opt_01_high_shift}

We assume $y$ is a deterministic function of $(x_c,u)$ given by $\text {sign}(y_{\text{sco}}(x_c,u))$. Note that knowing the deterministic effect of $x_c$ on $y_{\mathrm{sco}}$ requires knowing $u$, which is an unobserved random variable. The function $y_{\text{sco}}(x_c,u)$ is not explicitly defined here, and is not necessarily linear.
\begin{assumption}
\label{ass:y_sign}
    (Sign function $y$) $y(x_c,u)\triangleq \mathds 1\{y_{\text{sco}}(x_c,u)\geq0\}$, where $y_{\text{sco}}:\mathbb R^d\times\mathbb R\rightarrow\mathbb R$
\end{assumption}

\begin{assumption}
    \label{ass:ysco_continuity} ($y_{\mathrm{sco}}$ continuity) $y_{\mathrm{sco}}$ is continuous with respect to $x_c$.
\end{assumption}

Following \citet{boyd2004convex} we define an O-nondecreasing function. Let the collection of all orthants in $\mathbb R^d$ be $ O=\{\{x\in\mathbb R^d:s_ix_i\geq0,\forall\, i\in[d]\},s\in\{-1,+1\}^d\}$, and $O_s\in O$ be the orthant defined by $s$. We define a partial ordering on $\mathbb R^d$ as $x\preceq_{O_s} y\Leftrightarrow y-x\in O_s$. Similarly we have $x\prec_{O_s} y\Leftrightarrow y-x\in \mathrm{int}( O_s)$. We say a function $f:\mathbb R^d\rightarrow\mathbb R$ is called O-nondecreasing if $x\preceq_{O_s} y\Rightarrow f(x)\leq f(y)$. We then assume $y_{\mathrm{sco}}(x_c,u)$ is O-nondecreasing with respect to $x_c$, using the same orthant $O_s$ for all $u\in\mathcal U$.

\begin{assumption}
\label{ass:y_monotonic}
    (O-nondecreasing $y_{\mathrm{sco}}$) $\exists\, O_s\in O\,\forall\, x_c\in\mathcal X_c,x_c'\in\mathcal X_c, u\in\mathcal U :x_c\preceq_{O_s} x_c'\implies y_{\text{sco}}(x_c,u)\leq y_{\text{sco}}(x_c',u)$
\end{assumption}

We define $\mathcal X_\mathrm{ambiguous}$ as the subset of causal feature space where $y_\mathrm{sco}$ can take both negative and positive values.

\begin{definition}
    (Domain with ambiguous outcome)
    $\mathcal X_{\mathrm{ambiguous}}\triangleq\{x_c\,\,\mathrm{s.t.}\,\,\exists \,u\in\mathcal U : y_{\mathrm{sco}}(x_c,u)\geq0,\exists \,u'\in\mathcal U : y_{\mathrm{sco}}(x_c,u')<0\}\subseteq \mathcal X_c$
\end{definition}

\begin{assumption}
\label{ass:y_increase}
    (Ambiguity compensation through $x_c$)
    $\exists\,\delta\in\mathbb R:\forall x_c\in \mathcal X_{\mathrm{ambiguous}}: \exists\,
    v\in\mathbb R^d,\,||v||_p\leq\delta:\forall \,u,\,y_{\mathrm{sco}}(x_c+v,u)\geq0$
\end{assumption}

\begin{corollary}
\label{cor:two_bounds_app}
    (Partition of $\mathcal X_c$ through $\partial_\mathrm{upp}$ and $\partial_\mathrm{low}$) Assume \ref{ass:ysco_continuity}, \ref{ass:y_monotonic}, \ref{ass:y_increase}. Let $\partial_u=\{x_c: y_{\mathrm{sco}}(x_c,u)=0,\forall\tilde x_c\in \mathrm{int}(x_c-O_s), y_\mathrm{sco}(\tilde x_c,u)<0\}$, where the last condition prevents ``thick" boundary regions. Let $B\triangleq\underset{u}{\cup}\,\partial_u$.
    Define $\partial_\mathrm{upp}\triangleq\{x_c\in B:\mathrm{int}(x_c+O_s) \cap B = \varnothing\}$ and $\partial_\mathrm{low}\triangleq\{x_c\in B:\mathrm{int}(x_c-O_s) \cap B = \varnothing\}$. It follows that $\forall x_c:\exists\,x_\mathrm{upp}\in\partial_\mathrm{upp}, x_\mathrm{upp}\preceq_{O_s}x_c\Rightarrow x_c\not\in\mathcal X_\mathrm{ambiguous}, \forall u:y_\mathrm{sco}(x_c,u)\geq0$ and $\forall x_c:\exists\,x_\mathrm{low}\in\partial_\mathrm{low}, x_c\prec_{O_s}x_\mathrm{low}\Rightarrow x_c\not\in\mathcal X_\mathrm{ambiguous}, \forall u:y_\mathrm{sco}(x_c,u)<0$. Hence, we can partition $\mathcal{X}_c$ into three disjoint subsets, separated by $\partial_\mathrm{upp}$ and $\partial_\mathrm{low}$: $\mathcal{X}_\mathrm{ambiguous}$, $\{x_c: \forall u,y_\mathrm{sco}(x_c,u)\geq0\}$ and $\{x_c: \forall u,y_\mathrm{sco}(x_c,u)<0\}$.
\end{corollary}

\begin{proof}
    We prove by contradiction that $\forall \tilde x_c:\exists\,x_\mathrm{upp}\in\partial_\mathrm{upp}, x_\mathrm{upp}\preceq_{O_s}\tilde x_c\Rightarrow \tilde x_c\not\in\mathcal X_\mathrm{ambiguous}, \forall u:y_\mathrm{sco}(\tilde x_c,u)\geq0$.

    Suppose $\exists\, u',\tilde x_c:y_{\mathrm{sco}}(\tilde x_c,u')<0, \exists\,x_\mathrm{upp}\in\partial_\mathrm{upp}, x_\mathrm{upp}\preceq_{O_s}\tilde x_c$.

    From the definition of $\partial_\mathrm{upp}$, we have that $\exists\,x_\mathrm{upp}\in\partial_\mathrm{upp}, x_\mathrm{upp}\preceq_{O_s}\tilde x_c\Rightarrow \mathrm{int}(\tilde x_c+O_s)\cap B=\varnothing$ and hence $ \mathrm{int}(\tilde x_c+O_s)\cap \partial_{u'}=\varnothing$.

    Due to continuity (\ref{ass:ysco_continuity}) there exists an $\epsilon$-ball around $\tilde x_c$ where $y_{\mathrm{sco}}(\tilde x_c,u')<0$, particularly $\exists\, \epsilon\in O_s:y_{\mathrm{sco}}(\tilde x_c+\epsilon,u')<0$. It follows that $\mathrm{int}(\tilde x_c+O_s)\cap \partial_{u'}=\varnothing\Rightarrow\mathrm{int}(\tilde x_c+\epsilon+O_s)\cap \partial_{u'}=\varnothing$.

    However, from \ref{ass:y_increase}, $y_{\mathrm{sco}}(\tilde x_c+\epsilon,u')<0\Rightarrow\exists\,v\in \mathbb R^d:y_{\mathrm{sco}}(\tilde x+\epsilon+v,u')=0$. From \ref{ass:y_monotonic} there must also be a vector $v'\in O_s$ which increases $y_\mathrm{sco}$ to zero, $\exists\, q\in O_s, v'\triangleq q+v:v'\in O_s,y_{\mathrm{sco}}(\tilde x_c+\epsilon+v',u')=0$. 

    Hence $\mathrm{int}(\tilde x_c+\epsilon+O_s)\cap \partial_{u'}\neq \varnothing$, which is a contradiction.

    We can show similarly by contradiction that $\forall x_c:\exists\,x_\mathrm{low}\in\partial_\mathrm{low}, x_c\prec_{O_s}x_\mathrm{low}\Rightarrow x_c\not\in\mathcal X_\mathrm{ambiguous}, \forall u:y_\mathrm{sco}(x_c,u)<0$.

\end{proof}

\begin{corollary}
\label{cor:bounded_dist}
    (Bounded distance to $\mathcal X_\mathrm{ambiguous}$ boundary) $\forall x_c\in\mathcal X_\mathrm{ambiguous}\Rightarrow \underset{x_\mathrm{upp}\in\partial_\mathrm{upp}}{\min}||x_\mathrm{upp}-x_c||_p\leq\underset{x_\mathrm{low}\in\partial_\mathrm{low}}{\max} \underset{x_\mathrm{upp}\in\partial_\mathrm{upp}}{\min}||x_\mathrm{upp}-x_\mathrm{low}||_p=\underset{u}{\max}\,\underset{x_\mathrm{low}\in \partial_u}{\max}\,\underset{x_\mathrm{upp}\in\partial_\mathrm{upp}}{\min}\,||x_\mathrm{upp}
    -x_\mathrm{low}||_p<+\infty$
\end{corollary}
\begin{proof}
    By contradiction, suppose that $\exists\, \tilde x_c\in\mathcal X_\mathrm{ambiguous}:\underset{x_\mathrm{upp}\in\partial_\mathrm{upp}}{\min}||x_\mathrm{upp}-\tilde x_c||_p>\underset{x_\mathrm{low}\in\partial_\mathrm{low}}{\max}\underset{x_\mathrm{upp}\in\partial_\mathrm{upp}}{\min}||x_\mathrm{upp}-x_\mathrm{low}||_p$.
    
    It follows that $\tilde x_c\not\in\partial_\mathrm{low}$. 

    Denote $\tilde x_\mathrm{upp}\triangleq \underset{x_\mathrm{upp}\in\partial_\mathrm{upp}}{\arg\min}||x_\mathrm{upp}-\tilde x_c||_p$. From \ref{ass:y_monotonic}, we have that $\tilde x_\mathrm{upp}\in \tilde x_c+O_s$. To prove this, consider $\tilde x_\mathrm{upp-O}\triangleq\underset{x\in \partial_\mathrm{upp}\cap\tilde x_c+O_s}{\min}||x-\tilde x_c||_p$. We know $\forall x\in\partial_\mathrm{upp}\setminus\{x_\mathrm{upp-O}\}\Rightarrow x\not\in x_\mathrm{upp-O}-\mathrm{int}(O_s)$. Hence $\forall x\in\partial_\mathrm{upp}\setminus\{x_\mathrm{upp-O}\}, p<+\infty\Rightarrow ||x-\tilde x_c||_p>||x_\mathrm{upp}-\tilde x_c||_p$, therefore $x_\mathrm{upp}=x_\mathrm{upp-O}$. For $p=+\infty$ there can be ties with other $\arg\min$, but we pick $x_\mathrm{upp}$ since it is one of the minimizers.

    Since $\tilde x_c\not\in\partial_\mathrm{low}$, then $\exists \, x_\mathrm{low}\in\partial_\mathrm{low}:x_\mathrm{low}\in \tilde x_c-O_s$. Hence $\forall p, ||x_\mathrm{upp}-x_\mathrm{low}||_p>||x_\mathrm{upp}-\tilde x_c||_p$, leading to a contradiction. This shows that $\underset{x_\mathrm{upp}\in\partial_\mathrm{upp}}{\min}||x_\mathrm{upp}-x_c||_p\leq\underset{x_\mathrm{low}\in\partial_\mathrm{low}}{\max} \underset{x_\mathrm{upp}\in\partial_\mathrm{upp}}{\min}||x_\mathrm{upp}-x_\mathrm{low}||_p$.

    The equality $\underset{x_\mathrm{low}\in\partial_\mathrm{low}}{\max} \underset{x_\mathrm{upp}\in\partial_\mathrm{upp}}{\min}||x_\mathrm{upp}-x_\mathrm{low}||_p=\underset{u}{\max}\,\underset{x_\mathrm{low}\in \partial_u}{\max}\,\underset{x_\mathrm{upp}\in\partial_\mathrm{upp}}{\min}\,||x_\mathrm{upp}
    -x_\mathrm{low}||_p$ comes from the definition of $B$ and $\partial_\mathrm{low}$.

    Assumption~\ref{ass:y_increase} provides a finite upper bound on the distance between any point in the ambiguous region and a positive point outside it, and the path between these points necessarily crosses the boundary $\partial_{\text{upp}}$. 
    Note that any point $x_{\text{low}}$ lies either in the ambiguous region or on one of its two boundaries (by definition), and that any $x_{\text{low}}$ that lies on a boundary $\partial_{\text{u}}$ has a bounded distance to $\partial_{\text{upp}}$ (from above). Therefore, it follows that: $$\underset{u}{\max}\,\underset{x_\mathrm{low}\in \partial_u}{\max}\,\underset{x_\mathrm{upp}\in\partial_\mathrm{upp}}{\min}\,||x_\mathrm{upp}
    -x_\mathrm{low}||_p<+\infty.$$
\end{proof}

\begin{assumption}
\label{ass:adaptation_lpnorm}
    (Agent adaptation with $L_p$-norm cost) Before being classified, agents have knowledge of classifier $h(x)$. They adapt their features, from $x$ into $\Delta_{h}(x)$, by maximizing their reward $r_p(h,x)=\delta h (\Delta_{h}(x))-c(x,\Delta_{h}(x))$. Assume $c(x,x'):=||x'-x||_p$.
\end{assumption}

Under the previous assumptions, there exists a classifier which achieves zero 0-1 loss post-adaptation, by ignoring spurious features and setting a threshold that moves points away from $\mathcal X_{\mathrm{ambiguous}}$.

\begin{theorem}
\label{the:causal_01opt_app}
    (Causal $\ell_{0-1}$ optimality) Assume \ref{ass:y_sign}, 
    \ref{ass:ysco_continuity},
    \ref{ass:y_monotonic}, \ref{ass:y_increase}, \ref{ass:adaptation_lpnorm}. Let $h_c$ be a causal classifier whose outputs are not changed by $x_s$.
    
    $\exists\,e\in \mathbb R,h_c\in\mathcal H:\forall\, \delta\geq e : \mathbb E_{x,u}[\ell_{0-1}(h_c(\Delta_{h_c}(x;\delta)),\mathrm{sign}(y_{\mathrm{sco}}(\Delta_{h_c}(x;\delta),u)))]=0$

\end{theorem}

\begin{proof}
Consider the following definitions from Corollary~\ref{cor:two_bounds_app}, of $\partial_u=\{x_c\in\mathcal X_c: y_{\mathrm{sco}}(x_c,u)=0,\forall\tilde x_c\in \mathrm{int}(x_c-O_s), y_\mathrm{sco}(\tilde x_c,u)<0\}\subset\mathcal X_c$, and $B:=\underset{u}{\cup}\,\partial_u$.

Denote the learned classification boundary of a hypothesis $h(x)=\mathds 1\{f(x)\geq0\}$ by $\partial_h=\{x\in\mathcal X:f(x)=0\}$. Given $\partial_h$, $h(x)$ is such that if $\exists\, x'\in \partial_h: x'\preceq_{O_s} x$ then $h(x)=1$, else $h(x)=0$.

We can build the optimal classifier named $h_c(x_c):\mathcal X_c\rightarrow\{0,1\}$ s.t. $x_c\in\partial_{h_c}$ if $x_c\in B$ and $\mathrm{int}(x_c+O_s) \cap B = \varnothing$.

We split the proof in three parts:
\begin{enumerate}
    \item Static TPs (true positives where $\Delta_h(x)=x$) are correctly classified as $h(\Delta_h(x))=1$;
    \item All points in $\mathcal X_{\mathrm{ambiguous}}$ adapt ($\Delta_h(x)\neq x$) such that they are correctly classified as $h(\Delta_h(x))=1$;
    \item Pre-adaptation TNs are correctly classified post-adaptation, either remaining static ($\Delta_h(x)=x$) with $h(\Delta_h(x))=0$, or adapting into TPs.
\end{enumerate}

\textit{1. Correct classification of static TPs:}

From $\ref{ass:adaptation_lpnorm}$, $\forall x: h(x)=1\Rightarrow\Delta_h(x)=x$, since a point with $h(x)=1$ cannot further increase its utility by adapting features.

$h_c(\tilde x)=1\Rightarrow \forall u':y_{\mathrm{sco}}(\tilde x,u')\geq 0$. Below we prove this by contradiction. This implies all points $h_c(\tilde x)$ are correctly classified since their true outcome $y=1$ (from \ref{ass:y_sign}).

Assume $\exists\, u':y_{\mathrm{sco}}(\tilde x,u')<0, h(\tilde x)=1$.

Since $h(\tilde x)=1$, we know from the definition on $h_c$ that $\mathrm{int}(\tilde x+O_s)\cap B=\varnothing$ and hence $
\mathrm{int}(\tilde x+O_s)\cap \partial_{u'}=\varnothing$.

Due to continuity (\ref{ass:ysco_continuity}) there exists an $\epsilon$-ball around $\tilde x$ where $y_{\mathrm{sco}}<0$, particularly $\exists\, \epsilon\in O_s:y_{\mathrm{sco}}(\tilde x+\epsilon,u')<0$. It follows that $\mathrm{int}(\tilde x+O_s)\cap \partial_{u'}=\varnothing\Rightarrow\mathrm{int}(\tilde x+\epsilon+O_s)\cap \partial_{u'}=\varnothing$.

However, from \ref{ass:y_increase}, $y_{\mathrm{sco}}(\tilde x+\epsilon,u')<0\Rightarrow\exists\,v\in \mathbb R^d:y_{\mathrm{sco}}(\tilde x+\epsilon+v,u')=0$. From \ref{ass:y_monotonic} there must also be a vector $v'\in O_s$ which increases $y_\mathrm{sco}$ to zero, $\exists\, q\in O_s, v'\triangleq q+v:v'\in O_s,y_{\mathrm{sco}}(\tilde x+\epsilon+v',u')=0$. 

Hence $\mathrm{int}(\tilde x+\epsilon+O_s)\cap \partial_{u'}\neq \varnothing$, which is a contradiction.

\textit{2. Correct classification of $\mathcal X_{\mathrm{ambiguous}}$ post-adaptation:}

We proved above that $h_c(\tilde x)=1\Rightarrow \forall u':y_{\mathrm{sco}}(\tilde x,u')\geq 0$. From \ref{ass:y_sign} we have the same implication for post-adaptation data points:
$\forall \tilde x,h_c(\Delta_{h_c}(\tilde x))=1\Rightarrow \forall u',y(\Delta_{h_c}(\tilde x),u')=1$.

From \ref{ass:y_monotonic}, \ref{ass:adaptation_lpnorm} and the definition of $h_c$, we have $\forall x:h_c(x)=0,h_c(\Delta_{h_c}(x))=1\Rightarrow\Delta_{h_c}(x)\in\partial_{h_c}$ for any $L_p$-norm cost. To show by contradiction assume
$\exists\,x:h_c(x)=0,h_c(\Delta_{h_c}(x))=1,\Delta_{h_c}(x)\not\in \partial_{h_c}$. From the definition of $h_c$ and \ref{ass:y_monotonic}, $\exists\, b\in\partial_{h_c}:x\preceq_{O_s} b\preceq_{O_s} \Delta_{h_c}(x)$. From \ref{ass:adaptation_lpnorm} we have $c(x,b)<c(x,\Delta_{h_c}(x))$ for any $L_p$-norm, hence $x$ must have adapted instead to $b\in\partial_{h_c}$.

Define $e\triangleq\underset{u}{\max}\,\underset{x\in \partial_u}{\max}\,\underset{x'\in\partial_h}{\min}\,c(x,x')$. From \ref{cor:bounded_dist}
we have that this quantity is bounded.

$\delta\geq e\Rightarrow\forall x\in\mathcal{X}_\mathrm{ambiguous}, \Delta_{h_c}(x)\in\partial_{h_c}, h_c(x)=1,\forall u, y(x,u)=1$.

\textit{3. Correct classification of static TNs:}

If any remaining points $x_{\mathrm{neg}}\in\mathcal X_c$ exist not covered by the cases above, it has $\forall u': y_{\mathrm{sco}}(x_{\mathrm{neg}},u')<0$ and $h_c(x_{\mathrm{neg}})=0$, since $\forall x_\mathrm{neg},\exists\, x_\mathrm{low}\in\partial_\mathrm{low}:x_\mathrm{neg}\prec_{O_s} x_\mathrm{low}$. After adaptation it either remains unchanged or adapts such that $\forall u':y_\mathrm{sco}(\Delta_h(x_{\mathrm{neg}}),u')\geq0$ and $h(\Delta(x_{\mathrm{neg}}))=1$, since it must have adapted to $\partial_\mathrm{upp}$. 
\end{proof}

\section{CE loss optimality of causal classifier under enough shift}
\label{app:opt_CE_high_shift_detY}

From Corollary~\ref{cor:two_bounds_app}, we define the following sets:
\begin{align*}
    \mathcal{X}_\mathrm{upp} &\triangleq \mathcal \{x_c: \forall u,y_\mathrm{sco}(x_c,u)\geq0\} \\ 
    \mathcal{X}_\mathrm{low} &\triangleq \{x_c: \forall u,y_\mathrm{sco}(x_c,u)<0\}
\end{align*}

where the sets $\mathcal{X}_\mathrm{upp}$, $\mathcal{X}_\mathrm{ambiguous}$, $ \mathcal{X}_\mathrm{low}$ are disjoint subsets of $\mathcal X_c$ and $\mathcal{X}_\mathrm{upp} \cup \mathcal{X}_\mathrm{ambiguous} \cup \mathcal{X}_\mathrm{low} = \mathcal X_c$. 
Note that by this definition and following Assumption~\ref{ass:y_sign}, for any $u$, if $x_c \in \mathcal{X}_{upp}$, $y(x_c,u) = 1$, and if $x_c \in \mathcal{X}_{low}$, $y(x_c,u) = 0$.

\begin{align*}
    &\mathcal{L}_{\mathrm{CE}}(\hat{f}, \delta)
    =   \mathbb{E}_{x, y \sim \mathcal{D}^{(\hat{f}, \delta)}}\left[
    - y \log \hat{f}(x) 
    - (1-y) \log (1-\hat{f}(x))\right] \\
    = & \mathbb{E}_{x_c, x_s, u \sim \mathcal{D}^{(\hat{f}, \delta)} }\left[
    - 1_{\{y_{\mathrm{sco}}(x_c, u) \geq 0\}} \log \hat{f}(x_c, x_s) 
    - (1_{\{y_{\mathrm{sco}}(x_c, u) < 0\}}) \log (1-\hat{f}(x_c, x_s))\right] \\
    = &  \mathbb{E}_{x_c \sim \mathcal{D}^{(\hat{f}, \delta)} } \left[\mathbb{E}_{x_s, u \sim \mathcal{D}^{(\hat{f}, \delta)} }\left[
    - 1_{\{y_{\mathrm{sco}}(x_c, u) \geq 0\}} \log \hat{f}(x_c, x_s) 
    - 1_{\{y_{\mathrm{sco}}(x_c, u) < 0\}} \log (1-\hat{f}(x_c, x_s))\right]\right] \\
    = & \int_{\mathcal{X}_c} \mathbb{P}_{\mathcal{D}^{(\hat{f}, \delta)}}(x_c) * \left[\mathbb{E}_{x_s, u \sim \mathcal{D}^{(\hat{f}, \delta)} }\left[
    - 1_{\{y_{\mathrm{sco}}(x_c, u) \geq 0\}} \log \hat{f}(x_c, x_s) 
    - 1_{\{y_{\mathrm{sco}}(x_c, u) < 0\}} \log (1-\hat{f}(x_c, x_s))\right]\right] d x_c \\
    = & \int_{\mathcal{X}_\mathrm{low}} \mathbb{P}_{\mathcal{D}^{(\hat{f}, \delta)}}(x_c) * \left[\mathbb{E}_{x_s, u \sim \mathcal{D}^{(\hat{f}, \delta)} }\left[
    - 1_{\{y_{\mathrm{sco}}(x_c, u) \geq 0\}} \log \hat{f}(x_c, x_s) 
    - 1_{\{y_{\mathrm{sco}}(x_c, u) < 0\}} \log (1-\hat{f}(x_c, x_s))\right]\right] d x_c \\ 
    & + \int_{\mathcal{X}_\mathrm{ambiguous}} \mathbb{P}_{\mathcal{D}^{(\hat{f}, \delta)}}(x_c) * \left[\mathbb{E}_{x_s, u \sim \mathcal{D}^{(\hat{f}, \delta)} }\left[
    - 1_{\{y_{\mathrm{sco}}(x_c, u) \geq 0\}} \log \hat{f}(x_c, x_s) 
    - 1_{\{y_{\mathrm{sco}}(x_c, u) < 0\}} \log (1-\hat{f}(x_c, x_s))\right]\right] d x_c \\ 
    & + \int_{\mathcal{X}_\mathrm{upp}} \mathbb{P}_{\mathcal{D}^{(\hat{f}, \delta)}}(x_c)
    * \left[\mathbb{E}_{x_s, u \sim \mathcal{D}^{(\hat{f}, \delta)} }\left[
    - 1_{\{y_{\mathrm{sco}}(x_c, u) \geq 0\}} \log \hat{f}(x_c, x_s) 
    - 1_{\{y_{\mathrm{sco}}(x_c, u) < 0\}} \log (1-\hat{f}(x_c, x_s))\right]\right] d x_c \\
    = & \int_{\mathcal{X}_\mathrm{low}} \mathbb{P}_{\mathcal{D}^{(\hat{f}, \delta)}}(x_c) * \left[\mathbb{E}_{x_s, u \sim \mathcal{D}^{(\hat{f}, \delta)} }\left[
    - 1 \cdot \log (1-\hat{f}(x_c, x_s))\right]\right] d x_c \\ 
    & + \int_{\mathcal{X}_\mathrm{ambiguous}} \mathbb{P}_{\mathcal{D}^{(\hat{f}, \delta)}}(x_c) * \left[\mathbb{E}_{x_s, u \sim \mathcal{D}^{(\hat{f}, \delta)} }\left[
    - 1_{\{y_{\mathrm{sco}}(x_c, u) \geq 0\}} \log \hat{f}(x_c, x_s) 
    - 1_{\{y_{\mathrm{sco}}(x_c, u) < 0\}} \log (1-\hat{f}(x_c, x_s))\right]\right] d x_c \\ 
    & + \int_{\mathcal{X}_\mathrm{upp}} \mathbb{P}_{\mathcal{D}^{(\hat{f}, \delta)}}(x_c)
    * \left[\mathbb{E}_{x_s, u \sim \mathcal{D}^{(\hat{f}, \delta)} }\left[
    - 1 \cdot \log \hat{f}(x_c, x_s) 
    \right]\right] d x_c
\end{align*}

Using the same causal classifier from \Cref{the:causal_01opt_app}, we define $h_c$ s.t. $x_c\in\partial_{h_c}$ if $x_c\in B$ and $\mathrm{int}(x_c+O_s) \cap B = \varnothing$. This means for all points $x_c \in \mathcal{X}_\mathrm{upp}$, $h_c(x_c) = 1$ and $x'_c \in \mathcal{X}_\mathrm{low}$, $h_c(x'_c) = 0$. We define the ``scoring" function $\hat{f}_c: \mathcal{X} \rightarrow [0,1]$ that cross entropy uses as the same function as $h_c$, meaning it outputs strictly $0$ and $1$.

With this data generating process and classifier, we previously proved that for finite $\delta > e$ (which we call the max-gap), all points will move from $\mathcal{X}_\mathrm{ambiguous}$ into $\mathcal{ X}_\mathrm{upp}$, obtaining true outcome $y(x_c,u)=1$ and correct prediction $h(x_c)=1$.
Thus, the cross entropy loss after $\delta > e$ will also be $0$ with this specific causal classifier. We can clearly see this from the derivation above because with this $\hat{f}_c$, the first and last terms will $= 0$ for any $\delta$, and the middle term will take value $0$ once the probability density in that region becomes $0$, which occurs when all the points have adapted out of the region, i.e. $\delta > e$.

\section{Cross Entropy Loss Analysis}
\label{app:CE_loss_analysis}

\subsection{Cross Entropy Loss Decomposition}
\label{app:CE_loss_decomposition}

Assuming $0\log 0:=0$, $\hat f:\mathcal{X}\to(0,1)$ is measurable, and there is support-compatibility (i.e.
$\{g>0\}\subseteq\{\hat{f}>0,\,\hat{f}_\delta>0\}$ and 
$\{g<1\}\subseteq\{\hat{f}<1,\,\hat{f}_\delta<1\}$ almost surely), we can decompose the cross entropy loss after adaptation to classifier $\hat{f}$ as follows:

\begin{align*}
&\mathcal{L}_{\mathrm{CE}}(\hat{f}, \delta)
= \mathbb{E}_{x,y \sim \mathcal{D}^{(\hat{f}, \delta)}}\left[
    - y \log \hat{f}(x) 
    - (1-y) \log (1-\hat{f}(x))\right]
\\[6pt]
&= \mathbb{E}_{x_c, x_s, u \sim \mathcal{D}^{(\hat{f}, \delta)}}\left[
    - g(x_c, u) \log \hat{f}(x_c, x_s) 
    - (1-g(x_c, u)) \log (1-\hat{f}(x_c, x_s))\right]
\\[6pt]
&= \mathbb{E}_{\mathcal{D}^{(\hat{f}, \delta)}}\Big[
    - g(x_c, u)\log \hat{f}(x_c, x_s)  - (1-g(x_c,u))\log(1-\hat{f}(x_c, x_s))
   \\ &  \ \ \ \  + (1-1) * \Big(g(x_c, u)\log g(x_c, u)  + (1-g(x_c,u))\log(1-g(x_c,u))\Big)
  \Big]
\\[6pt]
&= \mathbb{E}_{\mathcal{D}^{(\hat{f}, \delta)}} \left[
     g(x_c, u) \log\frac{g(x_c, u)}{\hat{f}(x_c, x_s)}
    + (1-g(x_c, u)) \log\frac{1-g(x_c, u)}{1-\hat{f}(x_c, x_s)}
   \right]
  \\ &  \ \ \ \ - 
  \mathbb{E}_{\mathcal{D}^{(\hat{f}, \delta)}} \left[
     g(x_c, u) \log g(x_c, u)
    + (1-g(x_c, u)) \log(1-g(x_c, u))
   \right]
\\[6pt]
&= \underbrace{\mathbb{E}_{\mathcal{D}^{(\hat{f}, \delta)}} \left[\mathrm{KL} \left(g(x_c, u) \| \hat{f}(x_c, x_s) \right) \right]}_{\text{(KL divergence)}} 
   + 
\underbrace{\mathbb{E}_{\mathcal{D}^{(\hat{f}, \delta)}} \left[H \left(g(x_c, u) \right) \right]}_{\text{(entropy)}}.
\\[6pt]
&= \mathbb{E}_{\mathcal{D}^{(\hat{f}, \delta)}} \left[
     g(x_c, u) \log\frac{g(x_c, u)}{\hat{f}(x_c, x_s)}
    + (1-g(x_c, u)) \log\frac{1-g(x_c, u)}{1-\hat{f}(x_c, x_s)}
   \right]
   \\[3pt] &  \ \ \ \ - 
  \mathbb{E}_{\mathcal{D}^{(\hat{f}, \delta)}} \left[
     g(x_c, u) \log g(x_c, u)
    + (1-g(x_c, u)) \log(1-g(x_c, u))
   \right]
   \\[3pt] &  \ \ \ \ + \mathbb{E}_{\mathcal{D}^{(\hat{f}, \delta)}} \left[
     (1 - 1) \cdot g(x_c, u) \log \hat{f}^*_{(\hat{f}, \delta)}(x_c, x_s) 
    + (1 - 1) \cdot (1-g(x_c, u)) \log(1-\hat{f}^*_\delta(x_c, x_s))
   \right]
\\[6pt]
&= \mathbb{E}_{\mathcal{D}^{(\hat{f}, \delta)}} \left[
     g(x_c, u) \log\frac{g(x_c, u)}{\hat{f}^*_\delta(x_c, x_s)}
    + (1-g(x_c, u)) \log\frac{1-g(x_c, u)}{1-\hat{f}^*_\delta(x_c, x_s)}
   \right]
   \\[3pt] &  \ \ \ \ +
   \mathbb{E}_{\mathcal{D}^{(\hat{f}, \delta)}} \left[
     g(x_c, u) \log\frac{\hat{f}^*_\delta(x_c, x_s)}{\hat{f}(x_c, x_s)}
    + (1-g(x_c, u)) \log\frac{1-\hat{f}^*_\delta(x_c, x_s)}{1-\hat{f}(x_c, x_s)}
   \right]
   \\[3pt] &  \ \ \ \ - 
  \mathbb{E}_{\mathcal{D}^{(\hat{f}, \delta)}} \left[
     g(x_c, u) \log g(x_c, u)
    + (1-g(x_c, u)) \log(1-g(x_c, u))
   \right]
\\[6pt]
&= \mathbb{E}_{\mathcal{D}^{(\hat{f}, \delta)}} \left[\mathrm{KL} \left(g(x_c, u) \| \hat{f}^*_\delta(x_c, x_s) \right) \right]
   +
   \mathbb{E}_{\mathcal{D}^{(\hat{f}, \delta)}} \left[
     g(x_c, u) \log\frac{\hat{f}^*_\delta(x_c, x_s)}{\hat{f}(x_c, x_s)}
    + (1-g(x_c, u)) \log\frac{1-\hat{f}^*_\delta(x_c, x_s)}{1-\hat{f}(x_c, x_s)}
   \right]
   \\[3pt] &  \ \ \ \ + 
  \mathbb{E}_{\mathcal{D}^{(\hat{f}, \delta)}} \left[H \left(g(x_c, u) \right) \right]
\end{align*}

If we define 
\begin{align*}
    \hat{f} &= \underset{f \in \mathcal{F}}{\arg\min}  \ \mathbb{E}_{\mathcal{D}} \left[\mathrm{KL} \left(g(x_c, u) \| f(x_c, x_s) \right) \right] \\
    \hat{f}^*_\delta &= \underset{f \in \mathcal{F}}{\arg\min}  \ \mathbb{E}_{\mathcal{D}^{(\hat{f}, \delta)}} \left[\mathrm{KL} \left(g(x_c, u) \| f(x_c, x_s) \right) \right]
\end{align*}

we can interpret the the decomposition as follows:

\begin{align*}
\mathcal{L}_{\mathrm{CE}}(\hat{f}, \delta) &= \underbrace{\mathbb{E}_{\mathcal{D}^{(\hat{f}, \delta)}} \left[\mathrm{KL} \left(g(x_c, u) \| \hat{f}^*_\delta(x_c, x_s) \right) \right]}_{\text{(incomplete information error)}}
   \\[3pt] &  \ \ \ \ +
   \underbrace{\mathbb{E}_{\mathcal{D}^{(\hat{f}, \delta)}} \left[
     g(x_c, u) \log\frac{\hat{f}^*_\delta(x_c, x_s)}{\hat{f}(x_c, x_s)}
    + (1-g(x_c, u)) \log\frac{1-\hat{f}^*_\delta(x_c, x_s)}{1-\hat{f}(x_c, x_s)}
   \right]}_{\text{(transfer error)}}
   \\[3pt] &  \ \ \ \ + 
  \underbrace{\mathbb{E}_{\mathcal{D}^{(\hat{f}, \delta)}} \left[H \left(g(x_c, u) \right) \right]}_{\text{(entropy)}}    
\end{align*}

\textbf{Note:} all terms are nonnegative.

\begin{itemize}
    \item \textbf{incomplete information error:} KL Divergence is always non-negative. 
    \item \textbf{transfer error:} since we define $\hat{f}^*_\delta = \underset{f \in \mathcal{F}}{\arg\min}  \ \mathbb{E}_{\mathcal{D}^{(\hat{f}, \delta)}} \left[\mathrm{KL} \left(g(x_c, u) \| f(x_c, x_s) \right) \right]$, we have the fact that transfer error is nonnegative (i.e. $\geq 0$). Since  $\mathbb{E}_{\mathcal{D}^{(\hat{f}, \delta)}} \left[\mathrm{KL} \left(g(x_c, u) \| {f}^*_\delta(x_c, x_s) \right) \right] \leq \mathbb{E}_{\mathcal{D}^{(\hat{f}, \delta)}} \left[\mathrm{KL} \left(g(x_c, u) \| \hat{f}(x_c, x_s) \right) \right]$ for all $\hat{f}(x_c, x_s) \in \mathcal{F}$, we have:

\begin{align*}  
\mathcal{L}_{\mathrm{CE}}(\hat{f}, \delta) &= 
    \mathbb{E}_{\mathcal{D}^{(\hat{f}, \delta)}} \left[\mathrm{KL} \left(g(x_c, u) \| \hat{f}(x_c, x_s) \right) \right] 
    + 
    \mathbb{E}_{\mathcal{D}^{(\hat{f}, \delta)}} \left[H \left(g(x_c, u) \right) \right] \\
    & \geq \mathbb{E}_{\mathcal{D}^{(\hat{f}, \delta)}} \left[\mathrm{KL} \left(g(x_c, u) \| \hat{f}^*_\delta(x_c, x_s) \right) \right] 
    + 
    \mathbb{E}_{\mathcal{D}^{(\hat{f}, \delta)}} \left[H \left(g(x_c, u) \right) \right]
    \\[6pt]
    \text{but from the derivation above,}&\text{  we have:}
     \\[6pt]
    \mathcal{L}_{\mathrm{CE}}(\hat{f}, \delta)& = \mathbb{E}_{\mathcal{D}^{(\hat{f}, \delta)}} \left[\mathrm{KL} \left(g(x_c, u) \| \hat{f}^*_\delta(x_c, x_s) \right) \right]
   + 
  \mathbb{E}_{\mathcal{D}^{(\hat{f}, \delta)}} \left[H \left(g(x_c, u) \right) \right] 
  \\ & \hspace{3mm}+
   \mathbb{E}_{\mathcal{D}^{(\hat{f}, \delta)}} \left[
     g(x_c, u) \log\frac{\hat{f}^*_\delta(x_c, x_s)}{\hat{f}(x_c, x_s)}
    + (1-g(x_c, u)) \log\frac{1-\hat{f}^*_\delta(x_c, x_s)}{1-\hat{f}(x_c, x_s)}
   \right]
\end{align*}

implying that the last term (which is transfer error) is nonnegative: $$\mathbb{E}_{ \mathcal{D}^{(\hat{f}, \delta)}} \left[
     g(x_c, u) \log\frac{\hat{f}^*_\delta(x_c, x_s)}{\hat{f}(x_c, x_s)}
    + (1-g(x_c, u)) \log\frac{1-\hat{f}^*_\delta(x_c, x_s)}{1-\hat{f}(x_c, x_s)}
   \right] \geq 0$$

    \item \textbf{entropy:} The entropy of a binary variable $Y$ with success probability $p = g(x_c, u)$ is nonnegative almost surely, since for all $p \in [0, 1]$, entropy $-p\log p - (1-p)\log(1-p)$ is nonnegative and thus the expectation must be nonnegative with probability $1$. 
\end{itemize}

\subsection{Extending to training data with adaptive shifts}
\label{app:ce_decomp_extend}

We can further extend this decomposition by considering that a classifier has access to some strategic behavior when training the initial classifier $\hat{f}$. This means we define the classifier $\hat{f}$ as the optimal classifier under $\delta'$ to some classifier $f'$ (which could be $\hat{f}$ itself), formally: 

\begin{align*}
    \hat{f} &= \underset{f \in \mathcal{F}}{\arg\min}  \ \mathbb{E}_{\mathcal{D}^{(f', \delta')}} \left[\mathrm{KL} \left(g(x_c, u) \| f(x_c, x_s) \right) \right] \\
    \hat{f}^*_\delta &= \underset{f \in \mathcal{F}}{\arg\min}  \ \mathbb{E}_{\mathcal{D}^{(\hat{f}, \delta)}} \left[\mathrm{KL} \left(g(x_c, u) \| f(x_c, x_s) \right) \right]
\end{align*}

The decomposition still holds with nonnegative terms. Incomplete information and entropy are trivially nonnegative from the same reasoning as before. Transfer error also remains nonnegative since we are still selecting $\hat{f}^*_\delta$ to be the optimal classifier on the shifted data, so we still have for all $\hat{f}(x_c, x_s) \in \mathcal{F}$:

\[\mathbb{E}_{\mathcal{D}^{(\hat{f}, \delta)}} \left[\mathrm{KL} \left(g(x_c, u) \| \hat{f}^*_\delta(x_c, x_s) \right) \right] \leq \mathbb{E}_{\mathcal{D}^{(\hat{f}, \delta)}} \left[\mathrm{KL} \left(g(x_c, u) \| \hat{f}(x_c, x_s) \right) \right]\]

This allows us to conclude that even if the learned classifier is able to anticipate some strategic shift or only has access to training data with some shift already present, causal classifiers still provide robustness while spurious classifiers can have arbitrarily large error due to the transfer term.

\subsection{CE Error of Causal Classifier Family}
\label{app:ce_error_analysis_causal}

\subsubsection{Incomplete Information Error}
\label{app:ii_error_causal}

First, we consider families of causal classifiers only, $\mathcal{F}_{causal}$, i.e. classifiers that do not use spurious features for prediction. We defined $\hat{f}^*_\delta$ such that $\hat{f}^*_\delta = \underset{f \in \mathcal{F}_{causal}}{\arg\min}  \ \mathbb{E}_{\mathcal{D}^{(\hat{f}, \delta)}} \left[\mathrm{KL} \left(g(x_c, u) \| f(x_c) \right) \right]$, so the incomplete information error is the minimum value of this objective, and when trying to bound this term, we actually want to bound the minimum. 

If we assume that $\mathcal{F}_{causal}$ includes the minimum of this objective, which is the function $\mathbb{E}_{\mathcal{D}^{(\hat{f}, \delta)}}[Y | X_c]$ (for example a family of classifiers that includes logistic functions must include any linear combination of logistic functions), then: 

\begin{align*}
    &\underset{f \in \mathcal{F}_{causal}}{\min}  \ \mathbb{E}_{x_c, u \sim \mathcal{D}^{(\hat{f}, \delta)}} \left[\mathrm{KL} \left(g(x_c, u) \ \parallel \hat{f}^*(x_c) \right) \right] \\ &= \mathbb{E}_{x_c, u \sim \mathcal{D}^{(\hat{f}, \delta)}} \left[\mathrm{KL} \left(\mathbb{E}[Y | X_c, U] \parallel \mathbb{E}[Y | X_c] \right) \right] 
    \\ &= \mathbb{E}_{x_c,,u \sim \mathcal{D}^{(\hat{f}, \delta)} }\left[ \mathbb{E}[Y | X_c, U] \log\frac{\mathbb{E}[Y | X_c, U]}{\mathbb{E}[Y | X_c]} + (1-\mathbb{E}[Y | X_c, U]) \log\frac{1-\mathbb{E}[Y | X_c, U]}{1-\mathbb{E}[Y | X_c])}
   \right]
   \\ &= \mathbb{E}_{x_c,,u \sim \mathcal{D}^{(\hat{f}, \delta)} }\left[ \mathbb{P}(Y = 1 | X_c, U) \log\frac{\mathbb{P}(Y = 1 | X_c, U)}{\mathbb{P}(Y = 1 | X_c)} + \mathbb{P}(Y = 0 | X_c, U) \log\frac{\mathbb{P}(Y = 0 | X_c, U)}{\mathbb{P}(Y = 0 | X_c)}
   \right]
    \\ &=  \mathbb{E}_{x_c,u \sim \mathcal{D}^{(\hat{f}, \delta)}} \left[ \mathbb{E}_{y | x_c, u \sim \mathcal{D}^{(\hat{f}, \delta)}} \left[\log\frac{\mathbb{P}(Y | X_c, U)}{\mathbb{P}(Y | X_c)}\right]\right]
    \\ &= \mathbb{E}_{x_c \sim \mathcal{D}^{(\hat{f}, \delta)}} \left[ \sum_{u \in U}  \mathbb{P}(U | X_c) \cdot \ \mathbb{E}_{y | x_c, u \sim \mathcal{D}^{(\hat{f}, \delta)}} \left[\log\frac{\mathbb{P}(Y | X_c, U)}{\mathbb{P}(Y | X_c)}\right]\right]
    \\ &= \mathbb{E}_{x_c \sim \mathcal{D}^{(\hat{f}, \delta)}} \left[ \sum_{u \in U, y \in Y} \mathbb{P}(U) * \mathbb{P}(Y | X_c , U) \left[\log\frac{\mathbb{P}(Y | X_c, U)}{\mathbb{P}(Y | X_c)}\right]\right]
    \\ &= \mathbb{E}_{x_c \sim \mathcal{D}^{(\hat{f}, \delta)}} \left[ \sum_{u \in U, y \in Y} \mathbb{P}(Y, U | X_c) \left[\log\frac{\mathbb{P}(Y, U | X_c)/\mathbb{P}(U | X_c)}{\mathbb{P}(Y | X_c)}\right]\right]
    \\ &= \mathbb{E}_{x_c \sim \mathcal{D}^{(\hat{f}, \delta)} }\left[ \sum_{u \in U, y \in Y} \mathbb{P}(Y, U | X_c) \left[\log\frac{\mathbb{P}(Y, U | X_c)}{\mathbb{P}(Y | X_c) \ \mathbb{P}(U | X_c)}\right]\right]
    \\ &= \mathbb{E}_{\mathcal{D}^{(\hat{f}, \delta)} }\Big[\mathrm{KL} \Big( \mathbb{P}(Y, U | X_c)  \parallel \mathbb{P}(Y | X_c) \mathbb{P}(U | X_c)\Big)\Big]
    \\ &= I(Y ; U | X_c)
    \\ &= H(U | X_c) - H(U | X_c, Y)
    \\ &\leq H(U | X_c) = H(U) 
\end{align*}

Therefore, the transfer bias of a causal classifier is equal to the conditional mutual information between $Y$ and $U$ given $X_c$ and can be upper-bounded by the entropy of $U$.

\subsubsection{Transfer Error}
\label{app:transfer_error_causal}

We are only considering families of causal classifiers. We previously defined:  
\begin{align*}
    \hat{f} &= \underset{\hat{f} \in \mathcal{F}}{\arg\min} \ \mathbb{E}_{\mathcal{D}}\left[\mathrm{KL} \left(g(x_c, u) \| \hat{f}(x_c) \right) \right] \\ \hat{f}^*_\delta &= \underset{\hat{f}^* \in \mathcal{F}}{\arg\min} \ \mathbb{E}_{\mathcal{D}^{(\hat{f}, \delta)}}\left[\mathrm{KL} \left(g(x_c, u) \| \hat{f}^*(x_c) \right) \right]
\end{align*}

which are minimized at $\mathbb{E}_{\mathcal{D}}[Y | X_c]$ and $\mathbb{E}_{\mathcal{D}^{(\hat{f}, \delta)}}[Y | X_c]$ respectively.  $U$ cannot be intervened on and causal mechanisms are invariant, so the only difference is a covariate shift of the distribution, meaning they are both equal to $\mathbb{E}[Y | X_c] =\mathbb{E}\left[\mathbb{E}_U\left[Y | X_c, U\right]\right]$. Since we already assumed this function to be part of the classifier family, $\hat{f} =\hat{f}^*$ and therefore transfer error is $0$ when training with a causal classifier family. 

\subsubsection{Entropy Error}
\label{app:entropy_error_causal}

The outcome $Y$ is a binary variable i.e. $\in [0,1]$ so:
\[
H(Y) = \mathbb{E}[-\log P(Y)] \leq \max_x \left(-\log P(Y = y)\right) \leq \left(-\log 0.5\right) = 1.
\]
Hence, \( H(Y) \leq 1 \), (or entropy of any binary variable) with equality iff \( P(X = 0) = P(X = 1) = 0.5 \).

\subsection{CE Error of Spurious Classifier Family}
\label{app:ce_error_analysis_spurious}

\subsubsection{Incomplete Information Error}
\label{app:ii_error_spurious}

Repeating the same logic as before with a causal classifier family, when considering a family of classifiers $\mathcal{F}$ that now uses all features, including spurious ones, $\hat{f}^*_\delta$ is defined such that $\hat{f}^*_\delta = \underset{f \in \mathcal{F}_{all}}{\arg\min}  \ \mathbb{E}_{\mathcal{D}^{(\hat{f}, \delta)}} \left[\mathrm{KL} \left(g(x_c, u) \| f(x_c, x_s) \right) \right]$, thus the incomplete information error term is again the minimum value of this objective.  

If we similarly assume that $\mathcal{F}$ includes the function $\mathbb{E}_{\mathcal{D}^{(\hat{f}, \delta)}}[Y | X_c, X_s]$ : 

\begin{align*}
    \underset{f \in \mathcal{F}}{\min}  \ \mathbb{E}&_{x_c,x_s, u \sim \mathcal{D}^{(\hat{f}, \delta)}} \left[\mathrm{KL} \left(g(x_c, u) \| \hat{f}^*(x_c, x_s) \right) \right] = \mathbb{E}_{x_c, x_s, u \sim \mathcal{D}^{(\hat{f}, \delta)}} \left[\mathrm{KL} \left(\mathbb{E}[Y | X_c, U] \| \mathbb{E}[Y | X_c, X_s] \right) \right] 
    \\ &= \mathbb{E}_{x_c, x_s, u \sim \mathcal{D}^{(\hat{f}, \delta)}} \left[ \mathbb{E}[Y | X_c, U] \log\frac{\mathbb{E}[Y | X_c, U]}{\mathbb{E}[Y | X_c, X_s]}
    + (1-\mathbb{E}[Y | X_c, U]) \log\frac{1-\mathbb{E}[Y | X_c, U]}{1-\mathbb{E}[Y | X_c, X_s])}
   \right]
   \\ &= \mathbb{E}_{x_c, x_s, u \sim \mathcal{D}^{(\hat{f}, \delta)}} \left[ \mathbb{P}(Y = 1 | X_c, U) \log\frac{\mathbb{P}(Y = 1 | X_c, U)}{\mathbb{P}(Y = 1 | X_c, X_s)}
    + \mathbb{P}(Y = 0 | X_c, U) \log\frac{\mathbb{P}(Y = 0 | X_c, U)}{\mathbb{P}(Y = 0 | X_c, X_s))}
   \right]
    \\ &=  \mathbb{E}_{x_c, x_s, u \sim \mathcal{D}^{(\hat{f}, \delta)}} \left[ \mathbb{E}_{y | x_c, u \sim \mathcal{D}^{(\hat{f}, \delta)}} \left[\log\frac{\mathbb{P}(Y | X_c, U)}{\mathbb{P}(Y | X_c, X_s)}\right]\right]
    \\ &=  \mathbb{E}_{x_c, x_s, u \sim \mathcal{D}^{(\hat{f}, \delta)}} \left[ \mathbb{E}_{y | x_c, u \sim \mathcal{D}^{(\hat{f}, \delta)}} \left[\log\frac{\mathbb{P}(Y | X_c, U, X_s)}{\mathbb{P}(Y | X_c, X_s)}\right]\right]
    \\ &= \mathbb{E}_{x_c, x_s \sim \mathcal{D}^{(\hat{f}, \delta)} }\left[ \sum_{u \in U}  \mathbb{P}(U | X_s) \cdot \ \mathbb{E}_{y | x_c, u \sim \mathcal{D}^{(\hat{f}, \delta)}} \left[\log\frac{\mathbb{P}(Y | X_c, U, X_s)}{\mathbb{P}(Y | X_c, X_s)}\right]\right]
    \\ &= \mathbb{E}_{x_c, x_s \sim \mathcal{D}^{(\hat{f}, \delta)} }\left[ \sum_{u \in U, y \in Y} \mathbb{P}(U | X_s) * \mathbb{P}(Y | X_c, U, X_s) \left[\log\frac{\mathbb{P}(Y | X_c, U, X_s)}{\mathbb{P}(Y | X_c, X_s)}\right]\right]
    \\ &= \mathbb{E}_{x_c, x_s \sim \mathcal{D}^{(\hat{f}, \delta)} }\left[ \sum_{u \in U, y \in Y} \mathbb{P}(Y, U | X_c, X_s) \left[\log\frac{\mathbb{P}(Y, U | X_c, X_s)/\mathbb{P}(U | X_c, X_s)}{\mathbb{P}(Y | X_c, X_s)}\right]\right]
    \\ &= \mathbb{E}_{x_c, x_s \sim \mathcal{D}^{(\hat{f}, \delta)} }\left[ \sum_{u \in U, y \in Y} \mathbb{P}(Y, U | X_c, X_s) \left[\log\frac{\mathbb{P}(Y, U | X_c, X_s)}{\mathbb{P}(Y | X_c, X_s) \ \mathbb{P}(U | X_c, X_s)}\right]\right]
    \\ &= \mathbb{E}_{\mathcal{D}^{(\hat{f}, \delta)} }\Big[\mathrm{KL} \Big( \mathbb{P}(Y, U | X_c, X_S) \parallel \mathbb{P}(Y | X_c, X_s) \mathbb{P}(U | X_c, X_s)\Big)\Big]
    \\ &= I(Y ; U | X_c, X_s)
    \\ &= H(U | X_c, X_s) - H(U | X_c, X_s, Y)
    \\ &\leq H(U | X_c, X_s) = H(U | X_s) \leq H(U) 
\end{align*}

Therefore, the transfer bias of a spurious classifier is equal to the conditional mutual information of $Y$ and $U$ given $X_c$ and $X_s$ i.e. all features. It can also be upper-bounded by the conditional entropy of $U$ given $X_s$, or more loosely bounded by entropy of $U$.

\subsubsection{Transfer Error}
\label{app:transfer_error_spurious}

Now considering families of classifiers $\mathcal{F}$ that use all features, we have: 
\begin{align*}
\hat{f} &= \underset{f \in \mathcal{F}}{\arg\min} \ \mathbb{E}_{\mathcal{D}}\left[\mathrm{KL} \left(g(x_c, u) \| f(x_c, x_s) \right) \right] \\ \hat{f}^*_\delta &= \underset{f \in \mathcal{F}}{\arg\min} \ \mathbb{E}_{\mathcal{D}^{(\hat{f}, \delta)}}\left[\mathrm{KL} \left(g(x_c, u) \| f(x_c, x_s) \right) \right]    
\end{align*} 
which are minimized at $\mathbb{E}_{D}[Y | X_s, X_c]$ and $\mathbb{E}_{\mathcal{D}^{(\hat{f}, \delta)}}[Y | X_s, X_c]$ respectively. Only causal mechanisms are invariant, so these conditional expectations can be arbitrarily different based on the distribution shift. Thus, transfer error can be arbitrarily large, since KL divergence can range from $0$ to $\infty$ (for example if there are $x \in \mathcal{X}$ where $g(x_c, u)$ is very high and $\hat{f}(x_c, x_s)$ is very small, this will have a large KL value). 

\subsubsection{Analysis for unbounded transfer error with thresholded classifier}
\label{app:transfer_error_spurious_analysis}

To make this more concrete, consider the following setting: given that minimizing cross-entropy loss learns a probability estimator $\hat{f}(x): \mathcal{X} \rightarrow [0, 1]$,  institutions may require candidates to exceed a certain probability threshold for a positive prediction (which we previously assumed was $0.5$ or $50\%$); for instance, accepting a loan application only when the estimated repayment probability exceeds $\tau \in [0,1]$:
$$\hat y = h(x) = \mathbf{1}_{\{\hat{f}(x) \geq \tau\}}$$
The transfer error is then affected by three key factors: the cost function $c(x, x')$, which determines which features agents adapt and their relative cost; the budget $\delta$, which bounds how much effort agents can spend; and the threshold $\tau$, which determines the decision boundary agents adapt towards.

We can further analyze the transfer error by partitioning the input space as $\mathcal{X} = \mathcal{X}^{\mathrm{adapt}} \cup \mathcal{X}^{\mathrm{stay}}$, where $\mathcal{X}^{\mathrm{adapt}} = \{x : P_D(x) \neq P_{D^{(\hat{f},\delta)}}(x)\}$ contains points whose distribution changed under adaptation, and $\mathcal{X}^{\mathrm{stay}}$ its complement. Let $p_a = P(x \in \mathcal{X}^{\mathrm{adapt}})$ and $p_s = 1 - p_a$. The transfer error splits as:
\begin{align}
    \text{transfer error} \;
    &= \; p_a \, \mathbb{E}_{D^{(\hat{f},\delta)}}\!\left[\left. g \log \frac{\hat{f}^*_\delta}{\hat{f}} + (1-g) \log \frac{1-\hat{f}^*_\delta}{1-\hat{f}} \;\right|\; x \in \mathcal{X}^{\mathrm{adapt}} \right] \notag \\
    &\quad + \; p_s \, \mathbb{E}_{D}\!\left[\left. g \log \frac{\hat{f}^*_\delta}{\hat{f}} + (1-g) \log \frac{1-\hat{f}^*_\delta}{1-\hat{f}} \;\right|\; x \in \mathcal{X}^{\mathrm{stay}} \right],
\end{align}
where the stay term uses $D$ since $D$ and $D^{(\hat{f},\delta)}$ agree on $\mathcal{X}^{\mathrm{stay}}$ by construction. On $\mathcal{X}^{\mathrm{adapt}}$, all points have been moved to the decision boundary $\partial_h$, so $\hat{f}(x_c, x_s) = \tau$. Since $\tau$ is constant on this region, the adapt term separates as:
\begin{equation}
    \mathbb{E}_{D^{(\hat{f},\delta)}}\!\left[\left. g \log \hat{f}^*_\delta + (1-g) \log(1-\hat{f}^*_\delta) \;\right|\; x \in \mathcal{X}^{\mathrm{adapt}}\right] - \bar{g}\log\tau - (1-\bar{g})\log(1-\tau),
    \label{eq:transfer-adapt-tau}
\end{equation}
where $\bar{g} = \mathbb{E}_{D^{(\hat{f},\delta)}}[g(x_c, u) \mid x \in \mathcal{X}^{\mathrm{adapt}}]$. 

The transfer error on adapted points is amplified under two conditions that expand $\mathcal{X}^{\mathrm{adapt}}$: 
\begin{itemize}
    \item $\boldsymbol{\delta}$: increasing the adaptation budget $\delta$ allows more points to reach the decision boundary, increasing $p_a$ and giving this term more weight.
    \item $\mathbf{c(x, x')}$: since the cost function is a weighted $\ell_p$-norm, reducing the cost of perturbing spurious features $\mu_s$ has a compounding effect: for a fixed $\delta$, more agents can cross the boundary by modifying spurious features, increasing $p_a$; simultaneously these agents allocate proportionally more of their budget to changing $X_s$ rather than $X_c$, so the causal features of adapted points are perturbed less, keeping $\bar{g}$ bounded away from $1$.
\end{itemize} 

Furthermore, as $\tau \to 1^-$, the transfer error on adapted points diverges entirely, driven by the $-(1-\bar{g})\log(1-\tau)$ component, whenever $\bar{g} < 1$; that is, whenever any adapted point has nonzero negative-class probability. Since $g(x_c, u)$ depends only on causal features and the unobserved variable, and if agents entering $\mathcal{X}^{\mathrm{adapt}}$ do so by significantly changing spurious features, the expanding region includes points with diverse values of $g$, ensuring $\bar{g} < 1$. This illustrates why classifiers in $\mathcal{F}_{\mathrm{all}}$ that rely on spurious features risk unbounded transfer error when strategic behavior is unknown: $\delta$ and $\mu_s$ control how much agents adapt and how much of that adaptation targets spurious features, which leads to gaming, while more demanding $\tau$ makes the consequences of this gaming increasingly severe for the spurious classifier.

\subsubsection{Entropy Error}
\label{app:entropy_error_spurious}

Same as with the causal classifier case, the outcome $Y$ is a binary variable i.e. $\in [0,1]$ so \( H(Y) \leq 1 \), (or entropy of any binary variable).

\section{Additional Experimental Results}
\label{app:additional_experiments}

\subsection{Robustness under varying ambiguity.}

\label{app:robustness_add_exp}

We highlight a trade-off in post-adaptation cross-entropy loss in \Cref{the:robustness}: causal classifiers incur incomplete information loss by ignoring the informative spurious features while spurious classifiers incur transfer loss due to changing optimal predictive distributions. 
The theory allows us to form another prediction, related to our findings in \Cref{section_01_opt}: if we further bound how influential the latent $U$ is on the value of the outcome $Y$ (limiting the incomplete information loss further), a causal classifier trained on pre-adaptation data should achieve lower post-adaptation CE loss compared to their spurious counterparts. 
To implement this idea, we minimally modify the outcome model to be:
\begin{align}
    y_{\mathrm{sco}}  \triangleq  w_c X_c + w_u U\cdot\mathbf{1_{\{X_c \in (-\mathrm{max}\text{-}\mathrm{gap}, 0)\}} } + b
\end{align}
where max-gap ensures that latent $U$ impacts $Y$ in a bounded region.  
\Cref{fig:max_gap_exp} visualizes the post-adaptation CE loss difference between the pre-adaptation optimal causal and spurious classifiers (blue is better i.e. causal has an advantage) as the adaptation budget $\delta$ and max-gap vary.  We simulated the data using the setting described in \Cref{sec:xps} with a modified $y_{\mathrm{sco}}$, and used sklearn's Logistic Regression model to train the optimal classifiers with static data (both pre and post-strategic shift). As expected, for larger values of $\delta$ (inducing a bigger distribution shift) and max-gap, transfer loss of a spurious classifier is worse than information loss from masking spurious features, giving causal classifiers the advantage. 

\begin{figure}[H]
\centering
\includegraphics[width=0.45\linewidth]{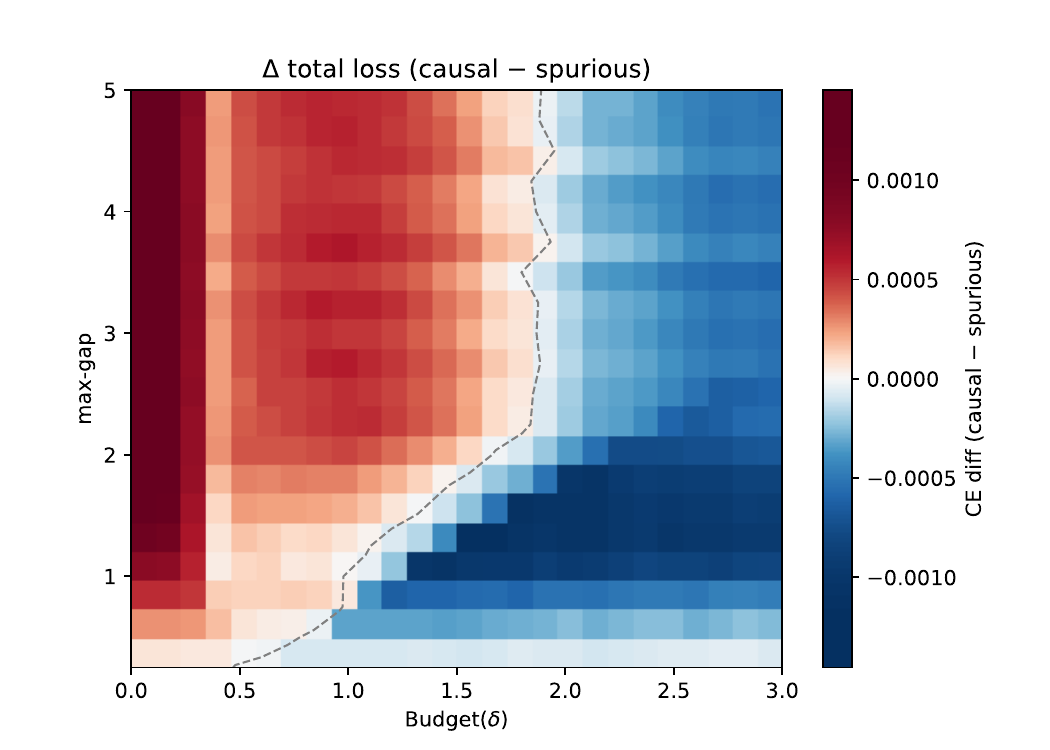}
\caption{Simulations with varying max-gap. When there is a bounded ambiguous region through max-gap in the $y_{\mathrm{sco}}$, a causal classifier trained on pre-adaptation data can have an advantage after sufficiently large $\delta$ (denoted by the blue region, grey dashed line represents regret $= 0$).}
\label{fig:max_gap_exp}
\end{figure}

However, in order to see the effect of incurring a significant amount of \emph{transfer error} when exploiting spurious features, we need to alter the adaptation incentives, specifically such that agents' are encouraged to adapt their spurious features more when a classifier puts weight on it. While large enough $\delta$ allows for agents to make larger strategic shifts, this may not be enough if they do not shift adversarially. 
We therefore explicitly model the agents' cost function of adapting features. We previously assumed the cost of adapting features is a standard $L_p$-norm, but we can instead consider \emph{weighted} $L_p$-norm in order to control how much agents' adapt their spurious features. By lowering cost of changing $x_s$, agents will be incentivized more drastic shifts along $x_s$, which can be adversarial for even slightly spurious classifiers. Based on \Cref{the:robustness}, this leads to \emph{significant} transfer error for spurious classifiers trained on pre-adaptation data, even with some incomplete information from settings like stochastic $Y$ with an unbounded $max\textit{-}gap$. This is because the agents will be incentivized to game and adapt to regions of $\mathcal{X}_s$ where the pre-adaptation classifier $\hat{f}$ very poorly estimates the post-adaptation $\hat{f} = \mathbb{E}_{\mathcal{D}^{(\hat{f}, \delta)}}[Y | X]$.
This leads the post-adaptation transfer error term to dominate the causal incomplete information term when there is an incentive to significantly game. 

We empirically show that the pre-adaptation causal classifier leads to robustness in these settings by reducing the cost of $x_s$. We simulate data using the described setting in \Cref{sec:xps}, and use sklearn's Logistic Regression model to train the classifiers on static data (both pre and post-strategic shift). We consider the cost function to be weighted $L_2$-norm of the two features in our experiments: \begin{align}c(x, x') = \sqrt{\sum_j \mu_j\, (x'_j - x_j)^2} = \sqrt{\mu_s\, (x'_s - x_s)^2 + \mu_c\, (x'_c - x_c)^2}.\end{align} This allows us to confirm our hypothesis that as the cost of changing the spurious feature, $\mu_s$, decreases, then the spurious transfer error can increases arbitrarily, which can lead spurious classifiers to have significantly more error in the worst case. On the other hand, causal classifiers will have relatively stable error, even when the cost of changing the spurious feature is larger. 
To ensure causal classifier actually provides robustness, we also look at the performance of the classifiers when the costs are switched, meaning the cost of changing $x_c$ is much less than $x_s$.

\begin{figure}[H]
\centering
\begin{subfigure}[!ht]{0.85\linewidth}
\centering
\includegraphics[width=\linewidth]{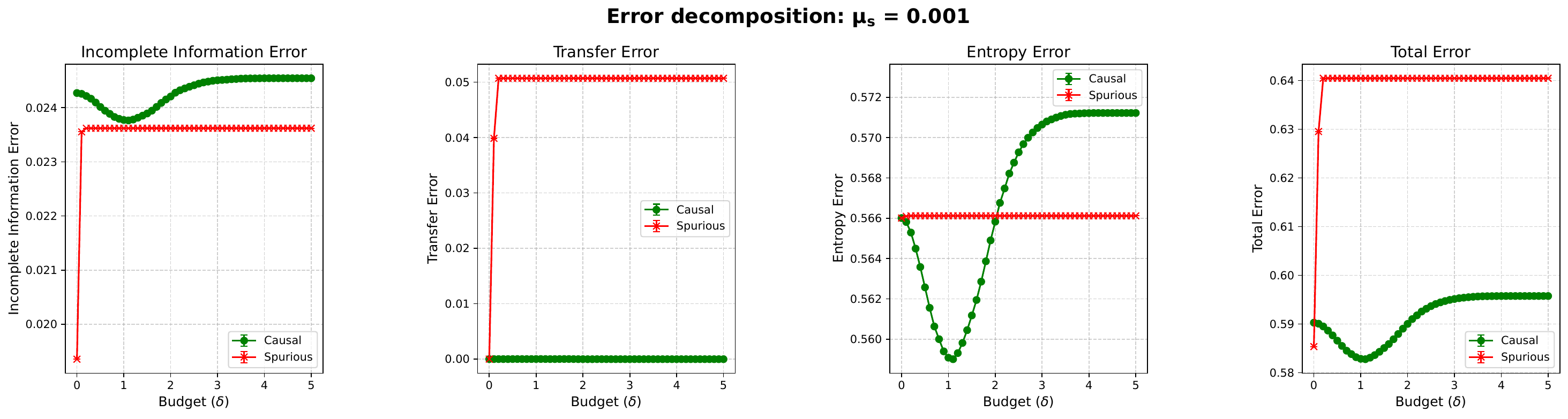}
\caption{Cost function weights: $(\mu_s = 0.001, \mu_c = 1)$}
\label{fig:CE_robust_mu_s}
\end{subfigure}

\vspace{1em}

\begin{subfigure}[!ht]{0.85\linewidth}
\centering
\includegraphics[width=\linewidth]{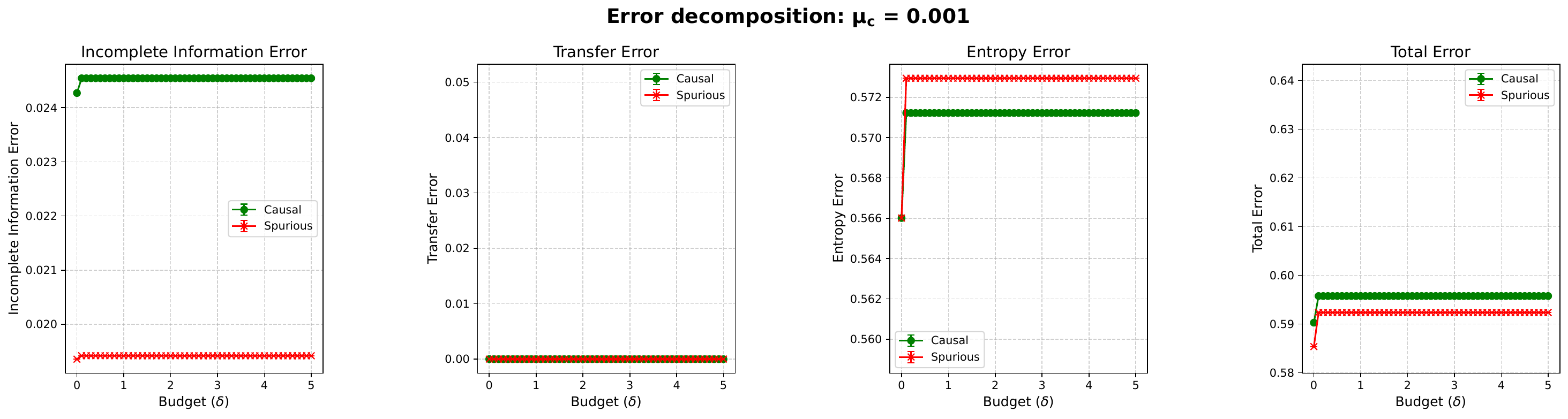}
\caption{Cost function weights: $(\mu_s = 1, \mu_c = 0.001)$}
\label{fig:CE_robust_mu_c}
\end{subfigure}

\caption{Simulations with varying cost. Top: When the cost of changing spurious feature is significantly less that the causal feature (i.e. $\mu_s \ll \mu_c$), transfer error leads the spurious classifier to have a \emph{significant} disadvantage after strategic shifts. Bottom: On the other hand when $\mu_c \ll \mu_s$, transfer error leads the causal classifier to have a \emph{slight} disadvantage after strategic shifts due to the incomplete information error difference. }
\label{fig:spur_cost}
\end{figure}

From \Cref{fig:spur_cost}, we observe that when $\mu_c \ll \mu_s$, the spurious advantage is relatively minor compared to the causal advantage  when the cost is reversed, due to the transfer error incurred by the spurious classifer. More specifically, when $\mu_c \ll \mu_s$, both classifier achieve $0$ transfer error because agents are not incentivized to game significantly, meaning the spurious feature remains a good proxy of the unobserved causal feature $U$. This leads the same spurious classifier to be optimal even after strategic shifts. However, when $\mu_s \ll \mu_c$, the spurious feature is cheaper and easier to modify, which incentivizes agents to game more and \emph{intervene} on their spurious feature. This results in $X_s$ becoming a bad proxy of $U$ after strategic shifts, and thus the optimal post-adaptation spurious classifier changes. This leads to significant transfer error for the spurious classifier, giving the pre-adaptaion causal classifier an advantage overall. 

The incomplete information also shows the effect of the spurious feature becoming a bad proxy of the unobserved causal feature. In the left-most graph of \Cref{fig:CE_robust_mu_s}, the  incomplete information error of the spurious classifier increases after strategic shifts. This behavior results from $X_s$ becoming a weaker signal for $U$, and thus the optimal spurious classifier post-adaptation must mainly rely on the causal feature, the same as the causal classifier. When $X_s$ remains a good proxy of $U$, which we see in the left graph of \Cref{fig:CE_robust_mu_c}, the spurious classifier has a stable advantage due to the stable incomplete information error gap. 
From the second to last graphs, there are some slight variations in entropy due to the strategic shifts to the causal and spurious classifiers, but neither are large enough ($\approx 10^-3$) to affect which classifier has an advantage; they in fact just make the ``advantage" gap in both cases slightly smaller.  

Overall, this confirms our interpretation of the decomposition in \Cref{sec:ce_robust} and robustness of causal classifiers. The causal advantage in \Cref{fig:CE_robust_mu_s} is due to the spurious classifier's transfer error, even though the spurious incomplete information error simultaneously increases slightly. On the other hand in  \Cref{fig:CE_robust_mu_c}, the slight advantage of the spurious classifier comes from the small advantage of using the spurious feature and resulting difference in incomplete information error. 

\subsection{Semi-synthetic Data}
\label{app:semi_synthetic_add_exp}

Because it is impossible to observe true post-adaptation data after agents strategically modify their features, our main experiments are conducted on fully synthetic data, allowing us to precisely control the data-generating process and ensure that the assumptions required by our theory hold. In particular, our theoretical results depend on knowledge of the causal structure and on the ability to compute counterfactual outcomes under strategic shifts, making experiments on purely real-world data infeasible for our purposes.

To incorporate real-world feature distributions while retaining these guarantees, we additionally conduct semi-synthetic experiments. In this setting, the observed features $X$ are taken from the real world Credit Risk dataset from Kaggle, while the remaining variables (namely $U$ and $Y$) are generated synthetically according to a known causal model \citep{semisynthetic_dataset}. This allows us to evaluate our results under realistic feature distributions without introducing ambiguity about the underlying causal structure or counterfactual outcomes. 

Specifically, 
\begin{itemize}
    \item From the dataset, we define the causal feature $X^C$ to be the applicant’s income and spurious feature $X^S$ to be their age (which should not directly cause whether someone will pay back a loan but is likely to be correlated, for example confounded by whether the applicant has a savings account). 
    \item We generate $U$ as a Bernoulli variable with probability $p = \sigma(X^S)$ so that $U$ and $X^S$ are correlated. 
    \item We generate $Y$ as a function of $X^C$ and $U$ as described in the fully synethetic experiments. 
\end{itemize}

Across all experimental settings, we observe the same qualitative behavior as in the fully synthetic case. In \Cref{fig:01_robust_semi}, the optimal post-adaptation classifier again converges to a causal classifier once the adaptation budget $\delta$ is sufficiently large, while classifiers relying on spurious features degrade as $\delta$ increases when trained on data with limited adaptation. In \Cref{fig:incentives_align_semi}, the incentive alignment experiments similarly show that the incentives of the institution and agents align for sufficiently large $\delta_2$ (the boundary of the region is denoted by the grey dashed line). Finally, the robustness experiments in \Cref{fig:CE_robust_semi} reproduce the same tradeoff between transfer error and incomplete-information error, which we see through both varying max-gap and different cost functions for adaptation. 

\begin{figure}[H]
\begin{minipage}[c]{0.64\linewidth}
        \centering
          \includegraphics[width=\linewidth]{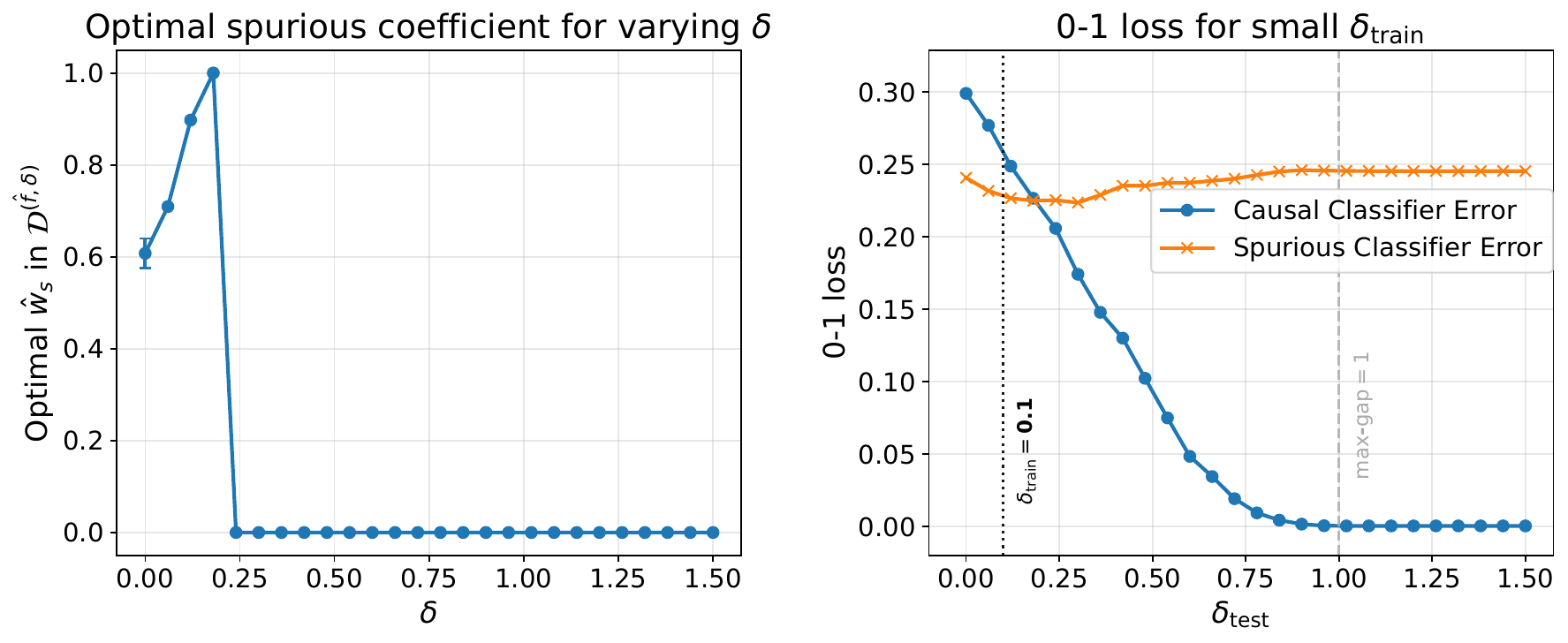}
          \caption{0-1 Loss (semi-synthetic data, cf. \Cref{fig:01_robust}). 
         To obtain error bars we resample $U$ five times, obtaining datasets with varying $Y$ and train-test splits on the entire dataset.  
          }
          \label{fig:01_robust_semi}
\end{minipage}
\hspace{2em}
\begin{minipage}[c]{0.3\linewidth}
    \centering
    \includegraphics[width=\linewidth]{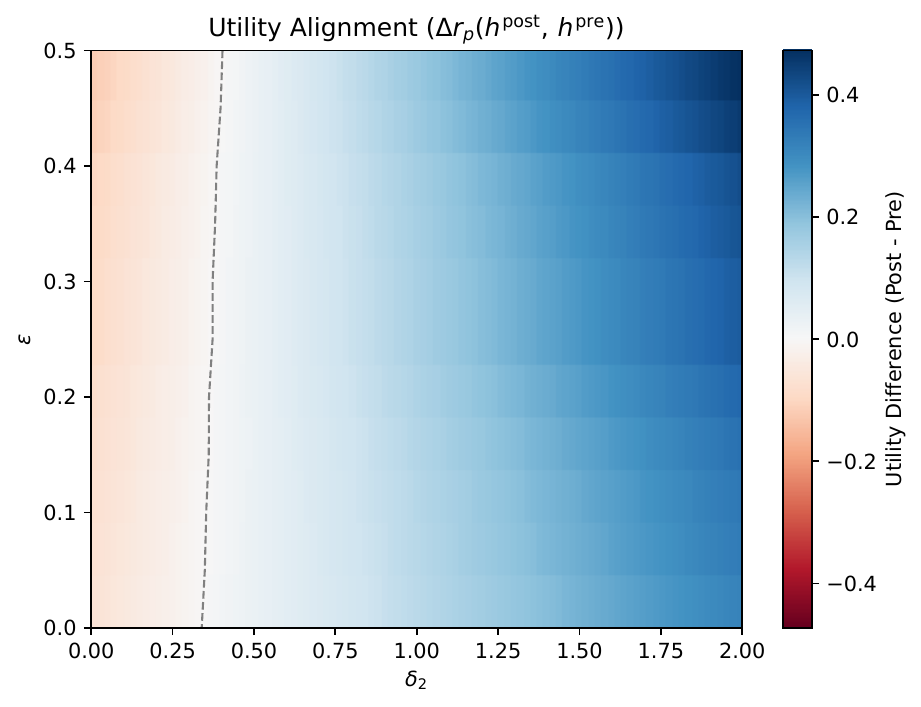}
    \caption{Incentive Alignment (semi-syntheic data, cf. \Cref{fig:incentives_align}).
    }
    \label{fig:incentives_align_semi} 
\end{minipage}
\end{figure}

\begin{figure}[H]
\centering
\begin{minipage}[c]{0.33\linewidth}
    \centering
    \includegraphics[width=\linewidth]{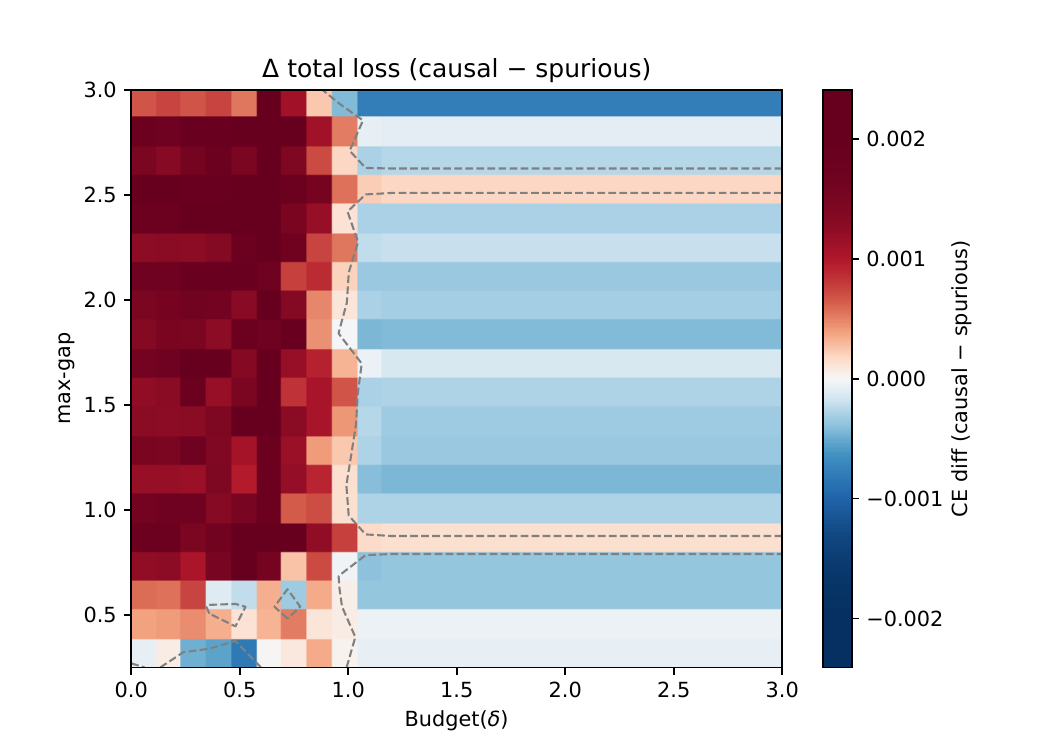}
    \subcaption{Varying max-gap (cf.\ \Cref{fig:max_gap_exp}).}
    \label{fig:max_gap_exp_semi}
\end{minipage}
\hfill
\begin{minipage}[c]{0.64\linewidth}
    \centering
    \includegraphics[width=\linewidth]{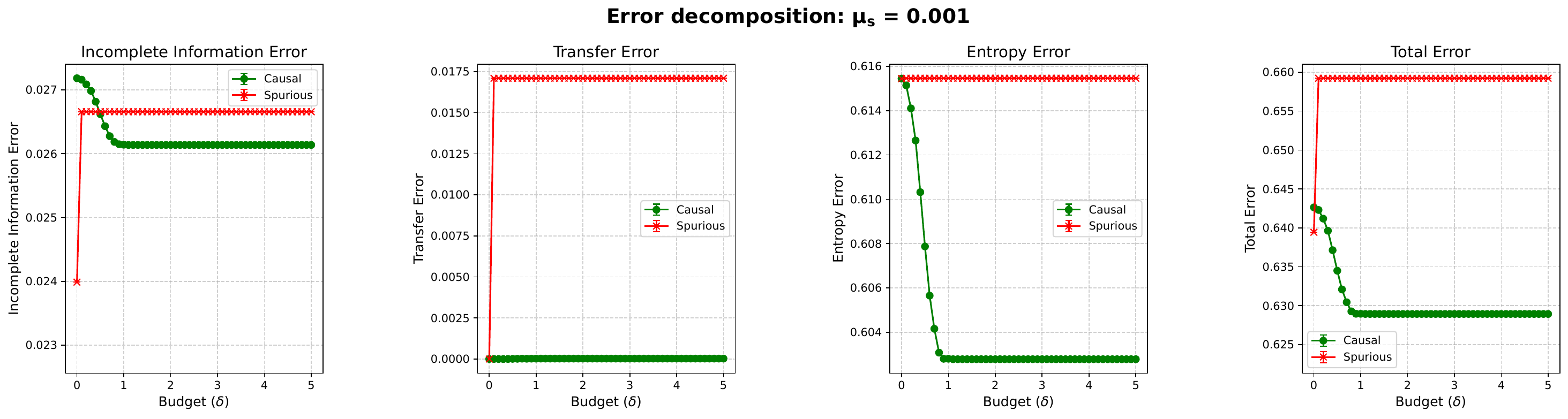}
    \subcaption{Cost weights: $(\mu_s = 0.001, \mu_c = 1)$}
    \label{fig:CE_robust_mu_s_semi}
    \vspace{0.5em}
    \includegraphics[width=\linewidth]{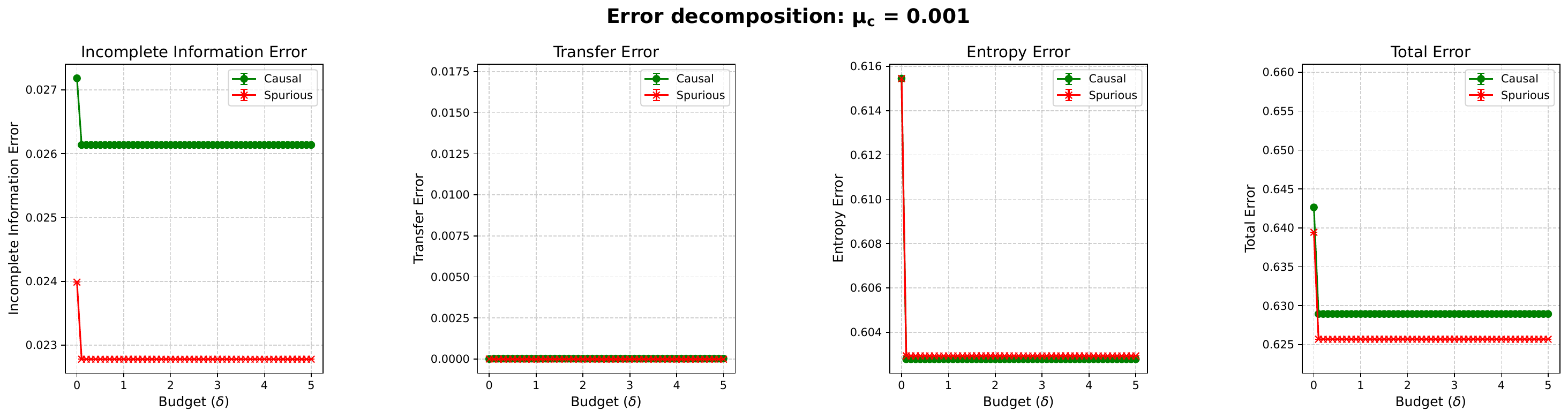}
    \subcaption{Cost weights: $(\mu_s = 1, \mu_c = 0.001)$}
    \label{fig:CE_robust_mu_c_semi}
\end{minipage}
\caption{Robustness of causal classifiers ross-entropy loss (semi-synthetic data). (a) Varying max-gap (cf.\ \Cref{fig:max_gap_exp}). (b)--(c) Varying cost function weights (cf.\ \Cref{fig:spur_cost}).}
\label{fig:CE_robust_semi}
\end{figure}

\section{Example Setting}
\label{app:example_section}

Consider the data generating process below, where $X$ is sampled from a uniform distribution.

\begin{align*}
  U &\sim \operatorname{Bern}(0.5) \\
  X_c &\sim \mathcal{U}(-1,1) \\
  X_s &\sim \mathcal{U}(-1+U,1+U) \\
    y&=\mathds 1\{ X_c + 0.5 U + b\geq0\}
\end{align*}

Under this setting, with knowledge of $U$ it would be possible to obtain zero $\ell_{0-1}$, by the following classifier:

$h(x,u)=
\begin{cases}
    \mathds 1\{x_c\geq0\} & \mathrm{if} \,U=0\\
    \mathds 1\{x_c\geq-0.5\} & \mathrm{if} \,U=1
\end{cases}$

\begin{figure}[!h]
    \centering
    \begin{minipage}[t]{0.48\linewidth}
        \centering
        \includegraphics[height=5cm]{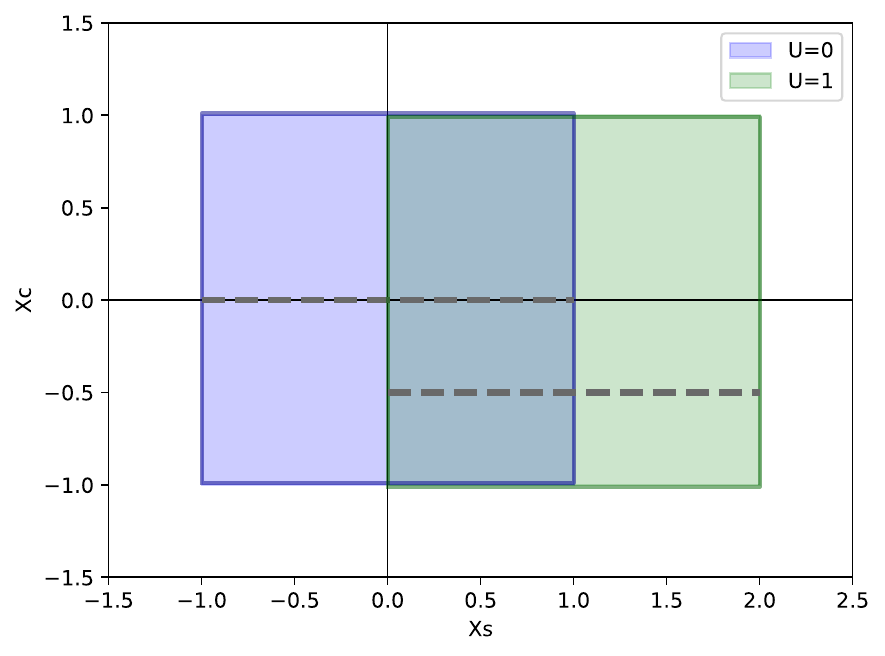}
    \end{minipage}
    \hspace{-0.2\linewidth}
    \begin{minipage}[t]{0.48\linewidth}
        \centering
        \includegraphics[height=5cm, trim={0 -33pt 0 -3pt}, clip]{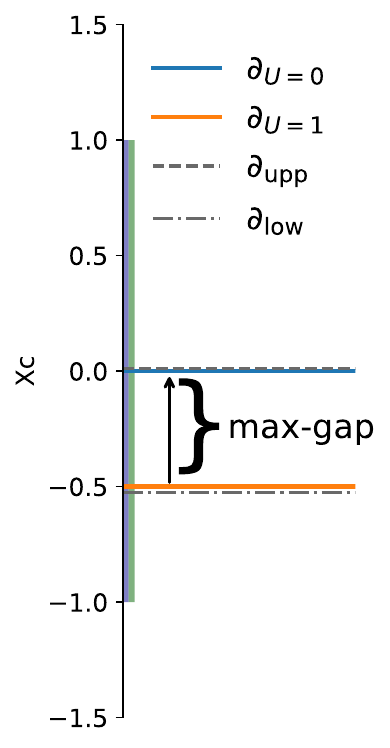}
    \end{minipage}
    \caption{Data generating process where $X$ is sampled from a uniform distribution, and $u\in\{0,1\}$. Dashed grey lines are the correct decision boundaries for $U=0$ and $U=1$. On the right-hand side we have the analogous of Figure~\ref{fig:bounds} for this example, where $\mathcal{X}_c\subset \mathbb R$ has one dimension. As long as $\delta\geq0.5$, where max-gap$=0.5$, our theory predicts the causal classifier $h_c(x_c)\triangleq \mathds 1\{x_c\geq0\}$ obtains zero $\ell_{0-1}$ post-adaptation.}
    \label{fig:dgp_xunif}
\end{figure}

However, without knowledge of $U$, there is an ambiguous region of $X$ pre-adpatation. We can see in Figure~\ref{fig:dgp_xunif} that, for $\delta\geq 0.5$, points can be moved out of this ambiguous region to obtain zero $\ell_{0-1}$ (Theorem~\ref{the:causal_01opt}). For $\delta<$max-gap, optimality of a causal $h_c(x_c)$ is more nuanced, and
we study the optimal classifier for this particular example in the next section.

\subsection{Characterizing post-adaptation \texorpdfstring{$\ell_{0-1}$}{l0-1} minimizers for \texorpdfstring{$\delta<$}{delta<}max-gap}

When $\delta<$ max-gap, existing theory does not predict whether a causal classifier will remain optimal post-adaptation. In this example, when $\delta<0.5$ we lose the guarantee of Theorem~\ref{the:causal_01opt}. However, we show that the same causal classifier remains optimal for any $\delta\geq \frac{1}{3}$, among linear classifiers.

We study the minimizer of $\ell_{0-1}$, among a family of linear classifiers defined as $h_\mathrm{lin}(x_c,x_s)=\mathds 1 \{x_c\geq ax_s+b\}$. We are interested in minimizing post-adaptation loss: $\underset{a,b}{\min} \underset{x,u}{\mathbb E}[\ell_{0-1}(h_\mathrm{lin}(\Delta_{h_\mathrm{lin}}(x;\delta)),y(x,u))]$.

The closed-form expression for $\ell(a,b)$ depends on how the boundary $\partial_{h_\mathrm{lin}}$ intersects the edges of the support over $X$. Here we assume adaptation cost is $L_1$-norm, resulting in the adaptation boundary $x_c = ax_s+b-\delta$, where agents are indifferent between adapting or not. Other $L_p$-norms consist simply of shifting this boundary. We split the loss in four regions: FN for $U=0$, FN for $U=1$, FP for $U=0$ and FP for $U=1$.

For FN, the shape of the error region depends on where the left and right edges of the support $\mathcal X$ for $U=u$ are intersected by the adaptation boundary $x_c = ax_s+b-\delta$. Denote this boundary as $h_{\delta\mathrm{\textit{-}lin}}(x_s) \triangleq ax_s+b-\delta$. For FN in $U=0$, we have 9 different possible shapes for the loss region of given $a,b$ parameters. Which of the 9 geometrical figures we get from a certain $a,b$ depends on whether $h_{\delta\mathrm{\textit{-}lin}}(-1)\leq0$, $h_{\delta\mathrm{\textit{-}lin}}(-1)\in(0,1)$ or $h_{\delta\mathrm{\textit{-}lin}}(-1)\geq1$, and $h_{\delta\mathrm{\textit{-}lin}}(1)\leq0$, $h_{\delta\mathrm{\textit{-}lin}}(1)\in(0,1)$ or $h_{\delta\mathrm{\textit{-}lin}}(1)\geq1$, resulting in $3\times3$ combinations. Similarly for FN in $U=1$, the error shape depends on whether $h_{\delta\mathrm{\textit{-}lin}}(-0.5)\leq0$, $h_{\delta\mathrm{\textit{-}lin}}(0)\in(-0.5,1)$ or $h_{\delta\mathrm{\textit{-}lin}}(0)\geq1$, and $h_{\delta\mathrm{\textit{-}lin}}(2)\leq-0.5$, $h_{\delta\mathrm{\textit{-}lin}}(2)\in(-0.5,1)$ or $h_{\delta\mathrm{\textit{-}lin}}(2)\geq1$.

For FP, the shape of the error region depends on the classification boundary. We denote the classification boundary by $h_{0\mathrm{\textit{-}lin}}(x_s)$, where $\delta=0$. We also obtain 9 possible shapes, which for $U=0$ depends on whether $h_{0\mathrm{\textit{-}lin}}(-1)\geq0$, $h_{0\mathrm{\textit{-}lin}}(-1)\in(0,-1+\delta)$ or $h_{0\mathrm{\textit{-}lin}}(-1)\leq-1+\delta$, and $h_{0\mathrm{\textit{-}lin}}(1)\geq0$, $h_{0\mathrm{\textit{-}lin}}(1)\in(0,-1+\delta)$ or $h_{0\mathrm{\textit{-}lin}}(1)\leq-1+\delta$. Similarly for $U=1$, the area of error of FP depends on whether $h_{0\mathrm{\textit{-}lin}}(0)\geq-0.5$, $h_{0\mathrm{\textit{-}lin}}(0)\in(-0.5,-1+\delta)$ or $h_{0\mathrm{\textit{-}lin}}(0)\leq-1+\delta$, and $h_{0\mathrm{\textit{-}lin}}(2)\geq-0.5$, $h_{0\mathrm{\textit{-}lin}}(2)\in(-0.5,-1+\delta)$ or $h_{0\mathrm{\textit{-}lin}}(2)\leq-1+\delta$. 

To circumvent the combinatorial explosion presented above, we resort to grid-search in this example to identify the shape of error induced by optimal $a,b$ parameters, for a given $\delta$. Let a spurious classifier be an $h(x_c,x_s)$ such that $x_s$ can change the output ($a\neq0$ in this parameterization). We identify a condition among the possible $9^4$ described above, which is optimal for at least $\delta\in[0.2,\frac{1}{3})$, arriving at a spurious classifier with the following description for $\ell_{0-1}$:

\begin{gather*}
    \ell_\mathrm{spur}(a,b)=\frac{1}{4}[
\underbrace {
(b-a-\delta)(1+\frac{\delta-b}{a})}_{\text{FN for }U=0}
+\underbrace{(0.5+\delta-b)(\frac{-0.5+\delta-b}{a})}_{\text{FN for }U=1}
+\underbrace{(1+\frac{b}{a})(2\delta-b-a)}_{\text{FP for }U=0 }
+\underbrace{0}_{\text{FP for }U=1}
]
\end{gather*}

By solving for $\frac{\partial\cdot}{\partial a}=0$ and $\frac{\partial\cdot}{\partial b}=0$ we get: $a_\mathrm{spur}=-\sqrt{1/12-0.5\delta^2}, b_\mathrm{spur}=\delta-1/6$.

\begin{figure}[thbp]
    \centering
    \begin{subfigure}{0.45\textwidth}
        \centering
        \includegraphics[width=\linewidth]{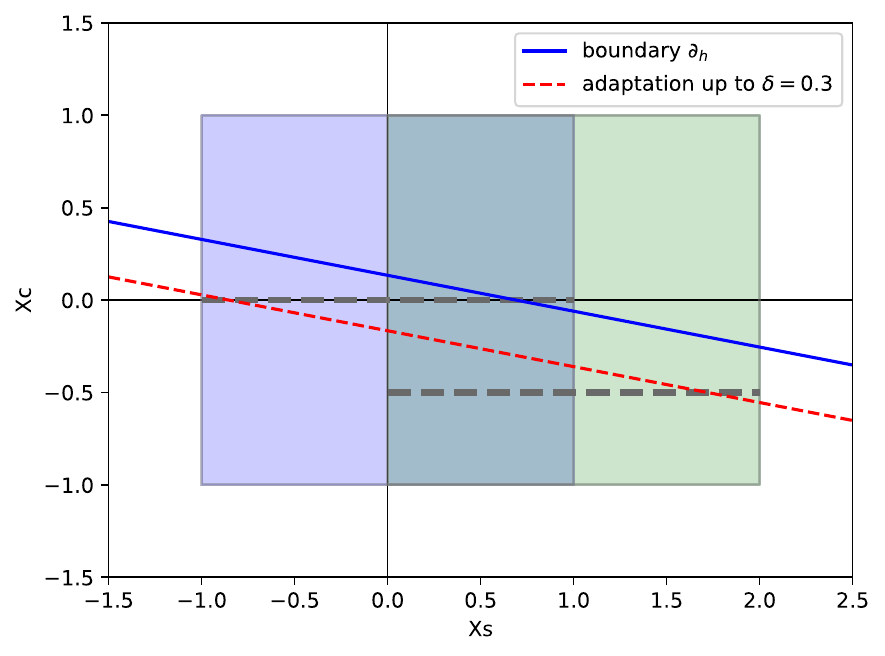}
        \caption{Optimal classifier for $\delta=0.3$.}
        \label{fig:x_unif_spurious}
    \end{subfigure}
    \hfill
    \begin{subfigure}{0.45\textwidth}
        \centering
        \includegraphics[width=\linewidth]{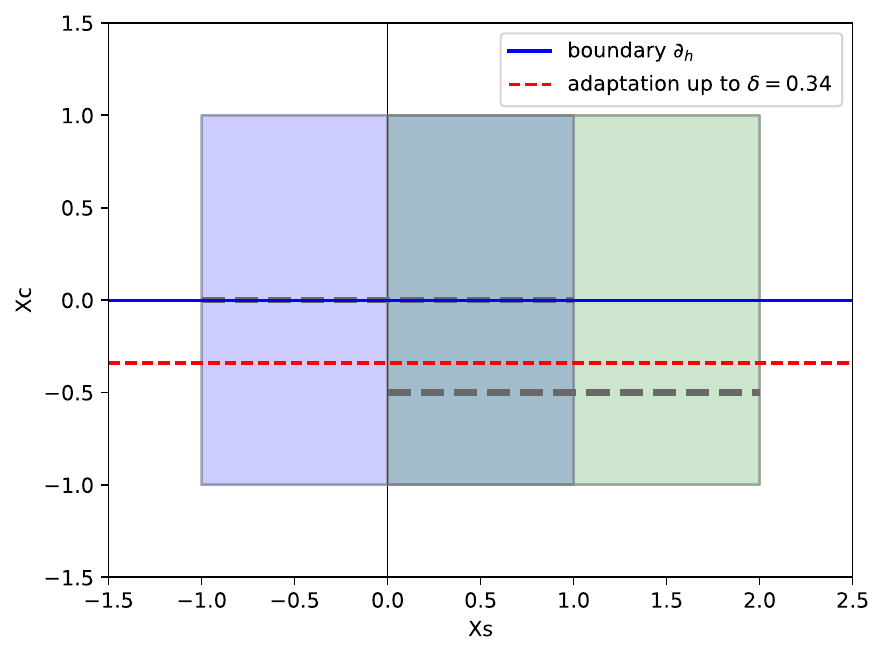}
        \caption{Optimal classifier for $\delta=0.34$.}
        \label{fig:x_unif_causal}
    \end{subfigure}
    
    \caption{Optimal classifiers for different $\delta$, near the transition between causal and spurious solutions.}
    \label{fig:x_unif_optimal}
\end{figure}

This analytical solution matches grid-search results and is depicted in Figure~\ref{fig:x_unif_spurious}. As $\delta$ increases, we see though grid-search a sudden transition to the causal classifier $a_\mathrm{caus}=b_\mathrm{caus}=0$, with loss $\ell_\mathrm{caus}(a,b)=2(0.5-\delta)$.

By setting $\ell_\mathrm{caus}(a=0,b=0)=\ell_\mathrm{spur}(a=-\sqrt{1/12-0.5\delta^2},b=\delta-1/6)$ and solving for $\delta$, we arrive at $\delta=\frac{1}{3}$. This identifies a phase transition, as $\delta$ increases, switching suddenly from an optimal spurious classifier into a causal one, when $\delta$ increases beyond $\frac{1}{3}$.

Interestingly, our theory predicted this causal classifier to be optimal for $\delta>\frac{1}{2}$, which is our max-gap for this setting. When $\delta<$max-gap it becomes challenging to charaterize the optimal solution generally. However, we provide here an example where the same causal solution remains optimal even for $\delta<$max-gap, particularly for any $\delta\in(\frac{1}{3},+\infty)$ despite max-gap=$\frac{1}{2}$.

\subsection{Including negative impact of FP's in population's utility}
\label{sec:utility_example_app}

Consider the generative process described in the beginning of this section, and $\delta=0.5$.
We compute below the population's utility (Definition~\ref{def:longterm_pred_util}) after deploying 2 classifiers: the one optimizing post-adaptation $\ell_{0-1}$ (which is causal, since $\delta\geq$max-gap) and the one optimizing pre-adaptation $\ell_{0-1}$ (which uses the spurious feature). We then characterize alignment in this example, for varying $\delta_2$.

\subsubsection{Best \texorpdfstring{$\ell_{0-1}$}{l0-1} post-adapt}

Assume $\delta=0.5$. The causal classifier $a=0, b=0$ obtains zero $\ell_{0-1}$ post-adaptation (Theorem~\ref{the:causal_01opt}).

\begin{figure}[!h]
    \centering
    \includegraphics[width=0.6\linewidth]{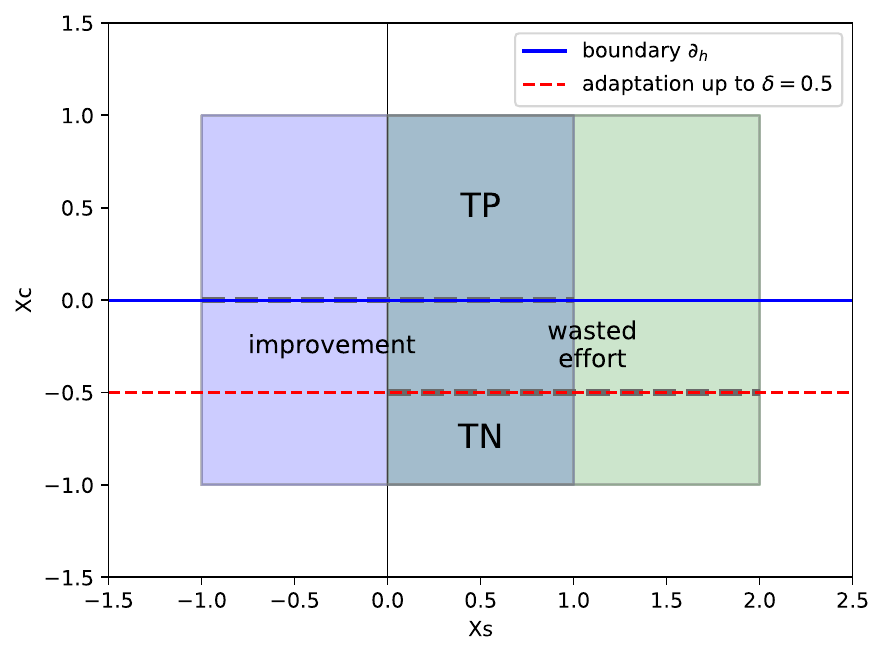}
    \caption{Best $\ell_{0-1}$ post-adaptation classification boundary $\partial_h$, for $\delta=0.5$, under the data generating process of section \ref{app:example_section}. Regions where points adapted are labeled following Definition~\ref{def:all_adaptations_app}.}
    \label{fig:best_post_unif}
\end{figure}

This yields:

\begin{itemize}
    \item 75\% positive labels
    \item Zero FPs
    \item Expected cost of adapting $\mathbb{E}_x[c(x,\Delta_{\hat f}(x))]=0.0625$. This comes from $\int_{x_c,x_s} c(x,\Delta_{\hat f}(x)) \mathbb{P}(x_c,x_s)=\frac{0.5}{2}*\frac{1}{4}+0*\frac{3}{4}$, where $\frac{0.5}{2}$ is the average cost of adapting for those who move. (area of a triangle with base $0.5$ and height $0.5$ — those who move from furthest spend $0.5$, the ones touching the boundary spend $0$).
\end{itemize}

The expected utility of the predicted agents $\mathbb E[r_p(h^\mathrm{post},x,u)]$ is then $.75*.5-.0625-0\delta_2=0.3125$.

\subsubsection{Best \texorpdfstring{$\ell_{0-1}$}{l0-1} pre-adapt}

The parameters for the linear classifier that minimizes pre-adaptation $\ell_{0-1}$ are $a=-0.224, b=-0.138$, which uses the spurious feature since $a\neq0$ (obtained through grid-search).

\begin{figure}[thbp]
    \centering
    \begin{subfigure}{0.45\textwidth}
        \centering
        \includegraphics[width=\linewidth]{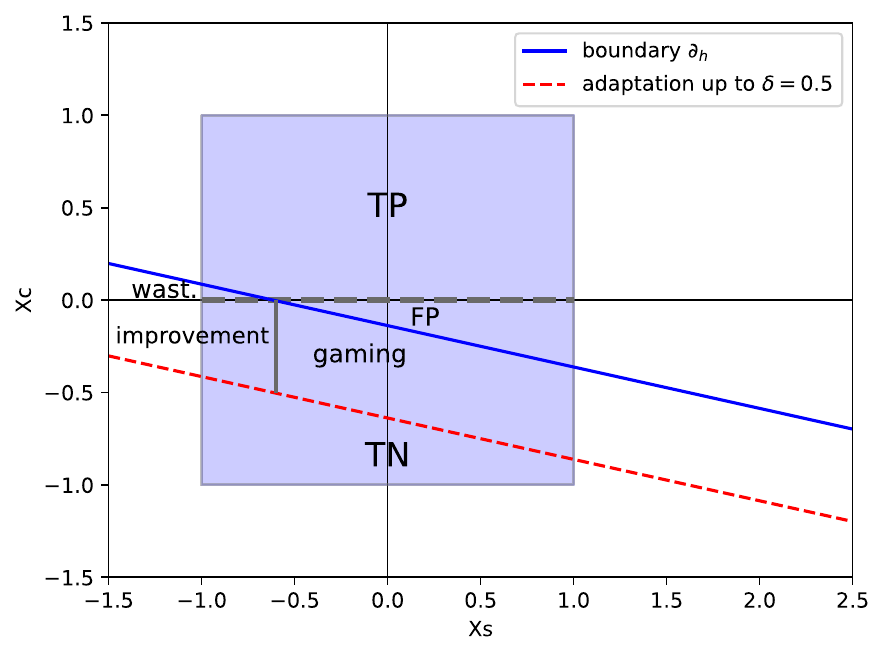}
        \caption{Support for $U=0$.}
        \label{fig:best_pre_u0}
    \end{subfigure}
    \hfill
    \begin{subfigure}{0.45\textwidth}
        \centering
        \includegraphics[width=\linewidth]{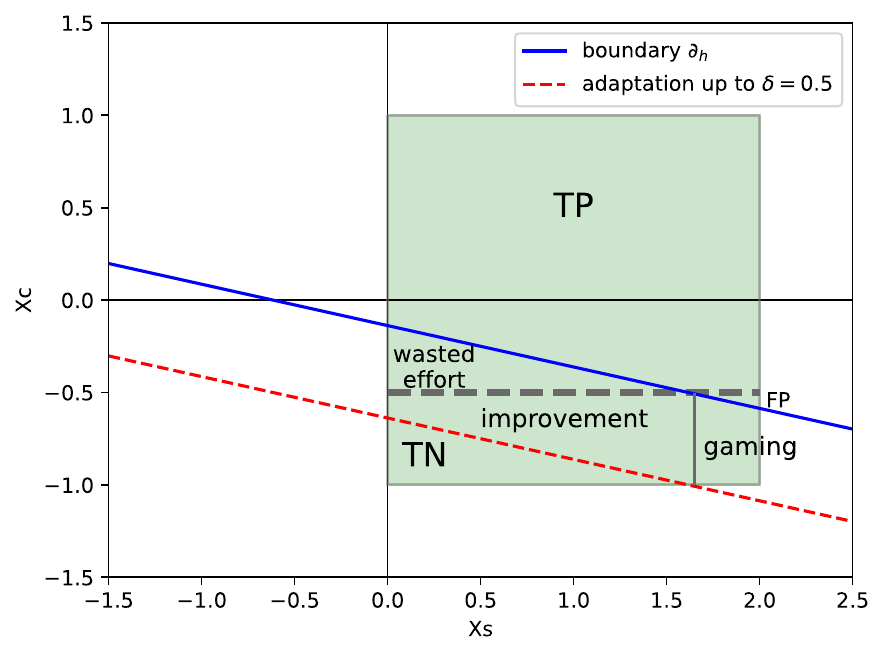}
        \caption{Support for $U=1$.}
        \label{fig:best_pre_u1}
    \end{subfigure}
    
    \caption{Best $\ell_{0-1}$ pre-adaptation classification boundary $\partial_h$, under the data generating process of section \ref{app:example_section}. Regions where points adapted are labeled following Definition~\ref{def:all_adaptations_app}.}
    \label{fig:best_pre_unif}
\end{figure}

\begin{itemize}
    \item 87.25\% positive labels
    \item 16.16\% FPs
    \item Expected cost of adapting is $0.0542<\mathbb{E}_x[c(x,\Delta_{\hat f}(x))]<0.0557$. It is $0.0625-d$ where $0.0068<d<0.0083$ is the small triangle outside support $\mathcal X$ (with area $.384*.086/2$) multiplied by its cost (which varies between .5 and .414 inside the triangle).
\end{itemize}

The expected utility for the predicted agents $\mathbb E[r_p(h^\mathrm{pre},x,u)]$ is then upper-bounded by $.8725*.5-0.0542-.1616\delta_2$ and lower-bounded by $.8725*.5-0.0557-.1616\delta_2$, or $0.38055-.1616\delta_2<\mathbb E[r_p(h,x,u)]<0.38205-.1616\delta_2$.

\subsubsection{Comparing predicted agents' utility \texorpdfstring{$r_p$}{rp} under \texorpdfstring{$h^\mathrm{pre}$}{h-pre} or \texorpdfstring{$h^\mathrm{post}$}{h-post}}
\label{app:longterm_alignment_example}

Agent utility from post-adaptation minimizer: $\mathbb E[r_p(h^\mathrm{post},x,u)]$ = 0.3125.

Agent utility from pre-adaptation minimizer: $0.38055-.1616\delta_2<\mathbb E[r_p(h^\mathrm{pre},x,u)]<0.38205-.1616\delta_2$.

The lowest $\delta_2^*$ which guarantees alignment ($\Delta r_p(h^{\text{post}},h^{\text{pre}})\geq0$) is then upper-bounded by $\delta_2^*<0.43$, for $\delta=0.5$.

Even if $\delta>\delta_2$ (i.e., FPs still provide a small gain) we have alignment, $\Delta r_p(h^{\text{post}},h^{\text{pre}})\geq0$. In words, the population still prefers on average a strategic institution that prevents gaming, even though this would bring a small benefit to agents compared to being classified as TN.

\section{Incentive Alignment}

Consider the following definition of $h$-change:

\begin{definition}
\label{def:cond_h_change}
    (Conditional $h$-change) Expected change in utility when switching from classifier $h$ into $h'$ for role $k$, where $k\in\{p,i\}$ is either predicted agent $p$ or institution $i$. 
For the subset of agents where $Y=y$, we have $\Delta r_{p|Y=y}$ and $\Delta r_{i|Y=y}$:
\begin{align*}
    \\ 
    \Delta r_{k|Y=y}(h',h)= \mathbb E_{x,u|y}[r_{k}(h',\Delta_{h'}(x),u)]
    -\mathbb E_{x,u|y}[r_{k}(h,\Delta_{h}(x),u)]
\end{align*}
\end{definition}

\subsection{Static Alignment}

\begin{proposition}
\label{pro:static_align}
    (Static alignment) Assume $\Delta_h(x)=x$ and $\delta_2=\epsilon=0$. For any pair of classifiers $(\hat f',\hat f)$, institution's goals match the goals of agents whose $y=1$ but not of $y=0$, in the following sense:
\begin{align*}
    \Delta r_{p|Y=1}(h',h)>0\iff \Delta r_{i|Y=1}(h',h)>0
    \\
    \Delta r_{p|Y=0}(h',h)<0\iff \Delta r_{i|Y=0}(h',h)>0
\end{align*}
\end{proposition}

\begin{definition}
    (FN, FP, TN, TP) Consider the following auxiliary definitions:
    
    $\texttt{FN}(h(x),y)=\mathds{1}\{h(x)=0\,\mathrm{and}\,y=1\}$

    $\texttt{FP}(h(x),y)=\mathds{1}\{h(x)=1\,\mathrm{and}\,y=0\}$

    $\texttt{TN}(h(x),y)=\mathds{1}\{h(x)=0\,\mathrm{and}\,y=0\}$

    $\texttt{TP}(h(x),y)=\mathds{1}\{h(x)=1\,\mathrm{and}\,y=0\}$
\end{definition}

\begin{proof}
    With static $(y,x)$ pairs, the only changes which can occur from a change in $h$ are:
    \begin{enumerate}
        \item for $y=1$, FN $\leftrightarrow$ TP
        \item for $y=0$, FP $\leftrightarrow$ TN
    \end{enumerate}
    From the definitions of $r_p$ (\ref{def:longterm_pred_util}) and $r_i$ (\ref{def:longterm_inst_util}), we see that in case 1. both $r_p$ and $r_i$ increase with TP and decrease with FN. For case 2. TN increases $r_i$ and reduces $r_p$, while FP reduces $r_i$ and increases $r_p$. \end{proof}

This follows from the fact that agents always gain when a negative prediction ($h(\Delta_{h}(X))=0$) is changed into positive ($h'(\Delta_{h'}(X))=1$), while the institution only benefits from this change if true outcome is also positive ($y=1$). It is illustrated by Table~\ref{tab:static_alignment}. 

\definecolor{lightgreen}{RGB}{220, 255, 220}
\definecolor{lightred}{RGB}{255, 230, 230}

\begin{table}[ht!]
\centering
\begin{tabular}{|c|c|c|}
\hline
& \textbf{Institution} & \textbf{Predicted Agents} \\
&                      & \textbf{$\delta_2=0$ ($\delta_2>\delta$)} \\
\hline
\rowcolor{lightgreen}
\textbf{TP} & \Checkmark & \Checkmark\,\,\,\,\,\,\,\,\,\,\\
\hline
\rowcolor{lightgreen}
\textbf{FN} & X & X\,\,\,\,\,\,\,\,\,\,\,\,\, \\
\hline
\rowcolor{lightred}
\textbf{TN} & \Checkmark & X (\Checkmark) \\
\hline
\rowcolor{lightred}
\textbf{FP} & X & \Checkmark (X) \\
\hline
\end{tabular}
\caption{Contribution of static points ($\Delta_h(x)=x$) to utility of institution ($r_i$) and of predicted agents ($r_p$). \Checkmark denotes a non-negative contribution, and X a decrease in utility. Green rows indicate similar impact for institution and agents (contributing to alignment), and red otherwise. Parentheses in predicted agents indicate contribution when $\delta_2>\delta$, if it does not match contribution when $\delta_2=0$.}
\label{tab:static_alignment}
\end{table}

\subsection{Dynamic Alignment}
\label{app:dynamic_align}

To understand how $h^{\mathrm{post}}$ may differ from $h^{\mathrm{pre}}$, we can ask what additional information the institution receives when it knows the correct adaptation model $\Delta_{h}(x;\delta)$, instead of wrongly assuming static behavior $\Delta_{h}(x)=x$. Agents adapt to increase their utility, which is only possible when, pre-adaptation, they would be assigned a negative label $h(x)=0$. This allows enumerating all four kinds of additional information for the institution when it becomes strategic (i.e., when it anticipates agents' adaptation). For a given classifier $h^{\text{pre}}$, a population of agents may respond through (some of) the following adaptations:

\begin{definition}
\label{def:all_adaptations_app}
    (All possible adaptations) For any $x:\Delta_h(x)\neq x$, its adaptation must fall in one of four categories, which we name below. This is because a point only moves to switch from negative $h(x)=0$ into positive $h(\Delta_h(x))=1$. We also characterize their impact on $r_i$.

    \begin{itemize}
    \item TN$\rightarrow$ TP (\textit{improvement}) $r_i\rightarrow$
    \item FN$\rightarrow$ TP (\textit{wasted effort}) $r_i\nearrow$
    \item TN$\rightarrow$ FP (\textit{gaming}) $r_i\searrow$
    \item FN$\rightarrow$ FP (\textit{reversed incentive}) $r_i\rightarrow$
\end{itemize}
\end{definition}

The arrows describe whether $r_i$ goes up ($\nearrow$), down ($\searrow$), or stays constant ($\rightarrow$) after agent-adaptation, for a fixed classifier, assuming short-term goals ($\delta_2=\epsilon=0$). Table~\ref{tab:dyn_alignment} also depicts this. An illustration is provided in the running example of \S~\ref{sec:utility_example_app} for intuition.

\begin{table}[ht!]
\centering
\begin{tabular}{|c|c|c|}
\hline
& \textbf{Institution} & \textbf{Predicted Agents} \\
& \textbf{$\epsilon=0$ ($\epsilon>0$)} & \textbf{$\delta_2=0$ ($\delta_2>\delta$) } \\
\hline
\rowcolor{gray!20}
Improvement (\textbf{TN$\rightarrow$ TP}) & = (\Checkmark) & \Checkmark \,\,\,\,\,\,\,\,\, \\
\hline
\rowcolor{lightgreen}
Wasted effort (\textbf{FN$\rightarrow$ TP}) & \Checkmark\,\,\,\,\,\,\,\,\,\, & \Checkmark \,\,\,\,\,\,\,\,\, \\
\hline
\rowcolor{lightred}
Gaming (\textbf{TN$\rightarrow$FP}) & X\,\,\,\,\,\,\,\,\,\, & \Checkmark (X)\\
\hline
\rowcolor{lightred}
Reversed incentive (\textbf{FN$\rightarrow$FP}) & X\,\,\,\,\,\,\,\,\,\, & \Checkmark (X)  \\
\hline
\end{tabular}
\caption{Contribution of dynamic points ($\Delta_h(x)\neq x$) to utility of institution ($r_i$) and of predicted agents ($r_p$), when they switch from pre-adaptation $x$ into post-adaptation $\Delta_h(x)$. \Checkmark denotes an increase in utility, and X a decrease. Green rows indicate similar impact for institution and agents (contributing to alignment), red indicates opposite impact, and grey indicates indifferent for the institution. Parentheses in predicted agents indicate contribution when $\delta_2>\delta$, if it does not match contribution when $\delta_2=0$. Parentheses in institution indicate contribution when $\epsilon>0$, if it does not match contribution when $\epsilon=0$.}
\label{tab:dyn_alignment}
\end{table}

We can then study which changes an institution can make when switching from $h ^{\text{pre}}(x)$ to $ h ^{\text{post}}(x)$ (when \textit{becoming strategic}). Note that, if our family of hypotheses $\mathcal H$ (where $h\in\mathcal H$) is sufficiently expressive, we can independently estimate $\mathbb P(Y|X)$ for each $X$, and then threshold it to obtain classifier $h(x)$.

\subsubsection{Short-term Alignment}

\begin{lemma} (Support over positive predictions)
\label{lem:positive_y_support_app}
    Let the learned classifier be $h(x)=\mathds 1\{f(x)\geq0\}$, where $f(x)$ estimates $\mathbb P(Y|X)$, and its boundary be $\partial_h=\{x\in\mathcal X:f(x)=0\}$. Assume the family of hypotheses $\mathcal H$ is flexible enough that it can estimate independently $\mathbb P(Y|X)$ for each $x$. Assume full support over points that can adapt towards $h^\mathrm{pre}$ and $h^\mathrm{post}$: $\Delta^{-1}_{h^{\mathrm{pre}}}(\partial_{h^{\mathrm{pre}}}) \cup \Delta^{-1}_{h^{\mathrm{post}}}(\partial_{h^{\mathrm{post}}})\subset \mathcal X$, also $\delta_2=\epsilon=0$, $h^\mathrm{post}\neq h^\mathrm{pre}$ and $y_\mathrm{sco}$ is O-nondecreasing with respect to $x_c$.
    Then, we have that points obtaining a positive prediction under $h^\mathrm{post}$ are a subset of those obtaining a positive prediction under $h^\mathrm{pre}$: $\{x\in\mathcal X: h^{\mathrm{post}}(x)=1\}\subset 
        \{x\in\mathcal X: h^{\mathrm{pre}}(x)=1\} $.
\end{lemma}

\begin{proof}
    Any agent's adaptation must fall in one of the four categories defined in \ref{def:all_adaptations_app}, since points only adapt to switch from $h(x)=0$ into $h(\Delta_h(x))=1$.

    Given $\mathcal{H}$ is flexible enough, $h^\mathrm{post}\neq h^\mathrm{pre}$ only happens as a response to wasted effort or gaming, since the two remaining adaptations do not impact $r_i$.
    
    $y_\mathrm{sco}$ is O-nondecreasing with respect to $x_c$. Hence, to improve utility as a response to gaming or wasted effort, the boundary should move from $h^\mathrm{pre}$ to $h^\mathrm{post}$ along orthant $O_s$. To prevent gaming, it should increase along the direction where $y_\mathrm{sco}$ increases (along $O_s$). Awareness of wasted effort allows to incur in error by being overly demanding (also along $O_s$) knowing that points will self-correct, by switching, from $h^\mathrm{post}(x)$ yielding FN, into $h^\mathrm{post}(\Delta_{h^\mathrm{post}}(x))$ yielding TP. $h^\mathrm{post}$ may gain from increasing wasted effort, if it avoids other errors made by $h^\mathrm{pre}$, such as FPs. Formally, we have that $\forall x_{\mathrm{post\textit{-}}\partial}\in\partial_\mathrm{post}\setminus\partial_\mathrm{pre},\exists\,x_{\mathrm{pre\textit{-}\partial}}\in\partial_\mathrm{pre}\setminus\partial_\mathrm{post}:x_{\mathrm{post\textit{-}}\partial}\in O_s\setminus\vec{0}+x_{\mathrm{pre\textit{-}\partial}}$.

    From the full support assumption:
    
    $\Delta^{-1}_{h^{\mathrm{pre}}}(\partial_{h^{\mathrm{pre}}}) \subset \mathcal X\Rightarrow \forall x_\mathrm{pre\textbf{-}\partial}\in\partial_{h^\mathrm{pre}} \exists\, \tilde x\in \Delta^{-1}_{h^\mathrm{pre}}(x_\mathrm{pre\textbf{-}\partial})\subset\mathcal{X}:\delta=\underset{x'}{\min}\, c(\tilde x,x') \,\mathrm{s.t.}\,(h^\mathrm{pre}(x')=1)$.

    Combining both results above, we have that $\exists\, \tilde x: h^\mathrm{pre}(\Delta_{h^\mathrm{pre}}(\tilde x))=1, h^\mathrm{post}(\Delta_{h^\mathrm{post}}(\tilde x))=0 $, since any point $\tilde x$ whose cost of adapting to $h^\mathrm{pre}$ was $\delta$ and $\Delta_{h^\mathrm{pre}}(\tilde x)\not\in \partial_{h^\mathrm{post}}$, has cost of adapting to $h^\mathrm{post}$ as $\delta'>\delta$, hence will not adapt.

    Analogously, since the post-adapt boundary became more demanding, no point that is positively classified by $h^\mathrm{post}$ is negatively classified by $h^\mathrm{pre}$:
    
    $\forall x_\mathrm{post\textbf{-}\partial}\in \partial_\mathrm{post}\setminus \partial_\mathrm{pre}, \exists\, x_\mathrm{pre\textbf{-}\partial}\in \partial_\mathrm{pre}:x_\mathrm{post\textbf{-}\partial}\in x_\mathrm{pre\textbf{-}\partial}+O_s\Rightarrow\not\exists\, \tilde x: h^\mathrm{post}(\Delta_{h^\mathrm{post}}(\tilde x))=1, h^\mathrm{pre}(\Delta_{h^\mathrm{pre}}(\tilde x))=0 $

    Since we have points $\tilde x$ that lose their positive label post-adapt, and no point that gains one, then:
    
    $\{x\in\mathcal X: h^{\mathrm{post}}(x)=1\}\subset 
        \{x\in\mathcal X: h^{\mathrm{pre}}(x)=1\} $    
\end{proof}

Under short term goals $\delta_2=\epsilon=0$, for $h^\mathrm{post}\neq h^\mathrm{pre}$, a flexible enough $\mathcal{H}$ and O-nondecreasing $y_\mathrm{sco}$, we do not have aligned incentives.

\begin{proposition}
    (Short-term misalignment) 
    Let $\Delta r_{p\mathrm{\textit{-}short}}$ be $h$-change for $\delta_2=0$, and $h^\mathrm{post\textit{-}short}$ be optimal for $\epsilon=0$. Under the assumptions of Lemma~\ref{lem:positive_y_support_app}, we have:

    \begin{align*}
        \Delta r_{p\mathrm{\textit{-}short}}(h^\mathrm{post\textit{-}short},h^\mathrm{pre})<0
    \end{align*}
\end{proposition}

\begin{proof}
    From the proof of Lemma~\ref{lem:positive_y_support_app} we have that $h^\mathrm{post}$ is more demanding than $h^\mathrm{pre}$, in the sense that $\forall x_{\mathrm{post\textit{-}}\partial}\in\partial_\mathrm{post}\setminus\partial_\mathrm{pre},\exists\,x_{\mathrm{pre\textit{-}\partial}}\in\partial_\mathrm{pre}\setminus\partial_\mathrm{post}:x_{\mathrm{post\textit{-}}\partial}\in x_{\mathrm{pre\textit{-}\partial}}+O_s\setminus\vec{0}$.

    Since $y_\mathrm{sco}$ is O-nondecreasing with respect to $x_c$, a point $\tilde x$ is classified as $h(\tilde x)=1$ if $\exists\, x_b\in\partial_h: x_b\prec_{O_s} \tilde x$.
    
    We know $\forall \tilde x, \Delta_{h^\mathrm{pre}}(\tilde x)\neq \tilde x\Rightarrow h(\tilde x)=0$, so when the boundary changes by summing $v\in O_s$ its distance to $\tilde x$ increases in $L_p$-norm.

    Hence $\forall \tilde x,\Delta_{h^\mathrm{post}}(\tilde x)\neq \Delta_{h^\mathrm{pre}}(\tilde x), h(\Delta_{h^\mathrm{post}}(\tilde x))=1:c(\tilde x,\Delta_{h^\mathrm{post}}(\tilde x))>c(\tilde x,\Delta_{h^\mathrm{pre}}(\tilde x))$.

    From Lemma~\ref{lem:positive_y_support_app}, the number of points receiving $\delta$ from $h(\Delta_h(x))=1$ will reduce.

    From the definition of $r_p$ (Definition~\ref{def:longterm_pred_util}), higher cost and reduced $\delta$ determine lower $\underset{x,u}{\mathbb E}[r_p(h,x,u)]$, since we are assuming $\delta_2=0$.

    Therefore $\Delta r_{p\mathrm{\textit{-}short}}(h^\mathrm{post\textit{-}short},h^\mathrm{pre})<0$.
\end{proof}

\subsubsection{Long-term Alignment}
\label{app:long-term-align}

Assuming $\delta_2>0$ and $\epsilon=0$, we have the same $h^\mathrm{post}$ as when $\delta_2=\epsilon=0$, since institution's utility $r_i$ does not change. From Lemma~\ref{lem:positive_y_support_app}, we have that the boundary $\partial_{h^\mathrm{post}}$ becomes more demanding than $\partial_{h^\mathrm{pre}}$, by increasing along $O_s$. Then for any $\tilde x$, if $\Delta_{h^\mathrm{post}}(\tilde x)\neq\Delta_{h^\mathrm{pre}}(\tilde x)$, its change from $h^\mathrm{pre}$ to $h^\mathrm{post}$ must have been one of the categories below:

\begin{itemize}
    \item (maint) Maintained FP or TP at higher cost;
    \item (impr) Switched from gaming to improvement (FP$\rightarrow$ TP);
    \item (TP$\rightarrow$ N) Switched from TP into FN or TN;
    \item (FP$\rightarrow$ TN) Switched from FP into TN.
\end{itemize}

Denote a point's change in adaptation cost as:

$$ \Delta c(x)\triangleq c(x,\Delta_{h^{\mathrm{post}}}(x))-c(x,\Delta_{h^{\mathrm{pre}}}(x))$$

To compute $\Delta r_{\text{p}}( h^{\text{pre}},h^{\text{post}})$ we must then consider how each group's utility $r_p$ changed:
\begin{itemize}
    \item Increased cost of maintaining FP or TP: $-\Delta c(x)$
    \item Gain of improving (FP$\rightarrow$ TP): $\delta_2-\Delta c(x)$
    \item Cost of TP$\rightarrow$ N: $-\delta-\Delta c(x)$
    \item Cost of FP$\rightarrow$ FN: $-(\delta-\delta_2)-\Delta c(x)$
\end{itemize}

For each of the previous groups, denote their densities as: $\mathbb P(\mathrm{maint}):=\int_{x\in \{\mathrm{maint}\}} \mathbb P(x)$, and similarly $\mathbb P(\mathrm{impr})$, $\mathbb P(\mathrm{TP\rightarrow N})$, $\mathbb P(\mathrm{FP\rightarrow TN})$. Their total costs are $c(\mathrm{all}):=\int_{x\in \mathcal X} \Delta c(x)\mathbb P(x)$.

\begin{proposition}
    (Long-term alignment) Considering a flexible enough $\mathcal H$ and $\epsilon=0$, having alignment requires:

    $$\delta_2>\frac{c(\mathrm{all}) 
     +\delta(\mathbb P(\mathrm{TP\rightarrow N})
     +\mathbb P(\mathrm{FP\rightarrow TN}))}
     {\mathbb P(\mathrm{impr})+\mathbb P(\mathrm{FP\rightarrow TN})}$$
    
\end{proposition}

\begin{proof}
    $\Delta r_p(h^{\text{post}},h^{\text{pre}})\geq0 \Leftrightarrow$

    $ \delta_2\mathbb P(\mathrm{impr}) 
    - \delta \mathbb P(\mathrm{TP\rightarrow N})
    -(\delta-\delta_2) \mathbb P(\mathrm{FP\rightarrow TN})
    - c(\mathrm{all})\geq0 \Leftrightarrow$

    $\delta_2>\frac{c(\mathrm{all}) 
     +\delta(\mathbb P(\mathrm{TP\rightarrow N})
     +\mathbb P(\mathrm{FP\rightarrow TN}))}
     {\mathbb P(\mathrm{impr})+\mathbb P(\mathrm{FP\rightarrow TN})}$    
\end{proof}

As long as enough agents benefit from improvement, $\delta_2>\delta$ is not required. One example of long-term alignment with $\delta_2<\delta$ is in \S~\ref{app:longterm_alignment_example} in the appendix.

\subsection{Simulations}
\label{app:align_simulations}

To understand Figure~\ref{fig:incentives_align}, we provide additional heatmaps in Figure~\ref{fig:allheatmaps}. In (c) we have the same plot as in Figure~\ref{fig:incentives_align}, but for a wider range of $\epsilon$. In (a) and (b) we have expected agent utilities under $h^\mathrm{pre}$ and  $h^\mathrm{post}$, respectively. Note (c) is generated by subtracting (b)$-$(a). In (b), for low values of $\epsilon$, the classifier $h^\mathrm{post}$ is always the same and does not incur any FPs, hence $r_p$ is not affected by variations in $\delta_2$. Since $h^\mathrm{pre}$ for low $\epsilon$ incurs some FPs, as $\delta_2$ increases agents progressively prefer $h^\mathrm{post}$, as seen in (c). As epsilon increases to intermediate values (around $\epsilon\in(0.3, 0.65)$), both $h^\mathrm{pre}$ and $h^\mathrm{post}$ incur in more FPs, as a consequence of avoiding TNs. This makes agents prefer higher $\epsilon$ when $\delta_2$ is low. However, increasing $\delta_2$ reverses this trend, and agents are more harmed by $h^\mathrm{pre}$ than by $h^\mathrm{post}$, since $h^\mathrm{pre}$ incurs more FPs. For high values of $\epsilon$, both $h^\mathrm{pre}$ and $h^\mathrm{post}$ arrive at the same solution.

\begin{figure}[htbp]
    \centering
    
    \begin{subfigure}{\textwidth}
        \centering
        \begin{subfigure}{0.3\textwidth}
            \centering
            \includegraphics[width=\linewidth]{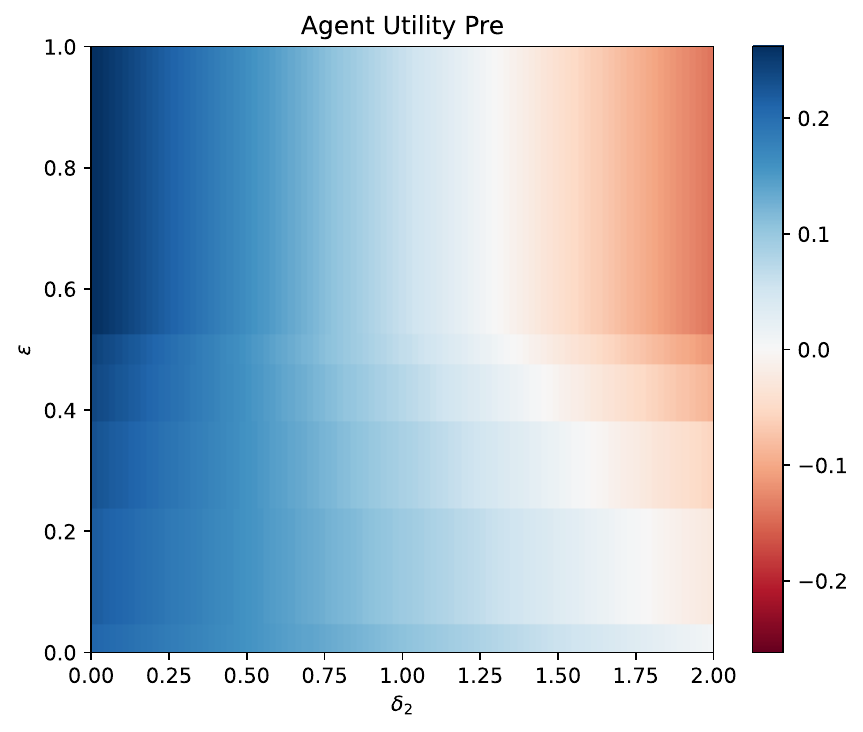}
            \caption{$\underset{x,u}{\mathbb E}[r_p(h^\mathrm{pre},x,u)]$, expected agent utility when minimizing $\ell_{0-1}$ unaware of agent adaptation.}
        \end{subfigure}
        \hfill
        \begin{subfigure}{0.3\textwidth}
            \centering
            \includegraphics[width=\linewidth]{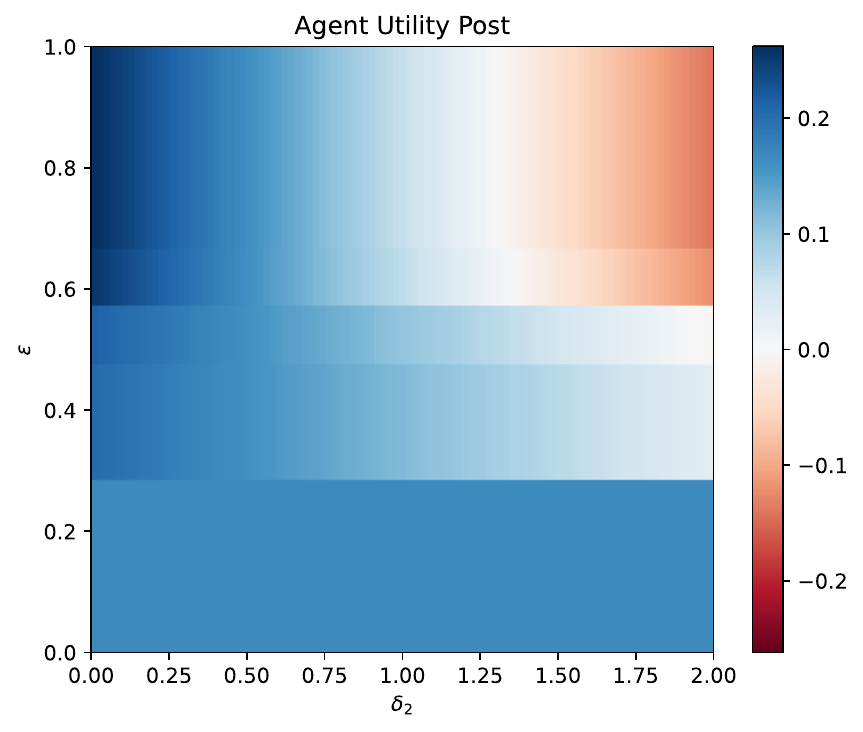}
            \caption{
            $\underset{x,u}{\mathbb E}[r_p(h^\mathrm{post},x,u)]$, expected agent utility when minimizing $\ell_{0-1}$ aware of agent adaptation.}
        \end{subfigure}
        \hfill
        \begin{subfigure}{0.3\textwidth}
            \centering
            \includegraphics[width=\linewidth]{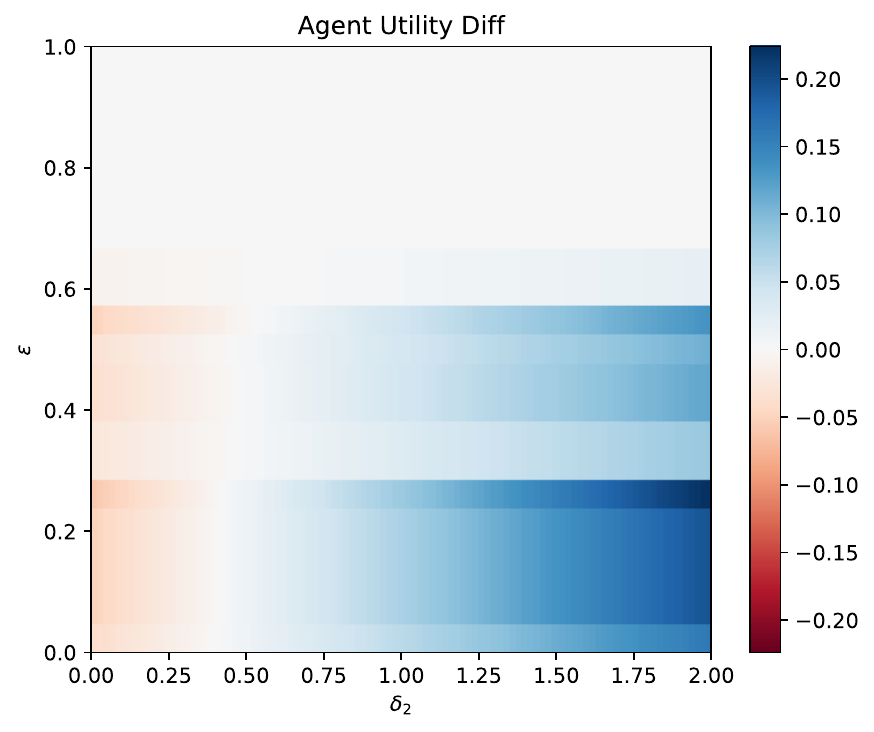}     \caption{Alignment $\underset{}{}\Delta r_p(h^\mathrm{post},h^\mathrm{pre})$, difference between both heatmaps to the left. ($h^\mathrm{post}-h^\mathrm{pre}$)}
            \label{fig:heatmap_diff_eps}
        \end{subfigure}
        \caption{All three plots above consider $\epsilon$ when minimizing pre-adaptation $\ell_{0-1}$}
    \end{subfigure}
    
    \vspace{0.5cm}

    \caption{Detail version of Figure~\ref{fig:incentives_align}, also using $\delta=0.3$. (c) is alignment $\Delta r_p(h^\mathrm{post}-h^\mathrm{pre})$, equal to Figure~\ref{fig:incentives_align} but for a wider range of $\epsilon$. (a) and (b) are expected agent utilities, where the subtraction (b)$-$(a) leads to plot (c).
    }
    \label{fig:allheatmaps}
\end{figure}

\end{document}